\documentclass{article}

\usepackage{my_style}
\usepackage[preprint]{neurips_2026}

\usepackage[utf8]{inputenc} % allow utf-8 input
\usepackage[T1]{fontenc}    % use 8-bit T1 fonts
\usepackage{hyperref}       % hyperlinks
\usepackage{url}            % simple URL typesetting
\usepackage{booktabs}       % professional-quality tables
\usepackage{amsfonts}       % blackboard math symbols
\usepackage{nicefrac}       % compact symbols for 1/2, etc.
\usepackage{microtype}      % microtypography
\usepackage[usenames,dvipsnames]{xcolor}         % colors
\usepackage{enumitem}
\usepackage{wrapfig}
\usepackage{needspace}

\title{Geometric Dictionary Learning of Dynamical\\ Systems with Optimal Transport}

\author{%
Thibaut Germain \\
CMAP, Ecole Polytechnique \\
\texttt{thibaut.germain@polytechnique.edu}
\And 
Sami Chemlal \\
CMAP, Ecole Polytechnique \\
\texttt{samychemlal@yahoo.fr}
\And
Rémi Flamary \\
CMAP, Ecole Polytechnique \\
\texttt{remi.flamary@polytechnique.edu}
\AND 
Vladimir R. Kostic \\
Istituto Italiano di Tecnologia \& University of Novi Sad \\
\texttt{vladimir.kostic@iit.it} \\
\And 
Karim Lounici \\
CMAP, Ecole Polytechnique \\
\texttt{karim.lounici@polytechnique.edu}
}

\begin{document}

\maketitle

\begin{abstract}
Learning dynamical systems through operator-theoretic representations provides a powerful framework for analyzing complex dynamics, as spectral quantities such as eigenvalues and invariant structures encode characteristic time scales and long-term behavior. However, dynamical operators are typically estimated independently for each system, preventing the discovery of shared structure across related dynamics. To address this limitation, we posit that related dynamical systems lie near a low-dimensional manifold in spectral operator space. Based on this hypothesis, we introduce \textbf{\DOODL} (\emph{Dynamical OperatOr Dictionary Learning}), a framework that learns a dictionary of characteristic spectral dynamics whose combinations approximate this manifold and yield compact, interpretable embeddings of individual systems.
Beyond representation learning, \DOODL{} enables fast and interpretable operator estimation from short and partially observed trajectories by constraining the estimation to the learned operator manifold. Experiments on metastable Langevin dynamics and turbulent plasma simulations demonstrate that \DOODL{} scales to highly complex multiscale regimes while capturing characteristic spectral structure governing the dynamics rather than merely fitting trajectories, achieving errors one to two orders of magnitude lower than independent operator estimation methods in challenging low-data regimes.
\end{abstract}

\section{Introduction}

Understanding and modeling dynamical systems (DS) from data is a central problem across scientific and engineering domains. A fundamental challenge is that trajectories of such systems exhibit complex temporal and spatial patterns, arising from multiple interacting time scales and modes. Disentangling these patterns is essential for tasks such as system identification, stability analysis, and control, and is naturally addressed through \emph{spectral decomposition}, which expresses dynamics in terms of fundamental modes, frequencies, and invariant structures. Koopman operator-theoretic approaches provide a principled framework by describing nonlinear dynamics through linear operators acting on observables \cite{koopman1931hamiltonian}. Spectral quantities such as eigenvalues and invariant subspaces encode meaningful information, including time scales and coherent structures.

\noindent
\textbf{Estimation of dynamical operators.}~Learning evolution operators and their spectral structure from data faces several challenges. First,
accurate estimation typically requires long and informative trajectories.
Second, the operator acts on an \emph{a priori unknown} space of observables
that is approximated in practice, for instance with finite non-linear embeddings or
kernel methods \cite{kevrekidis2016kernel, kawahara2016dynamic, klus2015numerical,
kostic2022learning, kostic2023sharp}. As a result, different systems may rely on
different observable spaces, and there is no canonical shared structure across
systems, making comparison and transfer ill-posed.
These challenges are exacerbated in practical settings with short and partial trajectories, often confined to metastable regions. In such regimes, tasks like system identification or long-term prediction become particularly difficult, especially in high-dimensional systems.

\noindent
\textbf{Dictionary learning.}~A natural strategy is to exploit shared structure across systems. Dictionary learning (DL) offers a principled framework for representing data as compositions of reusable atoms, yielding low-dimensional and often interpretable structure \cite{tovsic2011dictionary, mairal2009online, mairal2011task, olshausen1996emergence}. However, classical linear DL methods are designed for static, vector-valued data and do not readily apply to the non-Euclidean space of dynamical operators~\cite{huang2016sparse}. A prior work \cite{huang2016sparse} partially addresses this limitation by proposing a DL approach for linear dynamical systems parameterized by symmetric matrices. Yet, this formulation is restricted in scope: it neither captures the broader class of dynamical operators nor accounts for their spectral structure.
Related ideas have also been explored in Koopman-based approaches, where dictionaries of observables are learned within Extended Dynamic Mode Decomposition (EDMD) frameworks \cite{li2017extended}. However, these methods remain system-specific, operate at the level of observables rather than operators, and do not explicitly learn shared structure across multiple systems.

\noindent
\textbf{The geometry of dynamical systems.}~In parallel, a growing body of work have investigated similarities between dynamical systems through their spectral properties \cite{mezic2004comparison, mezic2016comparison, ishikawa2018metric, martin2000metric, georgiou2007distances, afsari2014distances, vishwanathan2007binet, chaudhry2013initial}. Notably, \cite{germain2026spectral} introduced a geometry based on optimal transport between spectral components, enabling meaningful notions of distance and interpolation. However, these approaches typically assume fully estimated operators and do not address the problem of learning shared structure from data. 

\noindent
\textbf{Research gap.}~These lines of work expose a fundamental gap: existing methods either provide expressive operator-level representations but remain instance-specific, or enable principled comparison without supporting the learning of shared structure. A unified framework that simultaneously captures shared operator structure while enabling transfer and geometric operations remains absent.

\noindent
\textbf{Contribution.}~We introduce DOODL (\emph{Dynamical Operator Dictionary Learning}), a novel framework that operates on a population of related dynamical systems. We posit that related dynamical systems lie near a low-dimensional manifold in spectral operator space. \DOODL{} learns a dictionary of characteristic dynamics (atoms) whose combinations approximate the population manifold and yield shallow vector embeddings of individual systems. \\
\DOODL{} leverages the spectral decomposition of operators to bind dynamical properties (modes, time scales, frequencies) to a smooth manifold, enabling interpolation and combination of dynamics while preserving dynamical properties. The dictionary is learned by minimizing an optimal transport loss and provides interpretable vector embeddings with physical insight that can be used for downstream tasks such as system identification. \\
Beyond its embedding capability, \DOODL{} also enables the estimation of dynamical operators from short trajectories by leveraging its learned representation of the population manifold. This behavior is formalized through an oracle inequality (\Cref{thm:stat-guarantee}). This estimator opens the door to new applications such as early system identification and detection of regime shifts. \\
Through experiments on metastable Langevin dynamics and plasma turbulence, we show that \DOODL{} captures characteristic spectral structure governing complex dynamics rather than merely fitting trajectories. The Langevin setting provides a controlled environment to study how learning the spectral manifold of related metastable systems improves the estimation of a new operator from short trajectories. 
The plasma experiments demonstrate that the learned spectral manifold remains effective in highly turbulent, multiscale regimes characterized by complex spectral interactions, where rapid identification from short trajectory prefixes is critical in fusion-relevant settings.

\section{Background}

\noindent
\textbf{Linear evolution operators \& estimation from data.}~Let $(X_t)_{t\geq 0}$ be a flow on a state space $\X$ with time-homogeneous dynamics. Under mild assumptions, it admits a linear operator representation on an observable space $\spF \subset \mdR^{\spX}$. The transfer operator $\TO_t:\spF\to\spF$ evolves an observable $f$ as $[\TO_t f](x)=\EE[f(X_t)\mid X_0=x]$. The family $(\TO_t)_{t\geq 0}$ forms a semigroup with infinitesimal generator $G=\lim_{t\to0^+}(\TO_t-\Id)/t$, so that $\TO_t=\exp(tG)$ \cite{lasota2013chaos,parzen1999stochastic}.

Assuming $\spF$ is a separable Hilbert space and $G$ has a discrete spectrum, it admits the decomposition $G=\sum_{j} \lambda_j\, g_j \otimes_{\spF} f_j$, where $(\lambda_j,f_j,g_j)$ are eigen-triplets satisfying $G f_j=\lambda_j f_j$, $G^* g_j=\overline{\lambda_j} g_j$, and $\langle f_j,g_j\rangle_{\spF}=\delta_{ij}$. Introducing the spectral projectors $P_j = g_j \otimes_{\spF} f_j$, the evolution reads $[\TO_t f](x)=\sum_j e^{\lambda_j t}\langle f,g_j\rangle_{\spF} f_j(x)$, decomposing $f$ into modes $m_j^f=\innerp{f}{g_j}_{\spF} f_j$ that evolve as scalar oscillators with timescales $\tau_j=\Re(\lambda_j)^{-1}$ and frequencies $\omega_j=\Im(\lambda_j)/(2\pi)$. This decomposition provides interpretable insights into the system dynamics (e.g., stability, periodicity, mixing), while $P_j$ characterize invariant subspaces and coherent structures.
In practice, $G$ is not observed and must be estimated from data. We consider finite-rank estimators $\widehat{G}\in\mcS_r(\mcH)$ capturing the dominant spectral components. Details on estimation procedures are deferred to Appendix~\ref{app:operators-repr}.

\noindent
\textbf{Spectral comparison of operators via optimal transport.}~To compare operators, we propose to use the Spectral-Grassmann OT (SGOT) \cite{germain2026spectral} framework to represent them through their spectral decompositions, capturing both their dynamical time scales and invariant subspaces. Let $G \in \mcS_r(\mcH)$ admit $G = \sum_{i=1}^{\ell} \lambda_i P_i$, where $\lambda_i$ are eigenvalues, $P_i$ the associated spectral projectors, and $m_i = \mathrm{rank}(P_i)$ their multiplicities. We associate to $G$ the spectral measure $\mu(G) = \sum_{i=1}^{\ell} \frac{m_i}{m_{\mathrm{tot}}} \, \delta_{(\lambda_i, P_i)}$, with $m_{\mathrm{tot}} = \sum_i m_i$.\\
Given a divergence $d_{\mcE}$ on projections and parameters $\eta \in (0,1)$, $q \geq 1$, define the ground cost $C_\eta^q((\lambda,P),(\lambda',P')) = \eta |\lambda - \lambda'|^q + (1-\eta) d^q_{\mcE}(P,P')$. The spectral optimal transport (SGOT) divergence between $G$ and $G'$ is then

\vspace{-4mm}
\begin{equation}
\label{eq:sgot_divergence}
d_{\mcS}(G,G') =
\min_{\pi \in \Pi(\mu(G), \mu(G'))}
\sum_{i,j} \pi_{ij} \,
C_\eta^q\big((\lambda_i,P_i),(\lambda'_j,P'_j)\big),
\end{equation}
\vspace{-4mm}

where $\Pi(\mu(G), \mu(G'))$ denotes the set of couplings between spectral measures.
This construction, introduced in \cite{germain2026spectral}, defines a geometry on operators that jointly captures spectral (eigenvalues) and geometric (invariant subspaces) information, enabling meaningful comparison and interpolation between dynamical systems. We refer to Appendix~\ref{app:sgot} for further details and properties.

\noindent
\textbf{Dictionary learning.}~Dictionary learning aims at representing data points $(x_i)_{i\in[N]}$ in a space $\mcX$ through a small number of atoms $D=(D_\ell)_{\ell\in[K]}$, using a decoding function $\tB$ such that $x_i \approx \tB(\alpha_i;D)$ with $\alpha_i \in \mdR^K$. The dictionary and codes are typically learned by minimizing a reconstruction objective of the form $\min_{D,(\alpha_i)} \sum_i \mcL(x_i,\tB(\alpha_i;D)) + \mcR(\alpha_i)$, where $\mcL$ is a data fitting loss and $\mcR$ a regularization term. In classical settings, $\mcX$ is a linear space and $\tB$ is taken to be linear, $\tB(\alpha;D)=\sum_\ell \alpha_\ell D_\ell$. 
While this paradigm has been extended to structured data (e.g., manifolds or probability measures) by replacing linear combinations with barycentric constructions~\cite{ho2013nonlinear,harandi2013dictionary,van2012kernel,van2013design,schmitz2018wasserstein}, such formulations have not been developed for dynamical systems represented through operators, whose intrinsic structure is spectral and operator-theoretic rather than Euclidean. See Appendix~\ref{app:dl} for more details.

\section{Dynamical Operator Dictionary Learning with Optimal Transport}

\paragraph{Problem setting and assumptions.}
We study dictionary learning of dynamical systems via non-defective operator representations. We assume access to a collection of generators estimated from sampled trajectories, and impose the following conditions: \textbf{(A1)} the relevant spectral components of all operators lie in a common finite-dimensional Hilbert space; \textbf{(A2)} all eigenvalues are simple (i.e., have multiplicity one). Assumption \textbf{(A2)} is standard in practice, since repeated eigenvalues occur with probability zero. Assumption \textbf{(A1)} introduces an approximation bias, which can be mitigated by increasing the expressivity of the function space (e.g., through additional Random Fourier Features or richer neural parameterizations). \\
In the following, we first describe the Riemannian structure of spectral decompositions and the associated optimization tools. We then introduce the Dynamical Operator Dictionary Learning (\DOODL) framework, and finally propose a dictionary-based estimator that improves performance on short trajectories while providing theoretical guarantees.

\subsection{A Riemanian manifold of spectral decomposition}
We provide a novel characterization of the manifold of spectral decompositions, and derive the tools required to perform gradient-based optimization on this space. While the differential geometry of classical matrix manifolds, including Stiefel and Grassmann manifolds, has been extensively studied and widely applied in optimization and machine learning \cite{absil2008optimization, boumal2020introduction}, the manifold structure induced by spectral decompositions has been, to the best of our knowledge, studied only in specific settings. The closest related works focus on real-valued symmetric matrices \cite{absil2012projection} or bi-orthogonal square matrices \cite{glashoff2016optimization}; however, these settings differ fundamentally from ours, as we consider low-rank complex matrices, leading to a distinct geometric structure.

\noindent
\textbf{A geometric perspective on spectral decomposition.}~Let $\mcH$ be a complex $\dimH$-dimensional Hilbert space and $\mcS_r(\mcH)$ being the set of non defective operators with rank at most $r$ and simple spectrum. Any $G \in \mcS_r(\mcH)$ admits a spectral decomposition
with matrix representation: 
\begin{equation}
    \mbG = \mbR \Diag(\bs\Lambda) \mbL^* \quad \text{s.t.} \quad \mbL^*\mbR = \mbI_r~,
\end{equation}
where $\bs\Lambda \in \mdC^r$ are the eigenvalues, $(\mbL,\mbR) \in (\mdC^{\dimH \times r})^2$ are the left/right eigenvectors satisfying the bi-orthogonal condition. 
However, operators' spectral decomposition are not unique: under the simple spectrum assumption, eigenvectors are defined up to independent rescaling. This induces a natural equivalence relation on spectral decompositions, corresponding to component-wise actions of $(\mdC^*)^r$ on eigenvectors. We therefore view spectral decompositions as elements of a quotient space
\begin{equation}
    \mcM = \mcN/(\mdC^*)^r \quad \text{s.t.} \quad \mcN = \{(\bs\Lambda,\mbL,\mbR) \ | \ \mbL^*\mbR = \mbI_r\}~,
\end{equation}
where $\mcN$ denotes the space of eigenvalues and bi-orthogonal eigenvectors. Endowed with this quotient structure, $\mcM$ admits a smooth manifold structure, and $\mcS_r(\mcH)$ can be identified with $\mcM$.
All technical details regarding the construction of $\mcN$ and $\mcM$ 
are deferred in Appendix~\ref{app:spectral_decomp_manifold}.

\noindent
\textbf{Optimization on the manifold of spectral decomposition.}~To enable optimization over $\mcM$, we endow it with a Riemannian metric inherited from the ambient space $\mcN$. We consider a metric that is invariant under the action of $(\mdC^*)^r$, ensuring that it is well defined on the quotient manifold $\mcM$. This construction yields consistent notions of tangent spaces, gradients, and vector transports.
In practice, optimization proceeds via Riemannian gradient methods: Euclidean gradients are projected onto the tangent space, and updates are mapped back to the manifold through a suitable retraction. From a computational standpoint, given the Euclidean gradient, a Riemannian gradient step has complexity $\mcO(r^3 + r^2 n^2)$ corresponding to matrix inversions and multiplications. In the typical low-rank regime ($r \ll n$), the inversion cost is negligible, and the overall complexity is effectively governed by the matrix multiplications, making the approach suited for large-scale optimization. We defer the explicit expressions of the metric, tangent space, gradient projection, and retraction to Appendix~\ref{app:spectral_decomp_manifold}.

\subsection{Dynamical Operator Dictionary Learning (\DOODL)}
For conciseness of notation, we denote by $\mbG$ the matrix representation of an operator $G \in \mcS_r(\mcH)$, and by $\bs\mcG = (\bs\Lambda, \mbL, \mbR) \in \mcN$ its spectral decomposition

\noindent
\textbf{Problem formulation \& data fitting loss.}~Given $N$ operators with spectral decompositions $\textstyle \{\bs\mcG_i\}_{i \in [N]} \subset \mcN$, we aim to learn a dictionary of $d$ atoms $\overline{\bs\mcG} \in \mcN^d$ by solving

\vspace{-3mm}
\begin{equation}
    \label{eq:spectral_dictionary_learning}
    \min_{\{\bs\alpha_i\}_{i \in [N]} \subset \bs\Delta^{d-1}, ~\overline{\bs\mcG} \in \mcN^d} \quad \frac{1}{N} \sum\nolimits_{i \in [N]}d_\mcS(\tB(\bs\alpha_i, \overline{\bs\mcG}), \bs\mcG_i)~,
\end{equation}
\vspace{-3mm}

where $\bs\alpha_i \in \bs\Delta^{d-1}$ are the coefficients associated to the
operator atom $\overline{\bs\mcG_j}$ in the dictionary $\overline{\bs\mcG}$ and
$\bs\Delta^{d-1}$ the $(d{-}1)$-simplex, and $\tB$ is the model
reconstructing the operator from its coefficients and the dictionary. The data fitting loss $d_\mcS$
corresponds to the optimal transport divergence defined in
\cref{eq:sgot_divergence}, with a cost function parametrized with $q=2$ and a
spectral projector divergence $d_{\mcE}$ that can be chosen depending on the learning task.
The dictionary learning problem  \ref{eq:spectral_dictionary_learning}    aims to minimize reconstruction error of
operators parametrized and compared through their spectral decompositions.
Unlike classical dictionary learning which includes a sparsity regularization
term, we constrain coefficients to lie in the simplex, yielding convex
combinations of atoms. 

\noindent
\textbf{Reconstruction model.}~The goal of the reconstruction model is to recover a spectral decomposition from coefficients $\bs\alpha$ given a dictionary $\overline{\bs\mcG}$. Classical DL approaches rely on a linear reconstruction, i.e., a weighted sum of dictionary atoms \cite{mairal2009online, mairal2011task}. While natural and effective for linear operators, such models fail to capture the nature of evolution operator, as shown by our numerical experiments.\\
In contrast, the SGOT divergence \cref{eq:sgot_divergence} naturally compares evolution operators through their spectral decompositions \cite{germain2026spectral} suggesting a nonlinear DL strategy at the level of distributions, where reconstruction is performed via Wasserstein barycenters. Such approaches, based on free-support optimal transport barycenters \cite{cuturi2014fast}, have been explored in \cite{schmitz2018wasserstein}. However, despite their geometric fidelity, these methods have substantial computational overhead limiting their practical applicability. 
A discussion on Wassersterin barycenter with SGOT can be found in Appendix \ref{app: wrm}. In a similar fashion, DL methods on Riemannian manifolds \cite{cherian2016riemannian, li2013log, harandi2012sparse}, typically consider approximation of geodesic barycenters as reconstruction models to avoid the computational burden of Riemannian geometry~\cite{ho2013nonlinear}.\\
Our approach combine the strength of existing methods; we propose a nonlinear model that performs a linear combination of spectral components followed by a projection onto the manifold of spectral decompositions $\mcN$. This approach leverage efficiency of linear models and accounts for the underlying manifold. Formally, given a dictionary $\overline{\bs\mcG} \in \mcN^d$ and coefficients $\bs\alpha \in \bs\Delta^{(d-1)}$, the reconstruction is: 
\begin{equation}
    \label{eq:reconstruction_model}
    \textstyle\tB^{\text{\tiny p}}(\bs\alpha; \overline{\bs\mcG}) \triangleq \tP_{\mcN}(\sum_{j \in [d]} \alpha_j \overline{\bs\Lambda}_j,\sum_{j \in [d]} \alpha_j \overline{\mbL}_j,\sum_{j \in [d]} \alpha_j \overline{\mbR}_j)~.
\end{equation}
The projection operator, $\textstyle\tP_{\mcN}(\bs\Lambda,\mbL, \mbR) \triangleq (\bs\Lambda, \argmin_{\tilde{\mbL}~ \text{s.t.}~\tilde{\mbL}^*\mbR = \mbI_r} \|\tilde{\mbL}-\mbL\|_F^2 ,\mbR)$, restores feasibility by
projecting the left eigenvectors to satisfy the bi-orthogonality constraint. The
construction is asymmetric, as it leaves eigenvalues and right eigenvectors
unchanged.  
For dynamical systems, this asymmetry is often well-justified since right
eigenvectors encode the dominant dynamic modes. In this sense, the model can be viewed as a smooth local retraction on $\mcN$ around $(\bs\Lambda, \mbR)$, see Appendix~\ref{app:spectral_decomp_manifold}. 
Compared to Wasserstein barycenter-based reconstruction, the proposed model is not invariant to permutations of spectral components due to the linear aggregation step. In practice, this sensitivity can be mitigated by increasing the dictionary size, which reduces ordering effects. On the other hand, its main advantage is its computational efficiency: it avoids any inner optimization or optimal transport problem, while still incorporating manifold structure through a single projection step.

\noindent
\textbf{Optimization scheme.}~We adopt a stochastic inexact block-coordinate
strategy similar to the one proposed in \cite{mairal2009online} for euclidean spaces. Given a batch of size $b$, we perform two successive steps:
\begin{enumerate}[leftmargin=*,itemsep=0pt,topsep=1pt]
    \item \emph{Coefficient estimation:} fix the dictionary and optimize
    $\{\bs\alpha_i\}_{i \in [b]}$ via gradient descent (with softmax
    parametrization). Since the problem is not convex, we use an initialization
    of  based on proximity to dictionary atoms, i.e.
    $\textstyle\bs\alpha_i \propto \{- d_\mcS(\bs\mcG_i,
    \overline{\bs\mcG}_j)\}_{j \in [d]}$ to start the optimization in a relevant region of the parameter space. 
    \item \emph{Dictionary update:} fix the coordinates $\{\bs\alpha_i\}_{i
    \in [b]}$ and update the dictionary with a Riemannian gradient step on
    $\mcN^d$ from the objective on the batch. We leverage the envelope theorem to ignore implicit gradients
    through $\{\bs\alpha_i\}_{i \in [b]}$.
\end{enumerate}
Further details on the optimization scheme can be found in Appendix~\ref{app:doodl}. Since operators are typically low-rank with parametrization on simplex, the main computational cost arises in the \emph{coefficient estimation}, due to repeated evaluations of the reconstruction model.

\noindent
The DOODL framework integrates the strengths of existing DL methods. Its reconstruction model combines efficiency with geometric fidelity which enables the learning of operator dictionaries through Riemannian optimization, where gradients are guided by an optimal transport divergence that captures the geometry of spectral decompositions central to evolution operators.

\subsection{Dictionary-based Operator Estimation from Short Trajectories}

\DOODL{} can be used to estimate operators from short trajectories by restricting the search space to the image of the reconstruction model $\tB^{\tp}$ induced by a learned dictionary. This low-dimensional constraint regularizes estimation in low-data regimes where classical estimators are unreliable. Given a trajectory $(X_t)_{t\in[\nsizeds]}$, we estimate the operator by solving
\begin{equation}
\label{eq: low_data_metric_estimator}
\textstyle\min{\bs\alpha} \quad\frac{1}{\nsizeds} \sum_{t \in [\nsizeds]}\| z_{t+1} - \tO\tp[\tB^{\text{\tiny p}}(\bs\alpha ; \overline{\bs\mcG})]^* z_t \|_{\mcH}^2~,
\end{equation}
where $z_t = \phi(X_t)$ denotes the feature representation of $X_t$ in the Hilbert space $\mcH$, and $\tO\tp$ maps a spectral decomposition to its associated operator. This corresponds to empirical risk minimization constrained to the learned operator manifold. We provide statistical guarantees below showing improved estimation in short-trajectory regimes.

\noindent
\textbf{Statistical guarantees.}~Let $\bar{\bs\mcG}^*$ denote the optimal dictionary solving \eqref{eq:spectral_dictionary_learning}, and $\hat{\bs\mcG}$ its empirical estimate. Define
$G_{\bs\alpha}:=\tO\tp[\tB(\bs\alpha;\hat{\bs\mcG})]$
and
$\Gset(\hat{\bs\mcG}) := \{G_{\bs\alpha}:\bs\alpha\in\Delta^{d-1}\}$.
Given a stationary test trajectory $(X_t)_{t=1}^{\nsizeds+1}$ independent of the training data, we define the empirical and population risks
\begin{equation}
    \ell_t(G) := \|\feat_{t+1} - G^*\feat_t\|_{\mcH}^2,
    \quad
    \erisk(\bs\alpha \mid \hat{\bs\mcG})
    := \tfrac{1}{\nsizeds}
    \textstyle\sum_{t=1}^{\nsizeds} \ell_t(G_{\bs\alpha}),
    \quad
    \risk(\bs\alpha \mid \hat{\bs\mcG})
    := \mathbb{E}_\pi[\ell_t(G_{\bs\alpha})],
\end{equation}
with minimizer $\hat{\bs\alpha} \in \argmin_{\bs\alpha}
\erisk(\bs\alpha\mid\hat{\bs\mcG})$.
We set $\kappa := \sup_{x}\|\phi(x)\|_{\mcH}$,
$M := \sup_{j}\|\tO\tp[\hat{\bs\mcG}_j]\|_{\mathrm{op}}$,
$m \leq d-1$ the intrinsic dimension of $\Gset(\hat{\bs\mcG})$
under $\|\cdot\|_{\mathrm{op}}$ (diameter $\leq 2M$,
since $\|G_{\bs\alpha}\|_{\mathrm{op}} \leq M$ for all $\bs\alpha$),
$\sigmavariance :=
\sup_{G \in \Gset(\hat{\bs\mcG})}\mathrm{Var}_\pi(\ell_t(G))$
and $\sigmavariancebis :=
\sup_{G \in \Gset(\hat{\bs\mcG})}\mathbb{E}_\pi[\ell_t(G)^2]$
the uniform loss variance and second moment.
For $\beta$-mixing trajectories with block length $\tau$,
let $\nsizeds_{\mathrm{eff}} := \lfloor\nsizeds/(2\tau)\rfloor$
and $\sigmavarianceeff :=
\sigmavariance + 4\sigmavariancebis\sum_{s=1}^{\tau-1}\beta(s)$. We need the following assumption.

\vspace{-.015truecm}
\begin{enumerate}[label={\rm \textbf{(WC)}},leftmargin=7ex,nolistsep]
\item\label{ass:sigma-min} \emph{Well-conditioned dictionary:}
The right eigenvectors $\{\bar{\mbR}_j\}_{j \in [d]} \subset \mathbb{C}^{p \times r}$
are in general position, i.e.\ no convex combination
$\sum_{j \in [d]} \alpha_j \bar{\mbR}_j$ is rank-deficient
for $\bs\alpha \in \Delta^{d-1}$.
\end{enumerate}
This assumption is analogous to task diversity conditions in meta-learning~\citep{tripuraneni2021provable}, ensuring that the dominant dynamical modes span a sufficiently rich subspace. It implies the uniform lower bound
$\sigma_{\min}(\sum_{j \in [d]} \alpha_j \bar{\mbR}_j)\ge c>0$
for all $\bs\alpha\in\Delta^{d-1}$.
We denote by $L_{\mathrm{dec}}$ the Lipschitz constant of the decoder; bounds for 
$\tB^{\text{\tiny p}}$ is given in Lemma~\ref{lem:proj-decoder-lipschitz}.

\begin{restatable}[Oracle inequality]{thm}{statguarantee}
\label{thm:stat-guarantee}
Let {\rm \textbf{(WC)}} be satisfied. Suppose the test trajectory $(X_t)_{t \in [\nsizeds+1]}$ is stationary,
$\beta$-mixing,
and independent of the training data.
Then for any $\delta \in (0,1)$, with probability at least
$1 - \delta - 2(\nsizeds_{\mathrm{eff}}-1)\beta(\tau)$,
\begingroup
\setlength{\abovedisplayskip}{3pt}
\setlength{\belowdisplayskip}{3pt}
\begin{equation}
\label{eq:stat-guarantee}
    \risk(\hat{\bs\alpha} \mid \hat{\bs\mcG})
    - \inf_{\bs\alpha}\risk(\bs\alpha \mid \bar{\bs\mcG})
    \;\leq \mathrm{bias}(\hat{\bs\mcG})  + c'
    \sqrt{
        \frac{
            \sigmavarianceeff\cdot
            m\log\!\left(M L_{\mathrm{dec}}(1+M)\kappa^2
            \sqrt{\nsizeds_{\mathrm{eff}}}\right)
            + \log(1/\delta)
        }{\nsizeds_{\mathrm{eff}}}
    }
\end{equation}
\endgroup
where $c'>0$ is a numerical constant and $\mathrm{bias}(\hat{\bs\mcG}) :=
\inf_{\bs\alpha}\risk(\bs\alpha\mid\hat{\bs\mcG})
- \inf_{\bs\alpha}\risk(\bs\alpha\mid\bar{\bs\mcG}^*) \geq 0$.
\end{restatable}

\noindent
\textbf{Discussion.}~The bound separates two terms. The \emph{coordinate fitting} term controls estimation of the low-dimensional embedding coefficients $\hat{\bs\alpha}$ from short trajectories, with complexity governed by the intrinsic dimension $m$ of the learned operator manifold rather than the ambient dictionary size $d$. {The learned dictionary therefore acts as a transferable inductive bias restricting estimation to dynamically plausible operators}, which is particularly advantageous in short-trajectory regimes where unconstrained operator estimation is statistically unstable. The effective sample size $\nsizeds_{\mathrm{eff}}$ reflects the mixing properties of the dynamics, while the variance term $\sigmavariance$ measures how well the learned manifold captures the observed dynamics. Finally, the dictionary bias quantifies the approximation gap induced by the learned manifold and vanishes asymptotically as $\ntrain \to \infty$ (Proposition~\ref{prop:bias-control}). Full proofs are deferred to Appendix~\ref{app:stat}.

\section{Numerical Experiments}

Our experiments are twofold: (i) metastable Langevin dynamics, where we study operator estimation from short non-equilibrium trajectories and detection of regime switches; and (ii) turbulent plasma dynamics relevant to nuclear fusion, where we investigate whether \DOODL{} can turn high-dimensional turbulent simulations into low-dimensional spectral coordinates supporting operator recovery and early regime identification.

\subsection{Dictionary Learning on Langevin dynamics}

\noindent
\textbf{Langevin dynamics in science.}~Langevin dynamics provides a fundamental framework for modeling stochastic systems driven by both deterministic forces and environmental noise. It plays a central role in scientific fields such as molecular dynamics and statistical physics, where key quantities are inherently dynamical and spectral: metastable states, transition pathways, relaxation timescales, and invariant distributions \cite{schutte2001transfer,schutte2023overcoming}. In practice, reliably estimating the associated evolution operators is challenging because trajectories are often short, high-dimensional, partially observed, and far from equilibrium \cite{bolhuis2002transition, wu2017variational}. In this experiment, we investigate whether the estimated manifold of dynamical systems through geometric dictionary learning can overcome these limitations.

\noindent
\textbf{Data simulation \& operator estimation.}~We consider a one-dimensional damped Langevin dynamics governed by $\td X_t = \nabla U_w(X_t) + \sqrt{2 \sigma} \td B_t$ with $\{B_t\}_{t>0}$ a Brownian motion, $\sigma >0$ its temperature, and a two well potential $U_w(x) \triangleq w^{-4}(x^2-w^2)^2$ where $w >0$ controls the distance between the two wells and their respective width, see Fig.\ref{fig:langevin_activation}(left). We uniformly sampled 256/256 train/test potential parameters $w \in [0.5,1.2]$ and generates trajectories of 40k samples at 100Hz according to their Langevin dynamics. The operators are estimated via Reduced Rank Regression (RRR) \cite{kostic2022learning} with a fixed rank fixed of three, and an observable space approximating the RBF kernel via 400 Random Fourier Features \cite{avron2017random} on sliding windows of 50 samples. Settings are detailed Appendix \ref{app: langevin exp}.

\begin{figure}[t]
    \centering\vspace{-3mm}
    \includegraphics[width=\linewidth]{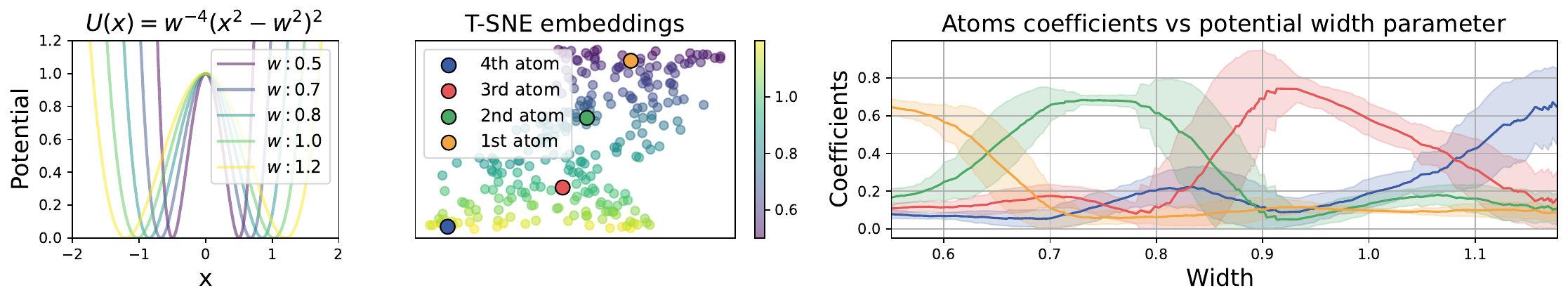}
    \vspace{-7mm}
    \caption{(left) Two-wells Langevin potential for different widths. (middle) T-SNE with SGOT of dictionary atoms and training operators colored by width. (right) Average operators' reconstruction coefficients depending on potential width.}
    \label{fig:langevin_activation}
\end{figure}

\noindent
\textbf{Learning the geometry of Langevin dynamics.}~On the training set, we learn DOODL dictionaries \cref{eq:spectral_dictionary_learning} with 2 to 5 atoms, using the SGOT divergence with $\eta=0.25$ and the log-Martin metric for the spectral projector term. For four atoms, \Cref{fig:langevin_activation} (middle) shows a 2D t-SNE embedding of SGOT distances of training operators and learned atoms, with samples colored by the potential width. \Cref{fig:langevin_activation} (right) reports the corresponding activation coefficients. These results reveal that this family of one-parameter Langevin dynamics lies on a low-dimensional (approximately 1D) manifold in operator space. The learned DOODL atoms align along this structure, yielding smooth and ordered activations as the width parameter varies. This shows that DOODL recovers a coherent low-dimensional organization of the underlying dynamics, where nearby systems share smoothly varying operator representations.

\begin{figure}[t]
    \centering\vspace{-2mm}
    \includegraphics[width=\linewidth]{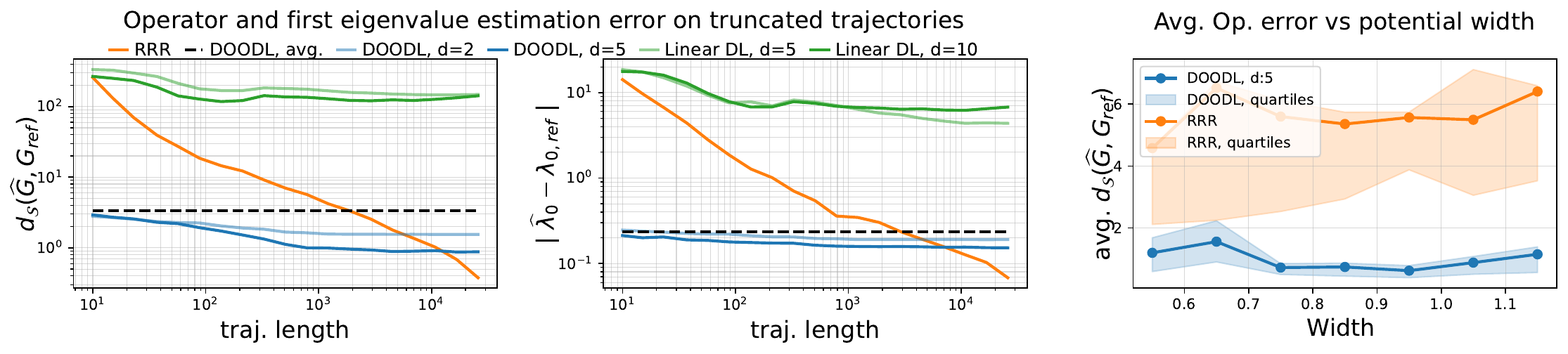}
    \vspace{-7mm}
    \caption{Comparison of operator estimator on truncated trajectories, including individual RRR, linear DL and DOODL. Error to ground truth in SGOT divergence between operators (left) and absolute error between first eigenvalues (middle). (right) Average SGOT error per potential width.}
    \label{fig:langevin_early_estimation}
\end{figure}

\noindent
\textbf{Dictionary-based operator estimation from short trajectories.}~With previously trained dictionaries, we estimate operators from test trajectories truncated between 10 and 40k samples (20 log-spaced steps), with \DOODL{} in \cref{eq: low_data_metric_estimator}. We compare against individual RRR estimators, linear DL baselines, and the average reconstruction model on the training set \cref{eq:reconstruction_model}. \Cref{fig:langevin_early_estimation}-(left,middle) reports reconstruction error with SGOT divergence and the error on the leading eigenvalue respectively. \DOODL{} outperforms all baselines up to 10k samples, achieving a SGOT error up to two orders of magnitude lower in the short-trajectory regime. This shows that constraining estimation to the learned spectral manifold preserves physically meaningful dynamical structure rather than merely fitting trajectories. {Note that for long enough trajectories direct operator estimation RRR performs better because of the bias term in Eq. \eqref{eq:stat-guarantee}.} Linear DL performs worst, suggesting that naive dictionary learning approaches that ignore the intrinsic geometry of operator space fail to recover dynamical properties.
This trend is further confirmed in \Cref{fig:langevin_early_estimation}-(right), where trajectory length is fixed to 1.2k samples: \DOODL{} yields lower variance across all potential widths compared to RRR, indicating that the learned operator manifold stabilizes operator estimation. Further results are in Appendix~\ref{app: langevin exp}.

\begin{figure}[t]
     \centering
     \includegraphics[width=\linewidth]{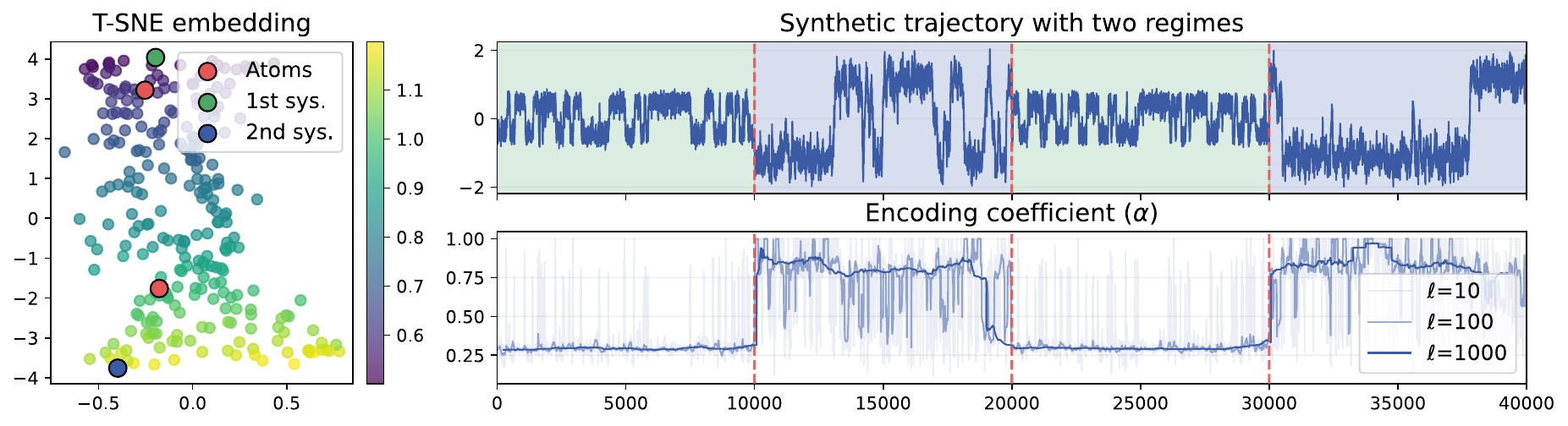}
     \vspace{-7mm}
     \caption{Regime-switch detection from operator embeddings. (left) SGOT geometry of training operators and learned atoms. (right-top) Trajectory with three switches. (right-bottom) Rolling DOODL coefficient for different window lengths.}
     \label{fig:langevin_switching}
\end{figure}

\noindent
\textbf{System identification through dictionary-based operator embeddings.}~Estimating operators from short trajectories naturally enables the detection of regime changes. We illustrate this on a system alternating between two Langevin regimes with wide and narrow potential widths. \Cref{fig:langevin_switching}-(top right) shows a representative trajectory with switching events. Considering a \DOODL{} dictionary with two atoms (\Cref{fig:langevin_switching}, left), we track the first embedding coefficient of operators estimated over sliding windows of length $\ell \in {10,100,1000}$. The resulting coefficient trajectories (\Cref{fig:langevin_switching}, bottom right) exhibit sharp discontinuities at switching times, with reduced variance as $\ell$ increases. This demonstrates that the manifold approximation induces  coordinate system capable of accurately tracking abrupt changes in the underlying dynamics directly from short trajectory windows.

\subsection{Dictionary Learning on plasma dynamics}

\begin{figure}[t]
    \centering
    \includegraphics[width=0.95\linewidth]{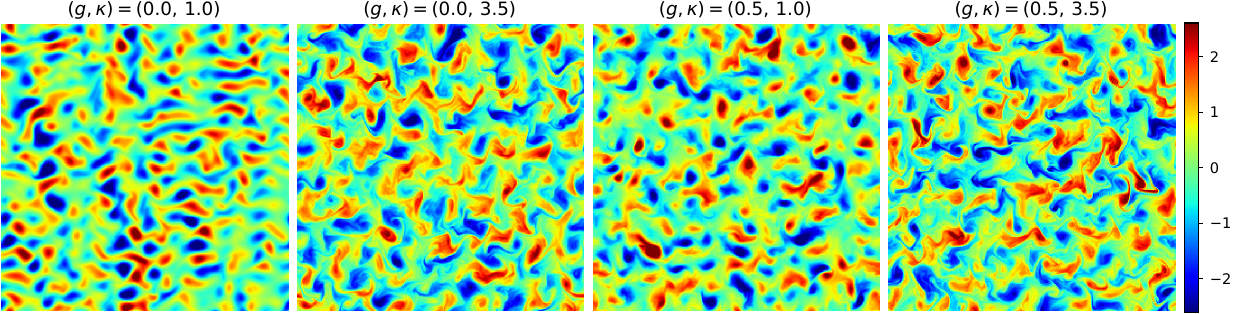}
\caption{Density fluctuations across plasma regimes for varying $(g,\kappa)$.
}
    \label{fig:plasma_snapshots}
\end{figure}

\begin{figure}[t]
    \centering
    \includegraphics[width=0.95\linewidth]{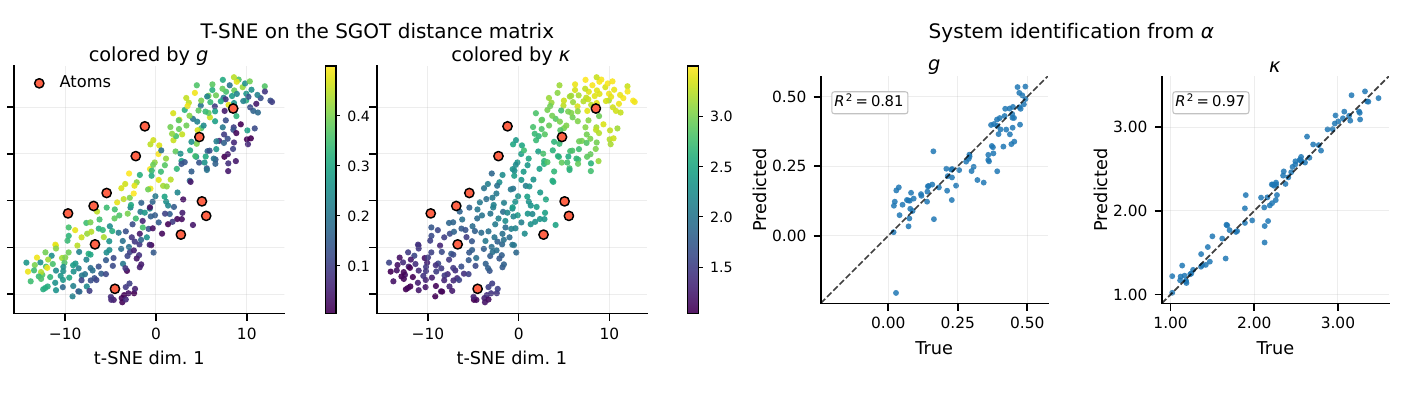}
    \vspace{-6mm}
    \caption{Geometry of turbulent plasma dynamics. Left: SGOT geometry organizes plasma regimes by $(g,\kappa)$. Right: regression of physical parameters from barycentric embeddings $\alpha_m$.
    %Learned geometry of turbulent plasma dynamics. Left: the SGOT operator geometry organizes plasma regimes into a coherent low-dimensional manifold aligned with the governing physical parameters $(g,\kappa)$, with learned DOODL atoms lying on the envelope of the operator manifold. Right: regression of physical parameters from the barycentric operator embeddings $\alpha_m$, showing that the learned geometry preserves physically meaningful regime organization.
}
    \label{fig:plasma_sgot_regression}
\end{figure}
\noindent
\textbf{Plasma dynamics and simulation.}~We consider a reduced two-field model for edge plasma turbulence with density and electrostatic potential. 
The parameter, $g$, controls the interchange coupling between density and vorticity while $\kappa$ controls the imposed density-gradient drive \citep{ghendrih2018generation,ghendrih2022role}. \Cref{fig:plasma_snapshots} shows examples of density fluctuations on different values of $g$ and $\kappa$. Increasing $\kappa$ changes the fluctuation scale and contrast, while
increasing $g$ makes the structures more distorted. See Appendix \ref{sec:plasma details} for details. 

Estimating such control parameters $\kappa$ and $g$ from plasma data is a challenge in practice. Turbulence, instabilities, and anomalous transport make quantitative modeling difficult \citep{boeuf2023physics}. Even in reduced 2D models, transport parameters require statistical calibration and can vary between regimes \citep{coosemans2021bayesian}. This motivates data-driven reduced-order representations that extract coherent dynamical structures from high-dimensional plasma trajectories \citep{faraji2024dynamic,taylor2018dynamic}.\\

The goal of this experiment is to test whether \DOODL{} can 
map high-dimensional turbulent Tokam2D \cite{tokam2d} simulations into a low-dimensional spectral coordinate system that enables rapid system identification from short turbulent trajectories.\\
We generate plasma dynamics with Tokam2D~\citep{tokam2d}, a reduced 2D fluid solver that
evolves density and vorticity fields in a plane transverse to the magnetic field
\citep{ghendrih2018generation,ghendrih2022role}. In the gradient-driven configuration, we uniformly
sample $M=400$ dynamics with parameters $(g,\kappa)$ in $[0,0.5]\times[1,3.5]$
keeping other numerical parameters fixed. Each system is simulated on a $128\times128$
grid for $T=4093$ time steps with $\Delta t=0.05$. We did $80/20$ train/test split over simulations.

\noindent
\textbf{Operator estimations from time series.}~The learning plasma dynamic operators is in two steps: (i) learning a common latent space for individual operator estimation, (ii) estimating the operator through their resolvent.
The common latent space is learned from sampled trajectories that are first embedded with POSEIDON-B, a foundational model for PDE~\citep{herde2024poseidon}, then fed to two MLP heads applied to consecutive states. The first
head produces a 64-dimensional latent state, while the second is only used during training. The heads are trained with a contrastive 
loss
inspired by \cite{kostic2024neural}.
Since the systems have different turbulence levels and loss scales, we train the
heads with a GradNorm-weighted multitask strategy~\citep{chen2018gradnorm}. From the resulting latent trajectories, we estimate a
individual rank-40 operators with a generator-resolvent (GR) estimator~\cite{kostic2024laplace}. Details  are given in Appendix~\ref{app:plasma}.

\noindent
\textbf{Operator dictionary and recovered manifold.}~We apply \DOODL{} on the learned operators
with $d=12$
atoms and the SGOT divergence with log-Martin and
$\eta=0.9$. 
Each system is represented by embeddings
$\alpha_m\in\Delta^{d-1}$
which are used for regression on the parameters.
T-SNE visualization in Figure~\ref{fig:plasma_sgot_regression} shows that the SGOT geometry recovers,
only from data, the structure of the 2D manifold of the plasma regimes
corresponding to changing both $g$ and $\kappa$. The learned atoms lie on the envelope
of this operator manifold, consistent with their role as dictionary prototypes. Regressing from
the $\alpha_m$ gives $R^2=0.81$ for $g$ and $R^2=0.97$ for $\kappa$ on held-out
regimes, showing that the \DOODL{} embedding encode the physical parameters.

\noindent
\textbf{Early operator recovery and system identification.}~We test whether the learned dictionary helps recover the plasma operator from short trajectories. 
For each test trajectory length, we estimate the operator with \DOODL{} and compare it to the reference operator estimated from the full trajectory. 
Figure~\ref{fig:plasma_early_operator} (top) shows that 
the constrained estimate has better performances compared to direct GR estimation for length up to 3k time steps.
%this constrained estimate is closer to the reference over most of the trajectory. Direct GR improves only when the prefix becomes long enough, and becomes
%competitive only near the end. 
%This demonstrates that the estimation of the operator manifold through geometric dictionary learning
This demonstrates that constraining estimation to the learned spectral manifold enables reliable identification of\begin{wrapfigure}[22]{r}{0.36\textwidth}
    %\vspace{-0.6em}
    \vspace{-2mm}
    \centering
    \includegraphics[width=\linewidth]{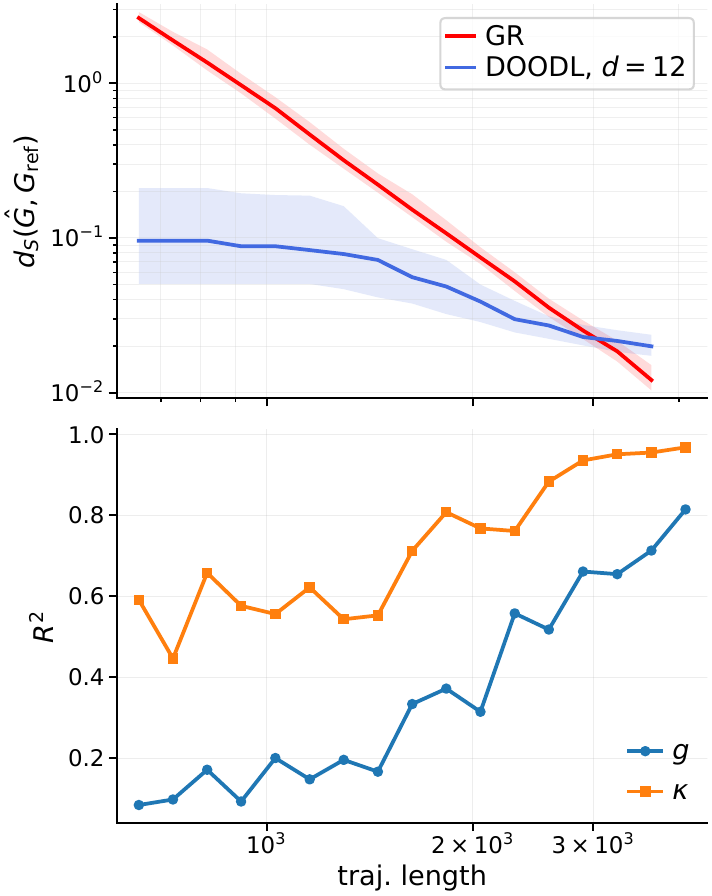}
    \vspace{-1.0em}
    \caption{Early operator recovery and parameter identification from short plasma trajectories. Top: operator estimation. Bottom: parameter regression from DOODL embeddings.
    % Early recovery of turbulent plasma dynamics from short trajectory prefixes. DOODL leverages transferable operator geometry to identify plasma regimes and recover physically meaningful parameters long before direct operator estimation becomes reliable. Top: operator estimation from partial trajectories. Bottom: parameter recovery from early DOODL embeddings.
    }
    \label{fig:plasma_early_operator}
    \vspace{-2.0em}
\end{wrapfigure} complex plasma regimes long before direct operator estimation becomes stable. Once the spectral manifold is learned, inference reduces to estimating a small number of barycentric coordinates rather than recovering an unconstrained high-dimensional operator from scratch, yielding a lightweight reduced-order representation of turbulent plasma dynamics.\\
%transferable operator geometry 
% enables reliable identification of complex plasma regimes from short trajectory prefixes, long before direct operator estimation becomes stable. 
% Once the operator geometry is learned, inference reduces to estimating a small number of barycentric coordinates rather than recovering an unconstrained high-dimensional operator from scratch, yielding a lightweight reduced-order representation of turbulent plasma dynamics.\\
We now test whether \DOODL{} coordinates estimated from short trajectories enables system identification.
%preserve physical regime information. 
Figure~\ref{fig:plasma_early_operator} (bottom) reports the $R^2$ obtained when predicting $g$ and $\kappa$ from coordinates estimated at each length. The two parameters behave differently. $\kappa$ is recovered from short trajectories, suggesting that its effect is encoded in leading spectral modes that emerge rapidly. In contrast, $g$ improves more gradually and requires longer trajectories, consistent with results in Appendix~\ref{app:plasma_spectral_coordinates} where it is associated with slower and higher-order spectral components. This indicates that the learned operator geometry captures how physical parameters are distributed across spectral time scales in turbulent plasma dynamics. 
Such lightweight operator representations are particularly attractive in fusion-relevant turbulent plasma regimes where rapid identification from short trajectory prefixes is critical. {More broadly, our results support the view that turbulent physical systems admit a low-dimensional manifold representation.}

\section{Conclusion}

We introduced \DOODL{}, a geometric dictionary learning framework for dynamical systems based on the hypothesis that related systems lie near a low-dimensional spectral manifold in operator space. \DOODL{} combines spectral theory with optimal transport and Riemannian dictionary learning to construct compact 
operator representations preserving key dynamical properties. \DOODL{} constrains estimation to the learned spectral manifold, enabling fast, interpretable, and data-efficient reasoning about complex dynamics.\\
Despite these promising results, several limitations remain. 
In particular, as is well known in dictionary learning, the underlying optimization problem is non-convex. In our setting, this challenge is further exacerbated by the manifold constraints, motivating the coefficient initialization strategy that we propose. As well, the learned manifold is static and does not adapt online to evolving systems or distribution shifts. Future work will investigate continual and adaptive variants of \DOODL{}, as well as tighter integration between operator estimation and geometric representation learning.

\bibliographystyle{plain}
\bibliography{biblio}

@article{minh2016operator,
  title={Operator-valued Bochner theorem, Fourier feature maps for operator-valued kernels, and vector-valued learning},
  author={Minh, Ha Quang},
  journal={arXiv preprint arXiv:1608.05639},
  year={2016}
}

@article{lanthaler2023error,
  title={Error bounds for learning with vector-valued random features},
  author={Lanthaler, Samuel and Nelsen, Nicholas H},
  journal={Advances in Neural Information Processing Systems},
  volume={36},
  pages={71834--71861},
  year={2023}
}

@inproceedings{brault2016random,
  title={Random fourier features for operator-valued kernels},
  author={Brault, Romain and Heinonen, Markus and Buc, Florence},
  booktitle={Asian Conference on Machine Learning},
  pages={110--125},
  year={2016},
  organization={PMLR}
}

@incollection{schutte2001transfer,
  title={Transfer operator approach to conformational dynamics in biomolecular systems},
  author={Sch{\"u}tte, Ch and Huisinga, Wilhelm and Deuflhard, Peter},
  booktitle={Ergodic theory, analysis, and efficient simulation of dynamical systems},
  pages={191--223},
  year={2001},
  publisher={Springer}
}

@article{wu2017variational,
  title={Variational Koopman models: Slow collective variables and molecular kinetics from short off-equilibrium simulations},
  author={Wu, Hao and N{\"u}ske, Feliks and Paul, Fabian and Klus, Stefan and Koltai, P{\'e}ter and No{\'e}, Frank},
  journal={The Journal of chemical physics},
  volume={146},
  number={15},
  year={2017},
  publisher={AIP Publishing}
}

@article{bolhuis2002transition,
  title={Transition path sampling: Throwing ropes over rough mountain passes, in the dark},
  author={Bolhuis, Peter G and Chandler, David and Dellago, Christoph and Geissler, Phillip L},
  journal={Annual review of physical chemistry},
  volume={53},
  number={1},
  pages={291--318},
  year={2002},
  publisher={Annual Reviews 4139 El Camino Way, PO Box 10139, Palo Alto, CA 94303-0139, USA}
}

@article{schutte2023overcoming,
  title={Overcoming the timescale barrier in molecular dynamics: Transfer operators, variational principles and machine learning},
  author={Sch{\"u}tte, Christof and Klus, Stefan and Hartmann, Carsten},
  journal={Acta Numerica},
  volume={32},
  pages={517--673},
  year={2023},
  publisher={Cambridge University Press}
}

@inproceedings{tripuraneni2021provable,
  title={Provable meta-learning of linear representations},
  author={Tripuraneni, Nilesh and Jin, Chi and Jordan, Michael},
  booktitle={International conference on machine learning},
  pages={10434--10443},
  year={2021},
  organization={Pmlr}
}

@article{rahimi2007random,
  title={Random features for large-scale kernel machines},
  author={Rahimi, Ali and Recht, Benjamin},
  journal={Advances in neural information processing systems},
  volume={20},
  year={2007}
}

@inproceedings{rio1993covariance,
  title={Covariance inequalities for strongly mixing processes},
  author={Rio, Emmanuel},
  booktitle={Annales de l'IHP Probabilit{\'e}s et statistiques},
  volume={29},
  number={4},
  pages={587--597},
  year={1993}
}

@article{kostic2026toeplitz,
  title={Toeplitz Based Spectral Methods for Data-driven Dynamical Systems},
  author={Kostic, Vladimir R and Lounici, Karim and Pontil, Massimiliano},
  journal={arXiv preprint arXiv:2602.09791},
  year={2026}
}

@article{koopman1931hamiltonian,
  title={Hamiltonian systems and transformation in Hilbert space},
  author={Koopman, Bernard O},
  journal={Proceedings of the National Academy of Sciences},
  volume={17},
  number={5},
  pages={315--318},
  year={1931}
}

@inproceedings{cuturi2014fast,
  title={Fast computation of Wasserstein barycenters},
  author={Cuturi, Marco and Doucet, Arnaud},
  booktitle={International conference on machine learning},
  pages={685--693},
  year={2014},
  organization={PMLR}
}

@article{li2017extended,
  title={Extended dynamic mode decomposition with dictionary learning: A data-driven adaptive spectral decomposition of the Koopman operator},
  author={Li, Qianxiao and Dietrich, Felix and Bollt, Erik M and Kevrekidis, Ioannis G},
  journal={Chaos: An Interdisciplinary Journal of Nonlinear Science},
  volume={27},
  number={10},
  year={2017},
  publisher={AIP Publishing}
}

@article{kostic2023sharp,
  title={Sharp spectral rates for Koopman operator learning},
  author={Kostic, Vladimir and Lounici, Karim and Novelli, Pietro and Pontil, Massimiliano},
  journal={Advances in Neural Information Processing Systems},
  volume={36},
  pages={32328--32339},
  year={2023}
}

@article{kostic2022learning,
  title={Learning dynamical systems via Koopman operator regression in reproducing kernel Hilbert spaces},
  author={Kostic, Vladimir and Novelli, Pietro and Maurer, Andreas and Ciliberto, Carlo and Rosasco, Lorenzo and Pontil, Massimiliano},
  journal={Advances in Neural Information Processing Systems},
  volume={35},
  pages={4017--4031},
  year={2022}
}

@article{kevrekidis2016kernel,
  title={A kernel-based method for data-driven Koopman spectral analysis},
  author={Kevrekidis, I and Rowley, Clarence W and Williams, M},
  journal={Journal of Computational Dynamics},
  volume={2},
  number={2},
  pages={247--265},
  year={2016}
}

@article{kawahara2016dynamic,
  title={Dynamic mode decomposition with reproducing kernels for Koopman spectral analysis},
  author={Kawahara, Yoshinobu},
  journal={Advances in neural information processing systems},
  volume={29},
  year={2016}
}

@article{klus2015numerical,
  title={On the numerical approximation of the Perron-Frobenius and Koopman operator},
  author={Klus, Stefan and Koltai, P{\'e}ter and Sch{\"u}tte, Christof},
  journal={arXiv preprint arXiv:1512.05997},
  year={2015}
}

@article{mezic2004comparison,
  title={Comparison of systems with complex behavior},
  author={Mezi{\'c}, Igor and Banaszuk, Andrzej},
  journal={Physica D: Nonlinear Phenomena},
  volume={197},
  number={1-2},
  pages={101--133},
  year={2004},
  publisher={Elsevier}
}

@article{mezic2016comparison,
  title={On comparison of dynamics of dissipative and finite-time systems using koopman operator methods},
  author={Mezic, Igor},
  journal={IFAC-PapersOnLine},
  volume={49},
  number={18},
  pages={454--461},
  year={2016},
  publisher={Elsevier}
}

@article{georgiou2007distances,
  title={Distances and Riemannian metrics for spectral density functions},
  author={Georgiou, Tryphon T},
  journal={IEEE Transactions on Signal Processing},
  volume={55},
  number={8},
  pages={3995--4003},
  year={2007},
  publisher={IEEE}
}

@incollection{afsari2014distances,
  title={Distances on spaces of high-dimensional linear stochastic processes: A survey},
  author={Afsari, Bijan and Vidal, Ren{\'e}},
  booktitle={Geometric Theory of Information},
  pages={219--242},
  year={2014},
  publisher={Springer}
}

@article{vishwanathan2007binet,
  title={Binet-Cauchy kernels on dynamical systems and its application to the analysis of dynamic scenes},
  author={Vishwanathan, SVN and Smola, Alexander J and Vidal, Ren{\'e}},
  journal={International Journal of Computer Vision},
  volume={73},
  number={1},
  pages={95--119},
  year={2007},
  publisher={Springer}
}

@article{martin2000metric,
  title={A metric for ARMA processes},
  author={Martin, Richard J},
  journal={IEEE transactions on Signal Processing},
  volume={48},
  number={4},
  pages={1164--1170},
  year={2000},
  publisher={IEEE}
}

@inproceedings{chaudhry2013initial,
  title={Initial-state invariant binet-cauchy kernels for the comparison of linear dynamical systems},
  author={Chaudhry, Rizwan and Vidal, Ren{\'e}},
  booktitle={52nd IEEE Conference on Decision and Control},
  pages={5377--5384},
  year={2013},
  organization={IEEE}
}

@article{ishikawa2018metric,
  title={Metric on nonlinear dynamical systems with perron-frobenius operators},
  author={Ishikawa, Isao and Fujii, Keisuke and Ikeda, Masahiro and Hashimoto, Yuka and Kawahara, Yoshinobu},
  journal={Advances in Neural Information Processing Systems},
  volume={31},
  year={2018}
}

@inproceedings{germain2026spectral,
  title={A Spectral-Grassmann Wasserstein metric for operator representations of dynamical systems},
  author={Germain, Thibaut and Flamary, R{\'e}mi and Kostic, Vladimir and Lounici, Karim},
  booktitle={International Conference on Learning Representations},
  volume={2026},
  pages={119652--119688},
  year={2026}
}

@inproceedings{huang2016sparse,
  title={Sparse coding and dictionary learning with linear dynamical systems},
  author={Huang, Wenbing and Sun, Fuchun and Cao, Lele and Zhao, Deli and Liu, Huaping and Harandi, Mehrtash},
  booktitle={Proceedings of the IEEE Conference on Computer Vision and Pattern Recognition},
  pages={3938--3947},
  year={2016}
}

@article{tovsic2011dictionary,
  title={Dictionary learning},
  author={To{\v{s}}i{\'c}, Ivana and Frossard, Pascal},
  journal={IEEE Signal Processing Magazine},
  volume={28},
  number={2},
  pages={27--38},
  year={2011},
  publisher={IEEE}
}

@inproceedings{mairal2009online,
  title={Online dictionary learning for sparse coding},
  author={Mairal, Julien and Bach, Francis and Ponce, Jean and Sapiro, Guillermo},
  booktitle={Proceedings of the 26th annual international conference on machine learning},
  pages={689--696},
  year={2009}
}

@article{mairal2011task,
  title={Task-driven dictionary learning},
  author={Mairal, Julien and Bach, Francis and Ponce, Jean},
  journal={IEEE transactions on pattern analysis and machine intelligence},
  volume={34},
  number={4},
  pages={791--804},
  year={2011},
  publisher={IEEE}
}

@article{olshausen1996emergence,
  title={Emergence of simple-cell receptive field properties by learning a sparse code for natural images},
  author={Olshausen, Bruno A and Field, David J},
  journal={Nature},
  volume={381},
  number={6583},
  pages={607--609},
  year={1996},
  publisher={Nature Publishing Group UK London}
}

@inproceedings{van2012kernel,
  title={Kernel dictionary learning},
  author={Van Nguyen, Hien and Patel, Vishal M and Nasrabadi, Nasser M and Chellappa, Rama},
  booktitle={2012 IEEE International Conference on Acoustics, Speech and Signal Processing (ICASSP)},
  pages={2021--2024},
  year={2012},
  organization={IEEE}
}

@inproceedings{ho2013nonlinear,
  title={On a nonlinear generalization of sparse coding and dictionary learning},
  author={Ho, Jeffrey and Xie, Yuchen and Vemuri, Baba},
  booktitle={International conference on machine learning},
  pages={1480--1488},
  year={2013},
  organization={PMLR}
}

@article{schmitz2018wasserstein,
  title={Wasserstein dictionary learning: Optimal transport-based unsupervised nonlinear dictionary learning},
  author={Schmitz, Morgan A and Heitz, Matthieu and Bonneel, Nicolas and Ngole, Fred and Coeurjolly, David and Cuturi, Marco and Peyr{\'e}, Gabriel and Starck, Jean-Luc},
  journal={SIAM Journal on Imaging Sciences},
  volume={11},
  number={1},
  pages={643--678},
  year={2018},
  publisher={SIAM}
}

@article{van2013design,
  title={Design of non-linear kernel dictionaries for object recognition},
  author={Van Nguyen, Hien and Patel, Vishal M and Nasrabadi, Nasser M and Chellappa, Rama},
  journal={IEEE Transactions on Image Processing},
  volume={22},
  number={12},
  pages={5123--5135},
  year={2013},
  publisher={IEEE}
}

@inproceedings{harandi2013dictionary,
  title={Dictionary learning and sparse coding on Grassmann manifolds: An extrinsic solution},
  author={Harandi, Mehrtash and Sanderson, Conrad and Shen, Chunhua and Lovell, Brian C},
  booktitle={Proceedings of the IEEE international conference on computer vision},
  pages={3120--3127},
  year={2013}
}

@article{cherian2016riemannian,
  title={Riemannian dictionary learning and sparse coding for positive definite matrices},
  author={Cherian, Anoop and Sra, Suvrit},
  journal={IEEE transactions on neural networks and learning systems},
  volume={28},
  number={12},
  pages={2859--2871},
  year={2016},
  publisher={IEEE}
}

@inproceedings{li2013log,
  title={Log-Euclidean kernels for sparse representation and dictionary learning},
  author={Li, Peihua and Wang, Qilong and Zuo, Wangmeng and Zhang, Lei},
  booktitle={Proceedings of the IEEE international conference on computer vision},
  pages={1601--1608},
  year={2013}
}

@inproceedings{harandi2012sparse,
  title={Sparse coding and dictionary learning for symmetric positive definite matrices: A kernel approach},
  author={Harandi, Mehrtash T and Sanderson, Conrad and Hartley, Richard and Lovell, Brian C},
  booktitle={European conference on computer vision},
  pages={216--229},
  year={2012},
  organization={Springer}
}

@book{lasota2013chaos,
  title={Chaos, fractals, and noise: stochastic aspects of dynamics},
  author={Lasota, Andrzej and Mackey, Michael C},
  year={2013},
  publisher={Springer Science \& Business Media}
}

@book{parzen1999stochastic,
  title={Stochastic processes},
  author={Parzen, Emanuel},
  year={1999},
  publisher={SIAM}
}

@article{brunton2021modern,
  title={Modern Koopman theory for dynamical systems},
  author={Brunton, Steven L and Budi{\v{s}}i{\'c}, Marko and Kaiser, Eurika and Kutz, J Nathan},
  journal={arXiv preprint arXiv:2102.12086},
  year={2021}
}

@inproceedings{liu2024physics,
  title={Physics-informed Koopman network for time-series prediction of dynamical systems},
  author={Liu, Yuying and Sholokhov, Aleksei and Mansour, Hassan and Nabi, Saleh},
  booktitle={ICLR 2024 Workshop on AI4DifferentialEquations In Science},
  year={2024}
}

@inproceedings{kostic2024learning,
  title={Learning invariant representations of time-homogeneous stochastic dynamical systems},
  author={Kostic, Vladimir and Novelli, Pietro and Grazzi, Riccardo and Lounici, Karim and others},
  booktitle={International Conference on Learning Representations},
  volume={2024},
  pages={54329--54341},
  year={2024}
}

@inproceedings{avron2017random,
  title={Random Fourier features for kernel ridge regression: Approximation bounds and statistical guarantees},
  author={Avron, Haim and Kapralov, Michael and Musco, Cameron and Musco, Christopher and Velingker, Ameya and Zandieh, Amir},
  booktitle={International conference on machine learning},
  pages={253--262},
  year={2017},
  organization={PMLR}
}

@misc{tokam2d,
  title = {{Tokam2D}: A 2D spectral solver for turbulence schemes},
  author = {{GYSELAX Team}},
  year = {2026},
  howpublished = {\url{https://github.com/gyselax/tokam2d}},
  note = {Accessed: 2026-05-01}
}

@inproceedings{ghendrih2018generation,
  title={Generation and dynamics of SOL corrugated profiles},
  author={Ghendrih, Philippe and Asahi, Y and Caschera, E and Dif-Pradalier, Guilhem and Donnel, P and Garbet, X and Gillot, C and Grandgirard, V and Latu, G and Sarazin, Y and others},
  booktitle={Journal of Physics: Conference Series},
  volume={1125},
  number={1},
  pages={012011},
  year={2018},
  organization={IOP Publishing}
}

@inproceedings{ghendrih2022role,
  title={Role of avalanche transport in competing drift wave and interchange turbulence},
  author={Ghendrih, Philippe and Dif-Pradalier, Guilhem and Panico, Olivier and Sarazin, Yanick and Bufferand, Hugo and Ciraolo, Guido and Donnel, Peter and Fedorczak, Nicolas and Garbet, Xavier and Grandgirard, Virginie and others},
  booktitle={Journal of Physics: Conference Series},
  volume={2397},
  number={1},
  pages={012018},
  year={2022},
  organization={IOP Publishing}
}

@article{kostic2024neural,
  title={Neural conditional probability for uncertainty quantification},
  author={Kostic, Vladimir R and Lounici, Karim and Pacreau, Gr{\'e}goire and Turri, Giacomo and Novelli, Pietro and Pontil, Massimiliano},
  journal={Advances in Neural Information Processing Systems},
  volume={37},
  pages={60999--61039},
  year={2024}
}

@inproceedings{chen2018gradnorm,
  title={Gradnorm: Gradient normalization for adaptive loss balancing in deep multitask networks},
  author={Chen, Zhao and Badrinarayanan, Vijay and Lee, Chen-Yu and Rabinovich, Andrew},
  booktitle={International conference on machine learning},
  pages={794--803},
  year={2018},
  organization={PMLR}
}

@article{herde2024poseidon,
  title={Poseidon: Efficient foundation models for pdes},
  author={Herde, Maximilian and Raoni{\'c}, Bogdan and Rohner, Tobias and K{\"a}ppeli, Roger and Molinaro, Roberto and De Bezenac, Emmanuel and Mishra, Siddhartha},
  journal={Advances in Neural Information Processing Systems},
  volume={37},
  pages={72525--72624},
  year={2024}
}

@book{absil2008optimization,
  title={Optimization algorithms on matrix manifolds},
  author={Absil, P-A and Mahony, Robert and Sepulchre, Rodolphe},
  year={2008},
  publisher={Princeton University Press}
}

@article{boumal2020introduction,
  title={An introduction to optimization on smooth manifolds},
  author={Boumal, Nicolas},
  journal={To appear with Cambridge University Press,(Apr 2022)},
  year={2020}
}

@article{glashoff2016optimization,
  title={Optimization on the biorthogonal manifold},
  author={Glashoff, Klaus and Bronstein, Michael M},
  journal={arXiv preprint arXiv:1609.04161},
  year={2016}
}

@article{absil2012projection,
  title={Projection-like retractions on matrix manifolds},
  author={Absil, P-A and Malick, J{\'e}r{\^o}me},
  journal={SIAM Journal on Optimization},
  volume={22},
  number={1},
  pages={135--158},
  year={2012},
  publisher={SIAM}
}

@article{kostic2024laplace,
  title={Laplace transform based low-complexity learning of continuous Markov semigroups},
  author={Kostic, Vladimir R and Lounici, Karim and Halconruy, H{\'e}l{\`e}ne and Devergne, Timoth{\'e}e and Novelli, Pietro and Pontil, Massimiliano},
  journal={arXiv preprint arXiv:2410.14477},
  year={2024}
}

@article{boeuf2023physics,
  title={Physics and instabilities of low-temperature E$\times$ B plasmas for spacecraft propulsion and other applications},
  author={Boeuf, Jean-Pierre and Smolyakov, Andrei},
  journal={Physics of Plasmas},
  volume={30},
  number={5},
  year={2023},
  publisher={AIP Publishing}
}

@inproceedings{coosemans2021bayesian,
  title={Bayesian analysis of turbulent transport coefficients in 2D interchange dominated ExB turbulence involving flow shear},
  author={Coosemans, Reinart and Dekeyser, Wouter and Baelmans, Martine},
  booktitle={Journal of Physics: Conference Series},
  volume={1785},
  number={1},
  pages={012001},
  year={2021},
  organization={IOP Publishing}
}

@article{faraji2024dynamic,
  title={Dynamic mode decomposition for data-driven analysis and reduced-order modeling of E$\times$ B plasmas: I. Extraction of spatiotemporally coherent patterns},
  author={Faraji, Farbod and Reza, Maryam and Knoll, Aaron and Kutz, J Nathan},
  journal={Journal of Physics D: Applied Physics},
  volume={57},
  number={6},
  pages={065201},
  year={2024},
  publisher={IOP Publishing}
}

@article{taylor2018dynamic,
  title={Dynamic mode decomposition for plasma diagnostics and validation},
  author={Taylor, Roy and Kutz, J Nathan and Morgan, Kyle and Nelson, Brian A},
  journal={Review of Scientific Instruments},
  volume={89},
  number={5},
  year={2018},
  publisher={AIP Publishing}
}

\appendix
\appendix

\section{Operator-Theoretic Foundations}
\label{app:operators}

\subsection{Operator representations and estimation}
\label{app:operators-repr}

\paragraph{Markov semigroup and transfer operator.}
Let $(X_t)_{t \geq 0}$ be a time-homogeneous Markov process on a measurable state space $\X$, and let $\mcF \subset \Lii $ be a Hilbert space of observables. The evolution of the system is described by the Markov semigroup $(\TO_t)_{t \geq 0}$ acting on $\mcF$, defined by
\begin{equation}
    [\TO_t f](x) = \mathbb{E}[f(X_t)\mid X_0 = x], \quad f \in \mcF.
\end{equation}
The family $(\TO_t)_{t\geq 0}$ is a strongly continuous semigroup satisfying
\[
\TO_0 = \mathrm{Id}, \qquad \TO_{t+s} = \TO_t \TO_s, \quad \forall t,s \geq 0.
\]

\paragraph{Infinitesimal generator.}
The infinitesimal generator $G$ associated with $(\TO_t)$ is defined as
\begin{equation}
    G f = \lim_{t \to 0^+} \frac{\TO_t f - f}{t},
    \label{eq:generator_appendix}
\end{equation}
with domain $\mathcal{D}(G) \subset \mcF$. The generator is, in general, an unbounded linear operator on $\mcF$ and characterizes the semigroup through $\TO_t = \exp(tG)$.

\paragraph{Spectral decomposition via compact resolvent.}
The spectral analysis of $G$ requires additional structure due to its infinite-dimensional and unbounded nature. Let $\pi$ be an invariant measure of the process and consider the Hilbert space $\Lii$.

We assume that $G$ has compact resolvent, i.e., there exists $\mu_0 \in \rho(G)$ such that $(\mu_0 \Id - G)^{-1}$ is compact on $\Lii$. Under this assumption, the spectrum of $G$ is discrete and consists of isolated eigenvalues $(\lambda_i)_{i\in\mathbb{N}}$ with finite multiplicity and no accumulation point except possibly at infinity.

If, in addition, $G$ is self-adjoint, there exists an orthonormal basis $(f_i)_{i\in\mathbb{N}}$ of $\Lii$ such that
\begin{equation}
    G f_i = \lambda_i f_i .
\end{equation}

For each isolated eigenvalue $\lambda_i$, the associated spectral projector $P_i$ is defined as the orthogonal projector onto the corresponding eigenspace in $\Lii$.

In the simple eigenvalue case, we have
\begin{equation}
    P_i f = \langle f, f_i\rangle_{\Lii} f_i .
\end{equation}

More generally, if $\lambda_i$ has multiplicity $m_i$ and $(f_{i,k})_{k=1}^{m_i}$ is an orthonormal basis of the eigenspace, then
\begin{equation}
    P_i f
    =
    \sum_{k=1}^{m_i}
    \langle f, f_{i,k}\rangle_{\Lii} f_{i,k}.
\end{equation}

The projectors satisfy
\begin{equation}
    P_i P_j = \delta_{ij} P_i,
    \qquad
    \sum_{i} P_i = \Id
\end{equation}
on the spectral subspace.

With this notation, the generator admits the spectral representation
\begin{equation}
    G f = \sum_{i} \lambda_i P_i f,
    \label{eq:spectral_appendix}
\end{equation}
for all $f \in \mathcal{D}(G)$ for which the expansion is well-defined, and the semigroup diagonalizes as
\begin{equation}
    \TO_t f
    =
    \sum_{i}
    e^{\lambda_i t} P_i f .
\end{equation}

\paragraph{Spectral perturbation under RKHS estimation.}

The true spectral projectors are defined in the ambient space $\Lii$, whereas data-driven estimators are learned on an RKHS $\RKHS$ whose geometry generally differs from that of $\Lii$. We therefore do not compare the generator $G$ directly with its estimator in the $\Lii \to \Lii$ operator norm. Instead, following the spectral perturbation framework of \cite{kostic2023sharp}, we compare the resolvent of the generator after embedding the RKHS into the appropriate ambient space.

Let $\mu>0$ and write
\begin{equation}
    R_\mu := (\mu \Id - G)^{-1}.
    \label{eq:resolvent_def_appendix}
\end{equation}
Let $\TZ:\RKHS\to \Lii$ denote the canonical injection into $\Lii$, and let
\begin{equation}
    \widehat R_\mu : \RKHS \to \RKHS
\end{equation}
be an estimator of the resolvent constructed from data. We measure the operator error by
\begin{equation}
    \error(\widehat R_\mu)
    :=
    \| R_\mu \TZ - \TZ \widehat R_\mu \|_{\RKHS\to\Lii}.
    \label{eq:appendix_operator_error}
\end{equation}

Since spectral perturbation is evaluated in the ambient geometry, one must also control the distortion between the RKHS norm and the ambient norm. For $h\in\RKHS$, define
\begin{equation}
    \metdist(h)
    :=
    \frac{\|h\|_{\RKHS}}{\|\TZ h\|_{\Lii}} .
    \label{eq:appendix_metric_distortion}
\end{equation}

Let $\widehat R_\mu$ admit spectral pairs
\begin{equation}
    \widehat R_\mu \,\widehat h_i
    =
    \widehat \rho_i \widehat h_i,
    \label{eq:resolvent_emp_spectrum}
\end{equation}
where $\widehat \rho_i$ approximates $(\mu-\lambda_i)^{-1}$. After embedding and normalization, define
\begin{equation}
    \widehat f_i
    :=
    \frac{\TZ \widehat h_i}{\|\TZ \widehat h_i\|_{\Lii}} .
    \label{eq:embedded_eigenfunction}
\end{equation}

If the corresponding eigenvalue $\rho_i=(\mu-\lambda_i)^{-1}$ of $R_\mu$ is isolated with spectral gap $\gap_i$, then Davis--Kahan-type perturbation arguments yield
\begin{equation}
    |\rho_i-\widehat \rho_i|
    \;\leq\;
    \error(\widehat R_\mu)\,\metdist(\widehat h_i),
    \label{eq:resolvent_eval_perturb}
\end{equation}
and
\begin{equation}
    \|\widehat f_i-f_i\|^2_{\Lii}
    \;\leq\;
    \frac{
        2\,\error(\widehat R_\mu)\,\metdist(\widehat h_i)
    }{
        [\gap_i-\error(\widehat R_\mu)\,\metdist(\widehat h_i)]_+
    } .
    \label{eq:appendix_evec_perturb}
\end{equation}

Finally, generator eigenvalues are recovered via
\begin{equation}
    \widehat \lambda_i
    =
    \mu - \widehat \rho_i^{-1}.
    \label{eq:generator_recovery_appendix}
\end{equation}

Thus, spectral estimation is controlled by two factors: the embedded operator error $\error(\widehat R_\mu)$ and the metric distortion $\metdist(\widehat h_i)$.

\paragraph{Learning initial generators via spectral filtering.}

The first step consists in learning a collection of initial generator atoms directly from trajectory data by estimating spectral components of the infinitesimal generator \(\Estim\). 
Rather than approximating \(\Estim\) through finite differences of the form \((A_{\Delta t}-I)/\Delta t\), which is unstable at small time scales, we rely on spectral filtering based on Toeplitz representations of analytic functions of the generator.

Let \(A_{\Delta t} = e^{\Delta t \Estim}\) denote the transfer operator at time step \(\Delta t\). 
We consider a Toeplitz symbol
\begin{equation}
    T(z) = \sum_{j\in\mathbb{Z}} a_j z^j,
    \label{eq:toeplitz_symbol_step1}
\end{equation}
and define the associated filtered operator
\begin{equation}
    F(\Estim) = T(A_{\Delta t}).
    \label{eq:functional_calculus_step1}
\end{equation}

Since \(A_{j\Delta t} = A_{\Delta t}^j\), the operator \(F(\Estim)\) admits the representation
\begin{equation}
    F(\Estim) = \sum_{j\in\mathbb{Z}} a_j A_{\Delta t}^j,
    \label{eq:toeplitz_operator_expansion_step1}
\end{equation}
which can be approximated from data using time-lagged observations. In practice, we use a truncated expansion
\begin{equation}
    F_\ell(\Estim) = \sum_{|j|\leq \ell} a_j A_{\Delta t}^j,
    \label{eq:truncated_toeplitz_step1}
\end{equation}
whose empirical action is represented by a banded Toeplitz matrix acting on the time-ordered data sequence. 
Thus, analytic functional calculus of the generator reduces to structured linear algebra on trajectory data.

\vspace{0.3em}
\noindent
\textbf{Resolvent and exponential filters.}
To probe the spectrum of the generator \(\Estim\), we use families of analytic filters.

The key object is the generator resolvent
\begin{equation}
    R_\mu = (\mu I - \Estim)^{-1}
    = \int_0^\infty e^{t \Estim} e^{-\mu t}\, dt,
    \label{eq:resolvent_definition_step1}
\end{equation}
which we approximate by a Toeplitz-weighted sum of transfer operators,
\begin{equation}
    R_{\mu,\ell} = \sum_{j=0}^{\ell} a_j A_{\Delta t}^j,
    \qquad
    a_j \approx \Delta t\, e^{-\mu j\Delta t},
    \label{eq:resolvent_toeplitz_step1}
\end{equation}
corresponding to a discretization of the Laplace transform.

Equivalently, using the transfer-operator resolvent, we have
\begin{equation}
    (e^\mu I - A_{\Delta t})^{-1}
    =
    \sum_{j=0}^{\infty} e^{-(j+1)\mu} A_{\Delta t}^j,
    \label{eq:transfer_resolvent_step1}
\end{equation}
which naturally yields Toeplitz coefficients. These filters concentrate spectral information near the shift parameter \(\mu\), allowing us to localize different regions of the spectrum.

In addition, exponential and trigonometric filters (e.g. \(e^{\Delta t \Estim}\), \(\cosh(\Delta t \Estim)\), \(\sinh(\Delta t \Estim)\)) can be used to emphasize specific spectral structures depending on the nature of the dynamics.

\vspace{0.3em}
\noindent
\textbf{Spectral representation learning.}
For each filter \(F_q(\Estim)\), indexed by a parameter \(q\) such as a resolvent shift \(\mu\), we do not first estimate \(F_q(\Estim)\) as an operator on the ambient space. Instead, we use the Toeplitz filter to learn a low-dimensional latent spectral representation.

Given the Toeplitz coefficients associated with \(F_q\), we form the Toeplitz-weighted lagged data operator. In the finite-dimensional representation this gives a filtered matrix \(W_q\), while in the sample representation it gives the corresponding Toeplitz-filtered Gram operator. The latent spectral space is then obtained from the leading solutions of the generalized eigenproblem
\begin{equation}
    W_q W_q^\top v_{q,k}
    =
    \sigma_{q,k}^2 C_{\gamma} v_{q,k},
    \qquad k=1,\ldots,r_q,
    \label{eq:latent_gep_step1}
\end{equation}
or equivalently from its dual version. We denote the resulting latent space by
\begin{equation}
    \mathcal U_q
    =
    \operatorname{span}
    \{v_{q,1},\ldots,v_{q,r_q}\}.
    \label{eq:latent_space_step1}
\end{equation}

Only after this latent space has been learned do we represent the filtered dynamics on it. The empirical compressed filter is
\begin{equation}
    \widehat F_q^{(r)}
    =
    V_q^\top W_q V_q,
    \label{eq:compressed_filter_step1}
\end{equation}
where \(V_q=[v_{q,1},\ldots,v_{q,r_q}]\). Its spectral decomposition
\begin{equation}
    \widehat F_q^{(r)} w_{q,k}
    =
    \widehat \nu_{q,k} w_{q,k}
    \label{eq:compressed_filter_eigenproblem_step1}
\end{equation}
yields the filtered spectral coordinates. The corresponding generator atom is obtained by inverting the scalar filter,
\begin{equation}
    \widehat \lambda_{q,k}
    =
    F_q^{-1}(\widehat \nu_{q,k}).
    \label{eq:generator_atom_inversion_step1}
\end{equation}
For the resolvent filter \(F_q(\lambda)=(\mu-\lambda)^{-1}\), this gives
\begin{equation}
    \widehat \lambda_{q,k}
    =
    \mu-\widehat \nu_{q,k}^{-1}.
    \label{eq:resolvent_atom_inversion_step1}
\end{equation}

Thus, Step 1 learns generator atoms by first learning a latent spectral representation induced by the Toeplitz-filtered trajectory, and then diagonalizing the compressed filtered operator in that learned latent space.

\vspace{0.3em}
\noindent
\textbf{Latent generator atoms.}
The outcome of this step is a collection of spectral atoms
\begin{equation}
    \mathcal{A}_0
    =
    \left\{
        \bigl(\widehat \lambda_{q,k}, \widehat \psi_{q,k}, q\bigr)
        :
        q \in \mathcal{Q},\; k \le r_q
    \right\},
    \label{eq:atoms_step1}
\end{equation}
which define a latent spectral representation of the generator. 
Each atom corresponds to a localized spectral component obtained through a specific filter \(q\).

These atoms provide a structured, low-dimensional parametrization of the generator spectrum and its associated modes. 
They serve as the input for the second step, where dictionary learning is performed on the manifold generated by these initial spectral representations.

\paragraph{Related work on operator estimation.}
Estimating evolution operators from data is challenging, as dynamical systems are typically observed through discrete trajectories, and neither the generator $G$ nor its domain is known. A common approach consists in learning the action of the time-$\delta_t$ operator $\TO = \exp(\delta_t G)$ on a predefined Reproducing Kernel Hilbert Space (RKHS) $\mcH$, yielding an estimator of the projected operator $\tilde{A} = P_{\RKHS} \TO_{\vert_{\RKHS}}$. This is typically achieved via empirical risk minimization with low-rank constraints~\cite{kawahara2016dynamic,kostic2022learning}, and is closely related to kernel-based variants of dynamic mode decomposition~\cite{brunton2021modern}.

An alternative approach consists in learning finite-dimensional function spaces, for instance parameterized by neural networks, in which the operator is well approximated~\cite{liu2024physics,kostic2024learning}. These methods aim at minimizing projection errors of the form $\|P_{\RKHS}^\bot \TO_{\vert_{\RKHS}}\|$, thereby adapting the representation space to the underlying dynamics.

More recently, spectral approaches based on analytic functional calculus have been proposed, where one does not directly approximate the generator or the transfer operator, but rather bounded transforms such as resolvents or Laplace-type operators. In particular, Toeplitz-based spectral methods~\cite{kostic2026toeplitz} leverage structured combinations of time-lagged observations to approximate analytic functions of the generator, enabling stable recovery of spectral quantities through finite-dimensional projections.

While such approaches provide consistent estimators under suitable assumptions, they typically operate in high-dimensional function spaces and exhibit nonparametric statistical rates, where $n$ denotes the number of samples~\cite{kostic2023sharp}. This can limit their reliability in low-data regimes. Moreover, the resulting spectral decompositions lie on non-linear manifolds, making subsequent comparison and learning tasks non-trivial.

\paragraph{Learning transfer operators in RKHS.}
Let $\Data=(x_i,y_i)_{i=1}^n$ be sampled from a stationary Markov process, with
$y_i \sim P(\cdot \mid x_i)$. Let $\phi:\X\to\RKHS$ be the canonical feature map and define
\[
    \Phi_X := [\phi(x_1),\ldots,\phi(x_n)],
    \qquad
    \Phi_Y := [\phi(y_1),\ldots,\phi(y_n)] .
\]
The empirical covariance and cross-covariance operators are
\[
    \ECx=\frac1n\Phi_X\Phi_X^\ast,
    \qquad
    \ECxy=\frac1n\Phi_X\Phi_Y^\ast .
\]

For $\reg>0$ and rank $r$, we estimate the transfer operator $A$ by the regularized reduced-rank estimator
\begin{equation}
    \widehat{A}_{r,\reg}
    \in
    \argmin_{\operatorname{rank}(A)\le r}
    \left\{
        \frac1n\sum_{i=1}^n
        \|A^\ast\phi(x_i)-\phi(y_i)\|_{\RKHS}^2
        +
        \reg \|A\|_{\HSs}^2
    \right\}.
    \label{eq:rrr_ridge_primal}
\end{equation}

Equivalently, we first form the ridge-stabilized regression operator
\begin{equation}
    \widehat A_{\reg}
    :=
    (\ECx+\reg \Id)^{-1}\ECxy ,
    \label{eq:ridge_operator}
\end{equation}
and then retain its leading $r$ spectral components. The regularization parameter $\reg$ controls stability, while $r$ controls the approximation rank.

\paragraph{Dual representation.}
In practice, the estimator is computed through Gram matrices
\[
    K_{XX}:=\Phi_X^\ast\Phi_X,
    \qquad
    K_{YY}:=\Phi_Y^\ast\Phi_Y,
    \qquad
    K_{XY}:=\Phi_X^\ast\Phi_Y .
\]
Define
\begin{equation}
    M_{\reg}
    :=
    (K_{XX}+n\reg I)^{-1/2}
    K_{XY}
    (K_{YY}+n\reg I)^{-1/2}.
    \label{eq:rrr_dual_matrix}
\end{equation}
If
\[
    M_{\reg}
    =
    \sum_{j\ge 1}
    \widehat\sigma_j
    \widehat u_j \widehat v_j^\top ,
\]
its rank-$r$ truncation is
\[
    M_{\reg,r}
    :=
    \sum_{j=1}^r
    \widehat\sigma_j
    \widehat u_j \widehat v_j^\top .
\]
The corresponding RKHS operator is
\begin{equation}
    \widehat{A}_{r,\reg}
    =
    \Phi_X
    (K_{XX}+n\reg I)^{-1/2}
    M_{\reg,r}
    (K_{YY}+n\reg I)^{-1/2}
    \Phi_Y^\ast .
    \label{eq:rrr_kernel_operator}
\end{equation}
All computations are thus performed via kernel evaluations and an $n\times n$ SVD.

\paragraph{Error decomposition.}
Let $A^\star_{r,\reg}$ denote the population counterpart. Then
\begin{equation}
    \error(\widehat{A}_{r,\reg})
    :=
    \| A\TZ - \TZ \widehat{A}_{r,\reg} \|_{\RKHS\to\Lii}
\end{equation}
admits the decomposition
\begin{equation}
    \error(\widehat{A}_{r,\reg})
    \;\lesssim\;
    \underbrace{
        \|A\TZ-\TZ A^\star_{r,\reg}\|_{\RKHS\to\Lii}
    }_{\text{approximation error}}
    +
    \underbrace{
        \|\TZ(A^\star_{r,\reg}-\widehat{A}_{r,\reg})\|_{\RKHS\to\Lii}
    }_{\text{statistical error}}.
    \label{eq:rrr_error_decomposition}
\end{equation}

\paragraph{Finite-rank approximation.}
Let
\begin{equation}
    A \TZ
    =
    \sum_{j\ge 1}
    \sigma_j \, u_j \otimes v_j
    \label{eq:A_svd}
\end{equation}
be the singular value decomposition of the embedded transfer operator. Then
\begin{equation}
    \inf_{\operatorname{rank}(B)\le r}
    \| A \TZ - B \|_{\RKHS\to\Lii}
    =
    \sigma_{r+1}(A\TZ),
    \label{eq:finite_rank_error}
\end{equation}
and the optimal approximation is obtained by truncation.

Under the regularity condition
\begin{equation}
    \Cxy\Cxy^\ast \preceq \rcon^2 \Cx^{1+\rpar},
    \qquad \rpar\in(0,2],
    \label{eq:rpar}
\end{equation}
and the spectral decay condition
\begin{equation}
    \lambda_j(\Cx) \leq \scon j^{-1/\spar},
    \qquad \spar\in(0,1],
    \label{eq:spar}
\end{equation}
the approximation error satisfies
\begin{equation}
    \|A\TZ-\TZ A^\star_{r,\reg}\|_{\RKHS\to\Lii}
    \;\lesssim\;
    \reg^{\rpar/2}
    +
    \sigma_{r+1}(A\TZ).
    \label{eq:rrr_approximation_bound}
\end{equation}

\paragraph{Statistical error.}
With probability at least $1-\delta$,
\begin{equation}
    \|\TZ(A^\star_{r,\reg}-\widehat{A}_{r,\reg})\|_{\RKHS\to\Lii}
    \;\lesssim\;
    \reg^{-\spar/2} n^{-1/2}\log(\delta^{-1}).
    \label{eq:rrr_statistical_bound}
\end{equation}

Balancing bias and variance yields
\begin{equation}
    \reg \asymp n^{-1/(\rpar+\spar)},
\end{equation}
and therefore
\begin{equation}
    \error(\widehat{A}_{r,\reg})
    \;\lesssim\;
    \sigma_{r+1}(A\TZ)
    +
    n^{-\frac{\rpar}{2(\rpar+\spar)}}\log(\delta^{-1}).
    \label{eq:rrr_nonparametric_rate}
\end{equation}

Thus, the total error decomposes into a finite-rank approximation term controlled by the spectral decay of the embedded transfer operator, and a statistical estimation term governed by $(\rpar,\spar,n)$.

\paragraph{Random Fourier Features approximation.}
When the kernel $k$ is shift-invariant, i.e.\ $k(x,y) = \kappa(x-y)$
for some positive definite function $\kappa$, Bochner's theorem
guarantees the existence of a spectral measure $\Lambda$ such that
\begin{equation}
    k(x,y)
    =
    \mathbb{E}_{\omega \sim \Lambda}
    \bigl[e^{i\omega^\top(x-y)}\bigr]
    =
    \mathbb{E}_{\omega \sim \Lambda, b \sim \mathrm{Unif}([0,2\pi])}
    \bigl[z_\omega(x) z_\omega(y)\bigr],
\end{equation}
where $z_\omega(x) := \sqrt{2}\cos(\omega^\top x + b)$ is a random
feature map~\citep{rahimi2007random}.
Given $p$ i.i.d.\ draws $(\omega_k, b_k)_{k=1}^p$ from
$\Lambda \times \mathrm{Unif}([0,2\pi])$, the random Fourier feature
(RFF) map is
\begin{equation}
    \phi(x)
    :=
    \frac{1}{\sqrt{p}}
    \bigl[z_{\omega_1}(x),\ldots,z_{\omega_p}(x)\bigr]^\top
    \in \mathbb{R}^p,
\end{equation}
so that $\phi(x)^\top\phi(y) \approx k(x,y)$.
The RFF approximation replaces the potentially infinite-dimensional
RKHS $\mcH$ by the finite-dimensional feature space $\mathbb{R}^p$,
with approximation error controlled by McDiarmid's inequality:
with probability at least $1 - \delta$ over the draw of features,
\begin{equation}
    \sup_{x,y \in \mcX}
    |k(x,y) - \phi(x)^\top\phi(y)|
    \;\leq\;
    \sqrt{\frac{2\log(2/\delta)}{p}}.
\end{equation}

\paragraph{RFF in the \DOODL{} framework.}
In our setting, the RFF map provides a natural finite-dimensional
feature space satisfying the boundedness assumption
$\kappa := \sup_x \|\phi(x)\|_2 < \infty$.
Indeed, since $|z_\omega(x)| \leq \sqrt{2}$ for all $x, \omega, b$,
we have
\begin{equation}
    \|\phi(x)\|_2
    =
    \frac{1}{\sqrt{p}}
    \left(\sum_{k=1}^p z_{\omega_k}(x)^2\right)^{1/2}
    \leq \sqrt{2},
    \qquad \forall x \in \mcX,
\end{equation}
so $\kappa = \sqrt{2}$ uniformly in $p$.
This makes the RFF map particularly well-suited for our statistical
guarantees: the feature bound $\kappa$ is dimension-free and does
not degrade as $p$ grows.

The RFF estimator of the transfer operator is obtained by
substituting $\phi$ for the kernel feature map in the
RRR estimator~\eqref{eq:rrr_ridge_primal}, yielding a
finite-dimensional linear regression problem of size $p \times p$.
This construction is consistent with the broader framework of
random features for operator-valued and vector-valued regression
\citep{brault2016random, minh2016operator, lanthaler2023error},
which establishes that RFF-based estimators recover the statistical
rates of their kernel counterparts when the number of features $p$
is chosen appropriately relative to the sample size.
Concretely, with $p$ growing polynomially in $n$, the additional
approximation error introduced by replacing the kernel Gram matrices
with their RFF counterparts does not affect the leading statistical
rate of the estimator \citep{avron2017random, rahimi2007random}.

Finally, in the \DOODL{} framework, the same RFF map $\phi$ is shared across
all $d$ training systems and the test system, ensuring a common
finite-dimensional observable space $\mathbb{R}^p$ as required
by Assumption~(A1).

\begin{rmk}[Bounded activations as an alternative]
When $\phi$ is the last layer of a neural network with bounded
activation (e.g.\ $\arctan$, sigmoid, or $\tanh$), the same
bound $\kappa < \infty$ holds by construction.
In this case all the statistical guarantees of
Appendix~\ref{app:stat} apply without modification,
with $\kappa$ determined by the activation range rather than
the RFF construction.
\end{rmk}

\subsection{Geometry of operators via spectral optimal transport}
\label{app:sgot}

We consider operators only through their finite-rank components. Let $\mcH$ be a separable Hilbert space and fix $r \in \mdN^*$. Denote by $\mcS_r(\mcH)$ the set of non-defective operators of rank at most $r$. In the setting of this work, each operator is replaced by an element $G \in \mcS_r(\mcH)$ capturing its dominant spectral modes. These operators are assumed to lie in a common $r$-dimensional subspace, either by construction or through the estimation procedure described above.

\paragraph{Spectral decomposition as measures over eigen-components.} Let $G \in \mcS_r(\mcH)$ admit the spectral decomposition
\begin{equation}
    G = \sum_{i=1}^{\ell} \lambda_i P_i,
\end{equation}
where $(\lambda_i)_{i\in[\ell]}$ are distinct eigenvalues, $P_i$ are the associated spectral projectors, and $m_i := \mathrm{rank}(P_i)$ their multiplicities, with $m_{\mathrm{tot}} := \sum_{i=1}^{\ell} m_i \leq r$.

We encode the spectral information of $G$ as a discrete probability measure on $\mdC \times \mcE_r(\mcH)$, where $\mcE_r(\mcH) := \{P \in \mcS_r(\mcH) : P^2 = P\}$ denotes the set of oblique projections, by
\begin{equation}
    \mu(G) = \sum_{i=1}^{\ell} \frac{m_i}{m_{\mathrm{tot}}} \, \delta_{(\lambda_i, P_i)}.
\end{equation}

\paragraph{Spectral Optimal Transport between Operators.} Let $d_{\mcE}$ be a divergence on $\mcE_r(\mcH)$. For $\eta \in (0,1)$ and $q \in \mdN^*$, define the ground cost
\begin{equation}
    C_\eta^q\big((\lambda,P),(\lambda',P')\big)
    =
    \eta\, |\lambda - \lambda'|^q
    +
    (1-\eta)\, d^q_{\mcE}(P,P').
\end{equation}

Given $G, G' \in \mcS_r(\mcH)$ with spectral measures
\begin{equation}
    \mu(G) = \sum_{i=1}^{\ell} w_i \, \delta_{(\lambda_i, P_i)},
    \qquad
    \mu(G') = \sum_{j=1}^{\ell'} w'_j \, \delta_{(\lambda'_j, P'_j)},
\end{equation}
their spectral optimal transport (SGOT) divergence is defined as
\begin{equation}
    d_{\mcS}(G,G')
    =
    \min_{\pi \in \Pi(w,w')}
    \sum_{i=1}^{\ell} \sum_{j=1}^{\ell'}
    \pi_{ij} \,
    C_\eta^q\big((\lambda_i,P_i),(\lambda'_j,P'_j)\big),
\end{equation}
where
\begin{equation}
    \Pi(w,w')
    =
    \left\{
        \pi \in \mdR_+^{\ell \times \ell'}
        :
        \sum_{j=1}^{\ell'} \pi_{ij} = w_i,
        \quad
        \sum_{i=1}^{\ell} \pi_{ij} = w'_j
    \right\}.
\end{equation}

The divergence $d_{\mcS}$ compares operators through their spectral components by jointly aligning eigenvalues and spectral projectors. It is invariant under permutations of the spectral decomposition and depends only on the associated invariant subspaces. When $d_{\mcE}$ is a metric and $q=1$, $d_{\mcS}$ coincides with a Wasserstein distance.

In the present work, $G$ and $G'$ are obtained as finite-rank estimators (or truncations) of underlying operators. Consequently, $\mu(G)$ and $\mu(G')$ are supported on finitely many atoms, and the computation of $d_{\mcS}$ reduces to a finite-dimensional optimal transport problem.

Note that $d_{\mcS}$ is in general a divergence, and reduces to a Wasserstein metric when $d_{\mcE}$ is a metric and $q=1$. In this regime, $d_{\mcS}$ defines a well-posed objective for learning operators from data: given a collection of estimated operators, it can be used to fit a shared low-dimensional representation (or dictionary) by minimizing the discrepancy between their spectral measures, thereby aligning their dominant eigenvalues and invariant subspaces in a common latent space.

\paragraph{Metric between spectral projectors.} Suppose two rank-one spectral projector in $\mdC^n$:  $\mbP = \mbu \otimes \mbv $ and $\tilde{\mbP} = \tilde{\mbu} \otimes \tilde{\mbv}$. Let the spectral angle $\theta_s$ between $\mbP$ and $\tilde{\mbP}$ be defined as:
\begin{equation}
    \cos(\theta_s) = \frac{|\innerp{\mbP}{\tilde{\mbP}}|}{\|\mbP\| \|\tilde{\mbP}\|} = \frac{|\innerp{\mbu}{\tilde{\mbu}}\innerp{\tilde{\mbv}}{\mbv}|}{\|\mbu\| \|\tilde{\mbu}\| \|\mbv\|\|\tilde{\mbv}\|}~.
\end{equation}
From spectral angle one can define metrics between rank-one operator detailed in \Cref{tab: grassmann metrics}.

\begin{table}[ht]
\centering
\caption{Metrics between rank-one spectral projectors. $\cos(\theta_s) = \delta$.}
\label{tab: grassmann metrics}
\begin{tabular}{l|c|c}
\toprule
\textbf{Metric} &
\textbf{Angle formulation} &
\textbf{Matrix formulation} \\
\midrule

Geodesic $(\displaystyle d_{\mathrm{geo}})$ &
$\theta_s^2$ &
$\arccos^2(\delta)$\\[2mm]

Chordal $(\displaystyle d_{\mathrm{chord}})$&
$ \sin^2 \theta_s$ &
$ 1 - \delta^2$ \\[2mm]

Procrustes $(\displaystyle d_{\mathrm{proc}})$&
$2(1 - \cos\theta_s)$ &
 $2(1 - \delta)$ \\[2mm]

Log-Martin $(\displaystyle d_{\mathrm{Martin}})$&
$-\log(\cos^2\theta_s)$ &
$ -\log(\delta^2)$ \\[2mm]

\bottomrule
\end{tabular}
\end{table}

\paragraph{Wasserstein barycenters of operators.}
Given a collection of finite-rank operators $(G^{(k)})_{k=1}^N \subset \mcS_r(\mcH)$ with associated spectral measures $(\mu(G^{(k)}))_{k=1}^N$, a natural notion of average operator is provided by the Wasserstein barycenter of these measures under the SGOT geometry. For weights $(\alpha_k)_{k=1}^N$ with $\alpha_k \geq 0$ and $\sum_{k=1}^N \alpha_k = 1$, the barycenter is defined as
\begin{equation}
    \bar{\mu}
    \in
    \argmin_{\nu \in \mcP(\mdC \times \mcE_r(\mcH))}
    \sum_{k=1}^N \alpha_k \, W_{C_\eta^q}(\nu, \mu(G^{(k)})),
\end{equation}
where $W_{C_\eta^q}$ denotes the optimal transport cost induced by the ground metric $C_\eta^q$. When $q=1$ and $d_{\mcE}$ is a metric, this corresponds to a Wasserstein barycenter in the space of spectral measures. In the finite-rank setting, $\bar{\mu}$ is supported on a finite number of atoms and can be written as
\begin{equation}
    \bar{\mu} = \sum_{j=1}^{\bar{\ell}} \bar{w}_j \, \delta_{(\bar{\lambda}_j, \bar{P}_j)},
\end{equation}
which defines a barycentric operator
\begin{equation}
    \bar{G} = \sum_{j=1}^{\bar{\ell}} \bar{\lambda}_j \bar{P}_j \in \mcS_r(\mcH).
\end{equation}
This construction defines an operator-valued average that is consistent with the SGOT geometry: the barycenter $\bar{G}$ minimizes the total transport cost to the input operators, and therefore provides a representative element whose spectral measure is optimally aligned, in the sense of optimal transport, with those of the operators $(G^{(k)})_{k=1}^N$. In particular, the support of $\bar{\mu}$ captures spectral components that best match, under the joint eigenvalue–eigenspace cost, the dominant modes of the input operators. In our framework, such barycenters are used as representative elements (or dictionary atoms) for collections of estimated operators.

\subsection{Dictionary learning on structured spaces}
\label{app:dl}

\paragraph{Dictionary learning as a reconstruction problem.}
Let $(x_i)_{i \in [N]}$ be data points in a space $\mcX$. Dictionary learning seeks a set of atoms
$D = (D_\ell)_{\ell \in [K]} \subset \mcX$ and codes $(\alpha_i)_{i \in [N]}$, with $\alpha_i \in \Delta_K$, such that each data point is approximated through a decoding function $\tB$:
\begin{equation}
    x_i \approx \tB(\alpha_i; D).
\end{equation}
The dictionary and codes are learned by minimizing a reconstruction objective
\begin{equation}
\label{eq:generic_dl}
    \min_{D, (\alpha_i)}
    \sum_{i=1}^N
    \mcL\big(x_i, \tB(\alpha_i; D)\big)
    + \mcR(\alpha_i),
\end{equation}
where $\mcL$ is a data fitting loss and $\mcR$ a regularization term.

A key component of the formulation in \eqref{eq:generic_dl} is the decoding function $\tB$, which specifies how atoms are combined. In Euclidean settings, $\tB$ is typically linear, $\tB(\alpha; D) = \sum_\ell \alpha_\ell D_\ell$. More generally, when $\mcX$ is non-linear, $\tB$ can be defined through barycentric combinations associated with a given geometry.

Dictionary learning was originally introduced in signal processing and computer vision with linear decoders and sparse regularization~\cite{olshausen1996emergence,mairal2009online}, and has since been extended to a wide range of machine learning tasks~\cite{tovsic2011dictionary}. Beyond Euclidean settings, several works have generalized this paradigm to structured spaces by adapting the decoding function $\tB$ to the underlying geometry, including kernel-based methods~\cite{van2012kernel,van2013design}, Riemannian approaches on manifolds~\cite{ho2013nonlinear,harandi2013dictionary}, and formulations based on Wasserstein barycenters for probability measures~\cite{schmitz2018wasserstein}. These extensions highlight that dictionary learning fundamentally depends on the choice of geometry and associated notion of barycentric combination.

\paragraph{Limitations for operator-valued data.}
Several extensions of dictionary learning have been proposed for dynamical systems. In Koopman-based approaches, one may learn a dictionary of observables used to construct finite-dimensional approximations of a single operator, for instance within EDMD frameworks \cite{li2017extended}. These methods learn the feature space on which an operator is represented, but do not learn a dictionary whose atoms are themselves dynamical operators, nor do they enable shared representations across systems.

Another line of work considers dictionary learning for linear dynamical systems (LDS), where both data points and atoms are parametric state-space models \cite{huang2016sparse}. Such approaches rely on Euclidean operations on finite-dimensional parameterizations and are therefore tied to specific model classes.

These constructions do not extend to the setting considered in this work, where data points are dynamical systems represented by their associated operators. In this case, the relevant structure is spectral and operator-theoretic: eigenvalues and invariant subspaces encode the fundamental dynamical modes. Linear combinations in parameter space do not preserve this structure, and therefore do not yield meaningful interpolations between operators.

\paragraph{Operator dictionary learning via spectral optimal transport.}
In our setting, the data are finite-rank operators $(G^{(k)})_{k=1}^N \subset \mcS_r(\mcH)$, each encoding the dominant spectral component of a dynamical system. We define dictionary learning directly in the SGOT geometry.

Let $\mu(G^{(k)})$ denote the spectral measure associated with $G^{(k)}$. Given $K \in \mdN^*$, we seek dictionary atoms $(D_\ell)_{\ell=1}^K \subset \mcS_r(\mcH)$, with associated spectral measures $\mu(D_\ell)$. For a code $\alpha \in \Delta_K$, the reconstruction of an operator is defined through a barycentric combination in the SGOT geometry:
\begin{equation}
    B_\alpha(D_1,\ldots,D_K)
    \in
    \argmin_{G \in \mcS_r(\mcH)}
    \sum_{\ell=1}^K
    \alpha_\ell \,
    d_{\mcS}(G, D_\ell),
    \label{eq:operator_barycentric_code}
\end{equation}
that is, $B_\alpha$ is a Wasserstein barycenter with respect to the SGOT divergence. Equivalently, at the level of spectral measures,
\begin{equation}
    \mu\!\left(B_\alpha(D_1,\ldots,D_K)\right)
    \in
    \argmin_{\nu \in \mcP(\mdC \times \mcE_r(\mcH))}
    \sum_{\ell=1}^K
    \alpha_\ell \,
    d_{\mcS}\big(\nu, \mu(D_\ell)\big).
\end{equation}

The operator dictionary learning problem is then given by
\begin{equation}
\label{eq:operator_dictionary_learning}
    \min_{(D_\ell)_{\ell=1}^K \subset \mcS_r(\mcH)}
    \min_{(\alpha^{(k)})_{k=1}^N \subset \Delta_K}
    \sum_{k=1}^N
    d_{\mcS}
    \left(
        G^{(k)},
        B_{\alpha^{(k)}}(D_1,\ldots,D_K)
    \right).
\end{equation}

Thus, reconstruction is not performed by Euclidean superposition of operators, but by barycentric interpolation in the spectral optimal transport geometry. The learned atoms correspond to representative spectral structures, and the coefficients $\alpha^{(k)}$ provide low-dimensional coordinates describing each system relative to these prototypes.

\section{A Riemannian manifold of spectral decomposition}
\label{app:spectral_decomp_manifold}

We provide a novel characterization of the manifold of spectral decompositions assuming low-rank settings and complex algebra. We also derive the tools required to perform Riemannian gradient-based optimization on this space. 

\subsection{Main results}
Main results rely on preliminary results on manifolds of bi-orthogonal matrices which can be found in \Cref{sec:bi-orthogonal manifold}.

\paragraph{The smooth manifold of spectral decompositions.}
Let $\mcH$ be a complex $\dimH$-dimensional Hilbert space and $\mcS_r(\mcH)$ being the set of non defective operators with rank at most $r$ and simple spectrum. Any $G \in \mcS_r(\mcH)$ admits a spectral decomposition
with matrix representation: 
\begin{equation}
    \mbG = \mbR \Diag(\bs\Lambda) \mbL^* \quad \text{s.t.} \quad \mbL^*\mbR = \mbI_r~,
\end{equation}
where $\bs\Lambda \in \mdC^r$ are the eigenvalues, $(\mbL,\mbR) \in (\mdC^{\dimH \times r})^2$ are the left/right eigenvectors satisfying the bi-orthogonal condition. 

As depicted in \Cref{lmm: b_manifold}, the set of low-rank bi-orthogonal matrices $\mcB = \{(\mbL,\mbR) \in (\mdC^{p \times r})^2\ | \ \mbL^*\mbR = \mbI_r\}$ forms a smooth manifold, hence by product of smooth manifold, the set of all possible spectral decomposition: 
\begin{equation}
    \mcN = \mdC^r \times \mcB =  \{(\bs\Lambda,\mbL,\mbR) \ | \ \mbL^*\mbR = \mbI_r\}~,
\end{equation}
forms a smooth manifold.

However, operators' spectral decomposition are not unique: under the simple spectrum assumption, eigenvectors are defined up to independent rescaling. This induces a natural equivalence relation on spectral decompositions, corresponding to component-wise actions of $(\mdC^*)^r$ on eigenvectors, i.e. $(\mbL,\mbR) \sim (\mbL',\mbR')$ if and only if there exists $\mbz \in (\mdC^*)^r$ such that $(\mbL',\mbR') = (\mbL\Diag(\bar{\mbz}),\mbR\Diag(\mbz^{-1}))$. As depicted in \cref{lmm:quotient_b}, this equivalence relationship defines smooth quotient manifold of bi-orthogonal phase/scale invariant matrices, denoted $\mcB/(\mdC^*)^r$. We therefore view spectral decompositions as elements of a quotient space
\begin{equation}
    \mcM = \mdC^r \times \mcB/(\mdC^*)^r.
\end{equation}
By product of quotient manifold, $\mcM$ is a smooth manifold, and $\mcS_r(\mcH)$ can be identified with $\mcM$.

\paragraph{A Riemannian metric on the manifold of spectral decomposition.}
Since the quotient manifold $\mcB/(\mdC^*)^r$ can be endow with a Riemannian as described in \Cref{prp:quotient_riem} and \Cref{prp:stable_metric}, by product with an Euclidean space, the manifold of spectral decomposition $\mcM$, can be endowed with the Riemannian metric: 
\begin{equation}
        g_{(\bs\Lambda,\mbL,\mbR)}((\bs\lambda,\bs\xi,\bs\zeta),(\bs\lambda',\bs\xi',\bs\zeta')) \triangleq 
        \bs\lambda^*\bs\lambda' +
        \Tr(\bs\xi^*\bs\xi'(\mbL^*\mbL)^{-1}) +  \Tr(\bs\zeta^*\bs\zeta'(\mbR^*\mbR)^{-1})~,
\end{equation} 
where $(\bs\Lambda,\mbL,\mbR) \in \mcM$ and $(\bs\lambda,\bs\xi,\bs\zeta),(\bs\lambda',\bs\xi',\bs\zeta') \in T_{(\bs\Lambda,\mbL,\mbR)}\mcM$ are tangent vectors.
Since the part of the metric related to the eigenvalues is the natural inner product, the necessary tools to perform Riemannian gradient descent (retraction and Riemannian gradient and transport) are just composition of the indentity on the eigenvalue space with tools on $\mcB/(\mdC^*)^r$ defined in \Cref{sec:bi-orthogonal manifold}.

\subsection{Manifolds of bi-orthogonal low-rank matrices and subspaces}
\label{sec:bi-orthogonal manifold}
\subsubsection{Smooth manifold structure of bi-orthogonal low-rank matrices and subspaces}
\label{section: bi-orthogonal matrices}
\begin{lmm}[Smooth manifold of bi-orthogonal low rank matrices.]
Considering real differentiability, the manifold of low-rank bi-orthogonal complex matrices
\begin{equation}
    \label{lmm: b_manifold}
    \mcB = \left\{ (\mbL,\mbR) \in (\mdC^{p\times r})^2 \ | \ \mbL^*\mbR = \mbI_r \right\}~,
\end{equation}
with $r<p$, the manifold is a smooth manifold of dimension $2rp-r^2$ with tangent spaces: 
\begin{equation}
    \mcT_{(\mbL,\mbR)}\mcB  = \left\{ (\bs\xi, \bs\zeta) \in (\mdC^{p\times r})^2 \ | \ 
        \bs\xi^*\mbR + \mbL^*\bs\zeta = \mb0_r \right\}, \quad \forall (\mbL,\mbR) \in  \mcB~.
\end{equation}
\end{lmm}

\begin{proof}
Let consider the application: 
\begin{equation}
\tF: (\mbL, \mbR) \in (\mdC^{p\times r})^2 \mapsto \mbL^*\mbR - \mbI_r \in \mdC^{r\times r}~.
\end{equation} 
As a quadratic function, it is a real-smooth map between two Hilbert spaces. In particular, $\mcB = \tF^{-1}(\mb0_r)$ and since its differential (with Wirtinger notation): 
\begin{equation}
    \tD \tF(\mbL,\mbR)[\bs\xi,\bs\zeta] \triangleq \bs\xi^*\mbR + \mbL^*\bs\zeta
\end{equation}
is surjective for any $(\mbL,\mbR) \in \tF^{-1}(\mb0_r)$, by the preimage theorem, $\mcB = \tF^{-1}(\mb0_r)$ is a smooth manifold with $\dim(\mcB) = 2nr - r^2$, and such that for  any $(\mbL,\mbR) \in \tF^{-1}(\mb0_r)$, its tangent space is defined as:
\begin{equation}
    \mcT_{(\mbL,\mbR)}\mcM  = \left\{ (\bs\xi, \bs\zeta) \in (\mdC^{n\times r})^2 \ | \ 
        \bs\xi^*\mbR + \mbL^*\bs\zeta = \mb0_r \right\}.
\end{equation}
\end{proof}

\paragraph{Retraction on $\mcB$.} As $\mcB$ is submanifold of the Euclidean space $(\mdC^{p \times r})^2$, a retraction updates $(\mbL,\mbR) \in \mcB$ in the direction of a tangent vector $(\bs\xi,\bs\zeta) \in \mcT_{(\mbL,\mbR)}\mcB$ by projecting the vector $(\mbL+\bs\xi,\mbR+\bs\zeta)$ back on the manifold $\mcB$.
\begin{prp}[Retraction on $\mcB$]
    \label{prp: retraction}
    The application $\tR \triangleq \tR^\tL \circ \tR^\tR : \mcT\mcB \mapsto \mcB$, such that for any $(\mbL,\mbR) \in \mcB$ and $(\bs\xi,\bs\zeta) \in \mcT_{(\mbL,\mbR)}\mcB$:
    \begin{equation*}
        \tR^\tR_{(\mbL,\mbR)}(\bs\xi, \bs\zeta) \triangleq \left(\bs\xi, \mbR + \bs\zeta\right)~,
    \end{equation*}
    \begin{equation*}
        \tR^\tL_{(\mbL,\mbR)}(\bs\xi,\tilde{\mbR}) \triangleq \left(\mbL + \bs\xi - \tilde{\mbR}(\tilde{\mbR}^*\tilde{\mbR})^{-1}(\tilde{\mbR}^*(\mbL + \bs\xi) -\mbI_r),\tilde{\mbR}\right)~,
    \end{equation*}
    forms a retraction on $\mcB$.    
\end{prp}

\begin{proof}
    Let $(\mbL,\mbR) \in \mcB$, since $\mbL^*\mbR = \mbI_r$, $\mbL$ and $\mbR$ are full-rank, hence by the Weyl inequality there exists a open neighborhood of $(\bs0,\bs0)$, denoted $\mcU$, in which $\mbL + \bs\xi$ and $\mbR + \bs\zeta$ remain full rank. Denoting $(\tilde{\mbL},\tilde{\mbR})= \tR(\bs\xi,\bs\zeta)$, one can verify that $\tilde{\mbL}^*\tilde{\mbR} = \mbI_r$. Hence $\tR$ is well defined on $\mcU$ and takes value in $\mcB$. In addition $\tR(\bs0,\bs0) = (\mbL,\mbR)$. Since any $(\tilde{\mbL},\tilde{\mbR}) \in \mcU$ are full rank, $\tilde{\mbR}^*\tilde{\mbR}$ is invertible, hence by addition, matrix product and inversion, the application $\tR$ is smooth and verifies: 
    \begin{equation}
        \tD\tR_{(\mbL,\mbR)}(\mb0,\mb0)[\mbd\bs\xi,\mbd\bs\zeta] \triangleq (\mbd\bs\xi - \mbR(\mbR^*\mbR)^{-1}(\mbR^*\mbd\bs\xi + \mbd\bs\zeta^*\mbL),\mbd\bs\zeta)~.
    \end{equation}
    However, since for any $(\mbd\bs\xi,\mbd\bs\zeta) \in \mcT_{(\mbL,\mbR)}\mcM$, $\mbR^*\mbd\bs\xi + \mbd\bs\zeta^*\mbL = \mb0$, it follows $\tD\tR_{(\mbL,\mbR)}(\mb0,\mb0) = \tI\td$ on $\mcU$. Hence by definition 4.1.1 in \cite{absil2008optimization}, $\tR$ is a retraction on $\mcB$.
\end{proof}

\begin{rmk}
    Let $(\mbL,\mbR) \in \mcB$ and $(\bs\xi,\bs\zeta) \in \mcT_{(\mbL,\mbR)}\mcB$. To remain in the open subset where the retraction is well-defined, one can estimate $\lambda \in \mdR^*_+$ such that $\mbR + \lambda \bs\zeta$ is full rank, i.e. $\lambda < \sigma_{\min}(\mbR)/\|\bs\zeta\|$, the norm being the operator norm or the HS norm.
\end{rmk}

Since bi-orthogonal low-rank matrices from a smooth manifold, we can investigates  properties of the manifold of phase/scale invariant bi-orthogonal matrices:

\begin{lmm}[Smooth manifold of phase/scale invariant bi-orthogonal matrices.]
\label{lmm:quotient_b}
 The manifold $\mcB /(\mdC^*)^r $ with $(\mbL,\mbR) \sim (\mbL',\mbR')$ if and only if there exists $\mbz \in (\mdC^*)^r$ such that $(\mbL',\mbR') = (\mbL\Diag(\bar{\mbz}),\mbR\Diag(\mbz^{-1}))$, is a smooth quotient manifold and the quotient map $\pi: \mcB \mapsto \mcB/(\mdC^*)^r$ is a submersion.
\end{lmm}

\begin{proof} 
    As $((\mdC^*)^r,\times)$ is an Abelian Lie group for the element-wise product, the action $(\mbz,(\mbL,\mbR)) \triangleq (\mbL\Diag(\bar{\mbz}),\mbR\Diag(\mbz^{-1}))$, for any $(z,(\mbL,\mbR)) \in (\mdC^*)^r \times \mcB$ is well defined and smooth. For any $(\mbL,\mbR) \in \mcB$, the stabilizer set $G_{(\mbL,\mbR)}=\{ z \in (\mdC^*)^r \ | \  (\mbL,\mbR) = \mbz \cdot (\mbL,\mbR) \} = \{\mbI_r\}$ is reduced to the neutral, hence the action is free.
    We prove the action's properness with the sequential characterization. Let $(\mbL_n,\mbR_n) \to (\mbL,\mbR)$ and $(\mbL_n \Diag(\bar{\mbz}_n),\mbR_n\Diag(\mbz_n^{-1})) \to (\mbL',\mbR')$, let's prove that $\mbz_n \to \mbz$. By convergence, since for $n \in \mdN$, $\mbL$ (resp. $\mbR$) is full rank, there exist constants $\tc_\mbL,\tC_\mbL>0$ (resp. $\tc_\mbR,\tC_\mbR$) such that for any $n \in \mdN$, $\tc_\mbL \leq \|\mbL_n\| \leq \tC_\mbL$ (resp. $\tc_\mbR \leq \|\mbR_n\| \leq \tC_\mbR$). Similarly, there exist $\tm_\mbL,\tM_\mbL>0$ (resp. $\tm_\mbR,\tM_\mbR$) such that for any $n \in \mdN$, $\tm_\mbL \leq \|\mbL_n \Diag(\bar{\mbz}_n)\| \leq \tM_\mbL$ (resp. $\tm_\mbR \leq \|\mbR_n\Diag(\mbz_n^{-1})\| \leq \tM_\mbR$). Hence, for any $n \in \mbN$: 
    \begin{equation*}
        \frac{\tm_\mbL}{\tC_\mbL} \leq \frac{\|\mbL_n \Diag(\bar{\mbz}_n)\|}{\|\mbL_n\|} \leq \|\mbz_n\| \quad \text{and} \quad 
        \frac{\tm_\mbR}{\tC_\mbR} \leq \frac{\|\mbR_n\Diag(\mbz_n^{-1})\|}{\|\mbR_n\|} \leq \|\mbz_n^{-1}\|~,
    \end{equation*}
    leading to the bounds:  $\tm_\mbL /\tC_\mbL \leq \|\mbz_n\| \leq r\tC_\mbR / \tm_\mbR $.
    As a sequence of in a compact set of $(\mdC^*)^r$, $(\mbz_n)_{n \in \mdN}$ admits a convergent subsequence. Hence, by the Bolzano-Weierstrass, the action is proper by compactness of the pre-image of a compact. Finally, as the action is smooth, free and proper, the quotient manifold is smooth and the quotient map $\pi: \mcB \mapsto \mcB/(\mdC^*)$ is a submersion.
\end{proof}

\begin{rmk}
    The manifold  $\mcB /(\mdC^*)^r$ also corresponds to the smooth manifold of set of $r$ orthogonal one-dimensional subspace of $\mdC^p$.
\end{rmk}

\subsubsection{Riemannian metrics on the $\mcB$ and $\mcB /(\mdC^*)^r$}

\paragraph{A natural Riemannian metric.} 
Since $(\mdC^{p \times r})^2$ endowed with the inner product $\innerp{(\mbA,\mbU)}{(\mbB,\mbV)} \triangleq \Tr(\mbA^*\mbB) + \Tr(\mbU^*\mbV)$ forms an Hilbert space, $\mcB$, as a submanifold, inherits a Riemannian metric $g$ such that for any $(\bs\xi,\bs\zeta),(\bs\xi',\bs\zeta') \in \mcT_{(\mbL,\mbR)}\mcB$, 
\begin{equation}
    g_{(\mbL,\mbR)}((\bs\xi,\bs\zeta),(\bs\xi',\bs\zeta')) \triangleq \innerp{(\bs\xi,\bs\zeta)}{(\bs\xi',\bs\zeta')}~.
\end{equation}
In particular, for $\bar{f} : (\mdC^{p \times r})^2 \mapsto \mdC$, and $f$ its restriction to $\mcB$ we have the relationship: 
\begin{equation}
    \grad_{(\mbL,\mbR)}f = \tP_{(\mbL,\mbR)}(\grad_{(\mbL,\mbR)}\bar{f})~,
\end{equation}
where $\tP_{(\mbL,\mbR)} : (\mdC^{p \times r})^2 \mapsto \mcT_{(\mbL,\mbR)}\mcB$ is the orthogonal projector on the subspace $\mcT_{(\mbL,\mbR)}\mcB$ at $(\mbL,\mbR) \in \mcB$.

\begin{prp}[Orthogonal projection on tangent spaces]
    For any $(\mbL,\mbR) \in \mcB$ endowed with its natural Riemannian metric, the orthogonal projector on the tangent space $\mcT_{(\mbL,\mbR)}\mcB$ verifies:

    \begin{equation}
        \tP_{(\mbL,\mbR)} : (\bs\xi',\bs\zeta') \in (\mdC^{p \times r})^2 \mapsto
        (\bs\xi' - \mbR\bs\mu^*, \bs\zeta' - \mbL\bs\mu ) \in \mcT_{(\mbL,\mbR)}\mcB ~,
    \end{equation}
    with: 
    \begin{equation}
        \bs\mu = \mbP_\mbL\mbZ\mbP_\mbR^*\quad \text{s.t.} \quad \mbZ_{ij} = \frac{[\mbP_\mbL^*(\bs\xi'^*\mbR + \mbL^*\bs\zeta')\mbP_\mbR]_{ij}}{\bs\lambda_{\mbR,j} + \bs\lambda_{\mbL,i}} \quad \text{and} \quad \left\{ \begin{array}{l}
            \mbR^*\mbR = \mbP_\mbR \Diag(\bs\lambda_\mbR)\mbP_\mbR^* \\
            \mbL^*\mbL = \mbP_\mbL \Diag(\bs\lambda_\mbL)\mbP_\mbL^* 
        \end{array}\right. ~.
    \end{equation}
\end{prp}

\begin{proof} Given $(\mbL,\mbR) \in \mcB$, since $\mcT_{(\mbL,\mbR)}\mcB$ is a vector subspace of the Hilbert space $(\mdC^{p \times r})^2$, by the projection theorem on a closed convex set, the orthogonal projector exists, it is a linear application and it verifies:
     \begin{equation}
        \tP_{(\mbL,\mbR)} : (\bs\xi',\bs\zeta') \in (\mdC^{p \times r})^2 \mapsto
            \argmin_{(\bs\xi,\bs\zeta) \in \mcT_{(\mbL,\mbR)}\mcB} \left\|(\bs\xi,\bs\zeta) - (\bs\xi',\bs\zeta')\right\|^2~.
    \end{equation}
    Given $(\bs\xi',\bs\zeta') \in (\mdC^{p \times r})^2$ and considering the real differentiability case, the minimization problem is strictly convex and strictly feasible, hemce the KKT conditions are necessary and sufficient conditions to characterize the optimum. The Lagrandian can be expressed as: 
    \begin{equation}
        \mcL(\bs\xi,\bs\zeta,\bs\lambda,\tilde{\bs\lambda}) = \left\|(\bs\xi,\bs\zeta) - (\bs\xi',\bs\zeta')\right\|^2 
        + \Tr(\bs\lambda^\intercal\Rel(\bs\xi^*\mbR + \mbL^*\bs\zeta))
        + \Tr(\tilde{\bs\lambda}^\intercal\Img(\bs\xi^*\mbR + \mbL^*\bs\zeta))
    \end{equation}
    As a quadratic function, the lagrangian is differentiable and, taking the Wirtinger notation, at the optimum it holds: 
    \begin{equation}
        \left\{
            \begin{array}{l}
               \nabla_{\overline{\bs\xi}} \mcL = \bs\xi - \bs\xi' + \mbR(\bs\lambda -i\tilde{\bs\lambda})^\intercal =\mb0 \\
               \nabla_{\overline{\bs\zeta}} \mcL = \bs\zeta - \bs\zeta' + \mbL(\bs\lambda + i\tilde{\bs\lambda}) = \mb0 \\ 
               \nabla_{\bs\lambda}\mcL = \Rel(\bs\xi^*\mbR + \mbL^*\bs\zeta) = \mb0 \\ 
               \nabla_{\tilde{\bs\lambda}}\mcL = \Img(\bs\xi^*\mbR + \mbL^*\bs\zeta) = \mb0 \\ 
            \end{array}
        \right.
    \end{equation}
    Denoting $\bs\mu = \bs\lambda + i\tilde{\bs\lambda}$, the optimal solution exists, is unique and verifies:
    \begin{equation}
        \left\{ \begin{array}{l}
            \bs\xi = \bs\xi' - \mbR\bs\mu^*\\
            \bs\zeta = \bs\zeta' - \mbL\bs\mu 
        \end{array}\right. 
        \quad \text{s.t.} \quad
        \bs\mu \mbR^*\mbR + \mbL^*\mbL\bs\mu  = \bs\xi'^*\mbR + \mbL^*\bs\zeta'
    \end{equation}
    Note that $\bs\mu$ is solution of a Sylvester system and since $\mbL$ and $\mbR$ are full rank, the matrices $\mbR^*\mbR$ and $\mbL^*\mbL$ are invertible and self adjoint, hence they admit a spectral decomposition with all eigenvalues real and strictly positive. Hence, $\bs\mu$ verifies: 
    \begin{equation}
        \bs\mu = \mbP_\mbL\mbZ\mbP_\mbR^*\quad \text{s.t.} \quad \mbZ_{ij} = \frac{[\mbP_\mbL^*(\bs\xi'^*\mbR + \mbL^*\bs\zeta')\mbP_\mbR]_{ij}}{\bs\lambda_{\mbR,j} + \bs\lambda_{\mbL,i}} \quad \text{and} \quad \left\{ \begin{array}{l}
            \mbR^*\mbR = \mbP_\mbR \Diag(\bs\lambda_\mbR)\mbP_\mbR^* \\
            \mbL^*\mbL = \mbP_\mbL \Diag(\bs\lambda_\mbL)\mbP_\mbL^* 
        \end{array}\right. ~.
    \end{equation}
\end{proof}

\paragraph{A stable Riemannian metric.} We propose a Riemannian metric that is numerically more stable by rescaling search directions and avoiding directions towards singularities. This metric has the benefit to also be well-defined on the quotient manifold $\mcB/(\mdC^*)^r$ as depicted below.

\begin{prp}[Stable Riemannian metric]
    \label{prp:stable_metric}
    For any $(\bs\xi,\bs\zeta),(\bs\xi',\bs\zeta') \in \mcT_{(\mbL,\mbR)}\mcB$, the application: 
    \begin{equation}
        g_{(\mbL,\mbR)}((\bs\xi,\bs\zeta),(\bs\xi',\bs\zeta')) \triangleq 
        \Tr(\bs\xi^*\bs\xi'(\mbL^*\mbL)^{-1}) +  \Tr(\bs\zeta^*\bs\zeta'(\mbR^*\mbR)^{-1})~,
    \end{equation} 
    defines a Riemannian metric on $\mcB$.
\end{prp}

\begin{proof}
    Let $(\mbL,\mbR) \in \mcB$, since $\mbL^*\mbR = \mbI_r$, $\mbL$ and $\mbR$ are full rank, leading to $\mbL^*\mbL$ and $\mbR^*\mbR$ being hermitian and positive definite. Thus, by sum of trace norms, $g_{(\mbL,\mbR)}$ is hermitian, positive definite (separated) and verifies the triangle inequality. 
    Finally, for any smooth vector fields on $\mcB$, $\mbX,\mbY \in \mathfrak{F}(\mcB)$, the application:
    \begin{equation*}
        (\mbL,\mbR) \in \mcB \mapsto g_{(\mbL,\mbR)}((\mbX^{\bs\xi}_{(\mbL,\mbR)},\mbX^{\bs\zeta}_{(\mbL,\mbR)}),(\mbY^{\bs\xi}_{(\mbL,\mbR)},\mbY^{\bs\zeta}_{(\mbL,\mbR)}))~,
    \end{equation*}
    is smooth as trace of matrix products of smooth matrices. Hence $g$ is a Riemannian metric.
\end{proof}

\begin{prp}[Orthogonal projection on tangent spaces]
    Let $(\mbL,\mbR) \in \mcB$ and consider the Hilbert space $((\mdC^{p \times r})^2,g_{(\mbL,\mbR)})$, the orthogonal projector on the tangent space $\mcT_{(\mbL,\mbR)}\mcB$ verifies:
    \begin{equation}
        \tP_{(\mbL,\mbR)} : (\bs\xi',\bs\zeta') \in (\mdC^{p \times r})^2 \mapsto
        \left(
            \begin{array}{l}
                \bs\xi' - \mbR(\mbR^*\mbR)^{-1}(\bs\xi'^*\mbR + \mbL^*\bs\zeta')^* \\
                \bs\zeta' - \mbL(\mbL^*\mbL)^{-1}(\bs\xi'^*\mbR + \mbL^*\bs\zeta')
            \end{array}
        \right) \in \mcT_{(\mbL,\mbR)}\mcB ~,
    \end{equation}    
\end{prp}

\begin{proof} Given $(\mbL,\mbR) \in \mcB$, since $(\mcT_{(\mbL,\mbR)}\mcB, g_{(\mbL,\mbR)})$ is a vector subspace of the Hilbert space $((\mdC^{p \times r})^2,g_{(\mbL,\mbR)})$, by the projection theorem on a closed convex set, the orthogonal projector exists, it is a linear application and it verifies:
     \begin{equation}
        \tP_{(\mbL,\mbR)} : (\bs\xi',\bs\zeta') \in (\mdC^{p \times r})^2 \mapsto
            \argmin_{(\bs\xi,\bs\zeta) \in \mcT_{(\mbL,\mbR)}\mcB} g^2_{(\mbL,\mbR)}((\bs\xi-\bs\xi',\bs\zeta - \bs\zeta'),(\bs\xi-\bs\xi',\bs\zeta - \bs\zeta')).
    \end{equation}
    Given $(\bs\xi',\bs\zeta') \in (\mdC^{p \times r})^2$ and considering the real differentiability case, the minimization problem is strictly convex and strictly feasible, hence the KKT conditions are necessary and sufficient conditions to characterize the optimum. The Lagrangian can be expressed as: 
    \begin{equation}
        \mcL(\bs\xi,\bs\zeta,\bs\lambda,\tilde{\bs\lambda}) = 
        g^2_{(\mbL,\mbR)}((\bs\xi-\bs\xi',\bs\zeta - \bs\zeta'),(\bs\xi-\bs\xi',\bs\zeta - \bs\zeta'))
        + \Tr(\bs\lambda^\intercal\Rel(\bs\xi^*\mbR + \mbL^*\bs\zeta))
        + \Tr(\tilde{\bs\lambda}^\intercal\Img(\bs\xi^*\mbR + \mbL^*\bs\zeta))
    \end{equation}
    As a quadratic function, the Lagrangian is differentiable and, taking the Wirtinger notation, at the optimum it holds: 
    \begin{equation}
        \left\{
            \begin{array}{l}
               \nabla_{\overline{\bs\xi}} \mcL = (\bs\xi - \bs\xi')(\mbL^*\mbL)^{-1} + \mbR(\bs\lambda -i\tilde{\bs\lambda})^\intercal =\mb0 \\
               \nabla_{\overline{\bs\zeta}} \mcL = (\bs\zeta - \bs\zeta')(\mbR^*\mbR)^{-1} + \mbL(\bs\lambda + i\tilde{\bs\lambda}) = \mb0 \\ 
               \nabla_{\bs\lambda}\mcL = \Rel(\bs\xi^*\mbR + \mbL^*\bs\zeta) = \mb0 \\ 
               \nabla_{\tilde{\bs\lambda}}\mcL = \Img(\bs\xi^*\mbR + \mbL^*\bs\zeta) = \mb0 \\ 
            \end{array}
        \right.
    \end{equation}
    Denoting $\bs\mu = \bs\lambda + i\tilde{\bs\lambda}$, the optimal solution exists, is unique and verifies:
    \begin{equation}
        \left\{ \begin{array}{l}
            \bs\xi = \bs\xi' - \mbR\bs\mu^*(\mbL^*\mbL)\\
            \bs\zeta = \bs\zeta' - \mbL\bs\mu (\mbR^*\mbR)
        \end{array}\right. 
        \quad \text{s.t.} \quad
        \bs\xi^*\mbR + \mbL^*\bs\zeta = 0
    \end{equation}
\end{proof}

\begin{prp}[Riemannian gradient]
    \label{prp:riem_grad}
    Consider $\bar{f}: (\mdC^{n \times r})^2 \mapsto \mdR$ a smooth function and let $f$ denotes its restriction to $\mcB$, for any $(\mbL,\mbR) \in \mcB$: 
    \begin{equation}
        \label{eq: stable orthogonal projection}
        \grad f (\mbL,\mbR) \triangleq \tP_{(\mbL,\mbR)}(\nabla_\mbL \bar{f}(\mbL^*\mbL),\nabla_\mbR \bar{f}(\mbR^*\mbR))~,
    \end{equation}    
    where $\nabla \bar{f} = (\nabla_\mbL \bar{f}, \nabla_\mbR \bar{f})$ is the gradient in the natural Euclidean space and $\tP_{(\mbL,\mbR)}$ the orthogonal projector on the tangent space at $(\mbL,\mbR) \in \mcB$.
\end{prp}

\begin{proof}
    Consider $(\mbL,\mbR) \in \mcB$ and let $\nabla \bar{f} = (\nabla_\mbL \bar{f}, \nabla_\mbR \bar{f})$ be the gradient of $\bar{f}$ in the natural Euclidean space $((\mdC^{n \times r})^2, \innerp{.}{.})$. Then its gradients in $((\mdC^{n \times r})^2, g_{(\mbL,\mbR)})$, denoted as $\nabla\bar{f}_g$, verifies: 
    \begin{equation*}
        \innerp{\nabla \bar{f}}{\bs\xi} = g_{(\mbL,\mbR)}(\nabla\bar{f}_g,\bs\xi)~, \quad \forall \bs\xi \in \mcT_{(\mbL,\mbR)}\mcB~.
    \end{equation*}
    Hence, $(\nabla_\mbL \bar{f}_g, \nabla_\mbR \bar{f}_g) = (\nabla_\mbL \bar{f}(\mbL^*\mbL), \nabla_\mbR \bar{f}(\mbR^*\mbR))$, and, by projection on $ \mcT_{(\mbL,\mbR)}\mcB$, the Riemannian gradient of $f$ at $(\mbL,\mbR)$, denoted $\grad f {(\mbL,\mbR)}$, verifies \cref{eq: stable orthogonal projection}.
\end{proof}

With the sable metric in mind we can now define a Riemannian structure on the manifold of phase/scale invariant low-rank matrices.

\begin{prp}
    \label{prp:quotient_riem}
    The manifold $(\mcB/(\mdC^*)^r,g)$ is a Riemannian quotient manifold of $(\mcB,g)$ where $g$ is the stable Riemannian metric. 
\end{prp}

\begin{proof}
    Let $(\mbL,\mbR) \in \mcB$, $(\bs\xi,\bs\zeta)$ and $ (\bs\xi',\bs\zeta')\in \mcT_{(\mbL,\mbR)}\mcB$, by the group action of $(\mdC^*)^r$, for any $\mbz \in (\mdC^*)^r$ it follows:
    \begin{equation*}
        \begin{array}{rcl}
            g_{\mbz \cdot (\mbL,\mbR)}(\td \mbz(\bs\xi,\bs\zeta),\td \mbz(\bs\xi',\bs\zeta')) & = & g_{(\bar{\mbz}\mbL,\mbz^{-1}\mbR)}((\bar{\mbz}\bs\xi,\mbz^{-1}\bs\zeta),(\bar{\mbz}\bs\xi',\mbz^{-1}\bs\zeta')) \\
            & = & \Tr(\Diag(|\mbz|^2) \bs\xi^*\bs\xi'(\mbL^*\mbL)^{-1}\Diag(|\mbz|^{-2}) ) \\
            && + \Tr(\Diag(|\mbz|^{-2}) \bs\zeta^*\bs\zeta'(\mbR^*\mbR)^{-1}\Diag(|\mbz|^{2}) ) \\
            & = &  g_{(\mbL,\mbR)}((\bs\xi,\bs\zeta),(\bs\xi',\bs\zeta'))~,
        \end{array}
    \end{equation*}
    where $|\mbz|$ is the element with norm vector. Thus, $g$ is invariant to the group action. Thus, $(\mcB/(\mdC^*)^r,g)$ is a Riemannian quotient manifold of $\mcB$.
\end{proof}

\begin{rmk}
    For any smooth function $\bar{f}: \mcB \to \mdR$ also defining a function on the quotient manifold $f: \mcB/(\mdC^*)^r \to \mdR$, the unique gradient represent at $[x]$ in horizontal space at $\mcT_x\mcB$, denoted with abuse of notation $\grad f([x])$, verifies:
    $\grad f([x]) = \grad \bar{f}(x)~.$
\end{rmk}

\section{Dynamical Operator Dictionary Learning}
\label{app:doodl}

\subsection{Wasserstein reconstruction model}
\label{app: wrm}
Since the data fitting loss function corresponds to a divergence, a 
natural choice of reconstruction map is the Fréchet mean barycenter with 
\begin{equation}
    \label{eq:ot_frechet_barycenter}
    \tB^{\text{\tiny OT}}(\bs\alpha; \overline{\bs\mcG}) \triangleq \argmin_{\bs\mcG \in \mcN} \textstyle\sum_{j \in [d]} \alpha_j d_\mcS(\bs\mcG, \overline{\bs\mcG}_j)~.
\end{equation}
This reconstruction map is consistent with the geometry of spectral operators
induced by the divergence $d_\mcS$ \cite{germain2026spectral}. A similar
strategy was proposed in \cite{schmitz2018wasserstein} for performing non-linear 
DL is the space of probability distributions endowed with the Wasserstein metric. 
It corresponds to a free-support optimal
transport barycenter \cite{cuturi2014fast}, but classical fixed point iteration
cannot be applied because of the Riemannian structure of the manifold. Instead,
the barycenter can be computed via Riemannian gradient-based optimization as detailed in Appendix \ref{app:spectral_decomp_manifold}. While this approach faithfully
captures the underlying geometry, it incurs a significant computational cost due
to the inner optimization problem, which can become prohibitive in practice.

\subsection{Optimization scheme}
\label{app: optimization scheme}

We adopt a stochastic inexact block-coordinate
strategy similar to the one proposed in \cite{mairal2009online} for euclidean spaces. Given a batch of size $b$, we perform two successive steps:
\begin{enumerate}[leftmargin=*,itemsep=0pt,topsep=1pt]
    \item \emph{Coefficient estimation:} fix the dictionary and optimize
    $\{\bs\alpha_i\}_{i \in [b]}$ via gradient descent (with softmax
    parametrization). Since the problem is not convex, we use an initialization
    of  based on proximity to dictionary atoms, i.e.
    $\textstyle\bs\alpha_i \propto \{- d_\mcS(\bs\mcG_i,
    \overline{\bs\mcG}_j)\}_{j \in [d]}$ to start the optimization in a relevant region of the parameter space. In practice we use the Adam optimizer with a learning rate of 0.5 and perform 100 iterations. \Cref{alg:coef_estimation} describes the pseudo code.
    \item \emph{Dictionary update:} fix the coordinates $\{\bs\alpha_i\}_{i
    \in [b]}$ and update the dictionary with a Riemannian gradient step on
    $\mcN^d$ from the objective on the batch. We leverage the envelope theorem to ignore implicit gradients
    through $\{\bs\alpha_i\}_{i \in [b]}$. Default gradient steps has a learning rate of 1e-2.
\end{enumerate}
\Cref{alg: main algorithm} describes the overall optimization procedure. Computation of the data fitting loss that is the SGOT divergence \citep{germain2026spectral} requires the computation of optimal transport plans that we implemented in order to leverage the computational efficiency of GPUs.

\begin{algorithm}[h]
\begin{algorithmic}[1]
\Require $\bs\mcG \triangleq \{\bs\mcG_i\}_{i \in [N]} \in \mcM^N$,  $\overline{\bs\mcG} \triangleq \{\overline{\bs\mcG}_i\}_{i \in [d]} \in \mcM^d$
\Comment{Operators / Dictionary}
\State $\bs\alpha \gets \text{InitializationCoefficientLogit}(\bs\mcG, \overline{\bs\mcG})$
\Comment{$\textstyle\bs\alpha_i \triangleq \{- d_\mcS(\bs\mcG_i,
    \overline{\bs\mcG}_j)\}_{j \in [d]}$}
\While{not converged}
\State $\widehat{\bs\mcG} \gets \text{ComputeReconstructedSpectralDecomposition}(\bs\alpha, \overline{\bs\mcG})$
\Comment{\Cref{eq:reconstruction_model}}
\State $\bs\mcL \gets \text{ComputeDataFittingLoss}(\widehat{\bs\mcG},\bs\mcG)$
\Comment{\Cref{eq:sgot_divergence}}
\State $\mbR,\bs\Lambda,\mbL \gets \text{UpdateCoefficientLogit}(\bs\mcL, \bs\alpha)$
\EndWhile
\\
\Return $\bs\alpha$
\end{algorithmic}
\caption{Operator coefficient estimation}
\label{alg:coef_estimation}
\end{algorithm}

\begin{algorithm}[h]
\begin{algorithmic}[1]
\Require $\bs\mcG \triangleq \{\bs\mcG_i\}_{i \in [N]} \in \mcM^N$,  $\overline{\bs\mcG} \triangleq \{\overline{\bs\mcG}_i\}_{i \in [d]} \in \mcM^d$
\Comment{Operator / Dictionary}
\While{not converged}
\State $\bs\alpha \gets \text{EstimateOperatorCoefficient}(\bs\mcG, \overline{\bs\mcG})$
\Comment{\Cref{alg:coef_estimation}}
\State $\widehat{\bs\mcG} \gets \text{ComputeReconstructedSpectralDecomposition}(\bs\alpha, \overline{\bs\mcG})$
\Comment{\Cref{eq:reconstruction_model}}
\State $\bs\mcL \gets \text{ComputeDataFittingLoss}(\widehat{\bs\mcG},\bs\mcG)$
\Comment{\Cref{eq:sgot_divergence}}
\State $\bs\xi \gets \text{ComputeAmbientSpaceDictionaryGradient}(\bs\mcL)$
\Comment{$\widehat{\bs\mcG}$ is fixed}
\State $\bs\xi \gets \text{ComputeRiemannianGradient}(\bs\xi,\overline{\bs\mcG})$
\Comment{\Cref{prp:riem_grad}}
\State $\overline{\bs\mcG} \gets \text{UpdateDictionary}(\bs\xi,\overline{\bs\mcG})$
\Comment{Perform a retraction, see \Cref{prp: retraction}}
\EndWhile
\\
\Return $\overline{\bs\mcG}$
\end{algorithmic}
\caption{Learn DOODL dictionary}
\label{alg: main algorithm}
\end{algorithm}

\section{Dictionary Learning on Langevin dynamics}
\label{app: langevin exp}

\subsection{Langevin dynamics}
Langevin dynamics provides a fundamental mathematical framework for modeling stochastic dynamical systems evolving under the combined influence of deterministic forces and thermal or environmental noise. It underpins a wide range of scientific domains, including molecular dynamics, statistical physics, chemistry, materials discovery, and biological conformational analysis, where the key quantities of interest are inherently dynamical and spectral: metastable states, transition pathways, relaxation time scales, and invariant distributions \cite{schutte2001transfer,schutte2023overcoming}. In these settings, obtaining reliable estimates of the underlying evolution operators is notoriously difficult because trajectories are often short, high-dimensional, partially observed, and collected far from equilibrium \cite{bolhuis2002transition, wu2017variational}.\\

\subsection{Experimental settings}
The whole experiment is ran on a single GPU (NVIDIA RTX A6000).
\paragraph{Trajectory simulation}
We consider a one-dimensional damped Langevin dynamics governed by $\td X_t = \nabla U_w(X_t) + \sqrt{2 \sigma} \td B_t$ with $\{B_t\}_{t>0}$ a Brownian motion, $\sigma >0$ its temperature, and a two well potential $U_w(x) \triangleq w^{-4}(x^2-w^2)^2$ where $w >0$ controls the distance between the two wells and their respective width, see \cref{fig:langevin_activation}-left. We uniformly sampled 256/256 train/test potential parameters $w \in [0.5,1.2]$ and generates trajectories of 40k samples at 100Hz according to their Langevin dynamics and with the Euler–Maruyama method. 

\paragraph{Operator estimation.}
The operators are estimated via Reduced Rank Regression (RRR) \cite{kostic2022learning} with a tikhonov regularization of 1e-6. Operators rank is fixed to 3 and the observable space approximates the RBF kernel via 400 Random Fourier Features \cite{avron2017random} on sliding windows of 50 samples.

\paragraph{Dictionary learning.} On the training set, we learn DOODL dictionaries \cref{eq:spectral_dictionary_learning} with 2 to 5 atoms, using the SGOT divergence with $\eta=0.25$ and the log-Martin metric for the spectral projector term. The training of each dictionary last for 3 epochs with batch sizes of 32. Dictionary is learned following the procedure described in Appendix \cref{app: optimization scheme} with the default settings.

\subsection{Additional results.}

We provide additional results for the estimation of operators from short trajectories in \Cref{fig: a1} and \Cref{fig: a2} by including DOODL dictionaries with 3 and 4 atoms. We observe the DOODL performs better than other estimators for trajectories up to 10k samples with lower variances for all potential widths. As well by increasing the number of atoms operator estimation with DOODL improves in accuracy.

\begin{figure}[ht]
    \centering
    \includegraphics[width=\linewidth]{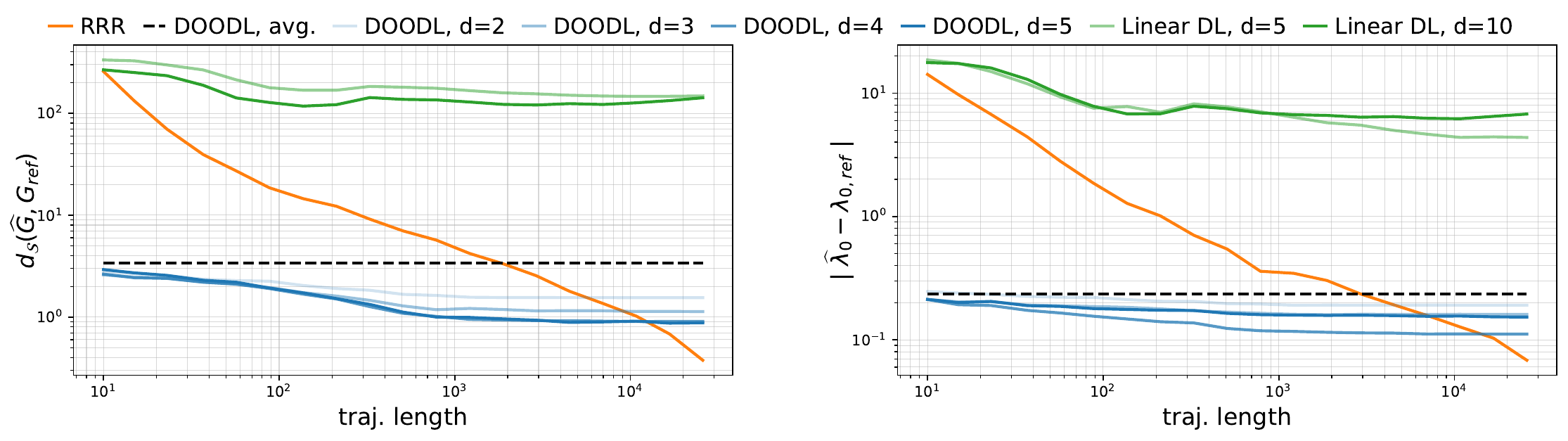}
    \caption{Comparison of operator estimator on truncated trajectories, including individual RRR, linear DL and DOODL. Error to ground truth in SGOT divergence between operators (left) and absolute error between first eigenvalues (right).}
    \label{fig: a1}
\end{figure}

\begin{figure}[ht]
    \centering
    \includegraphics[width=0.5\linewidth]{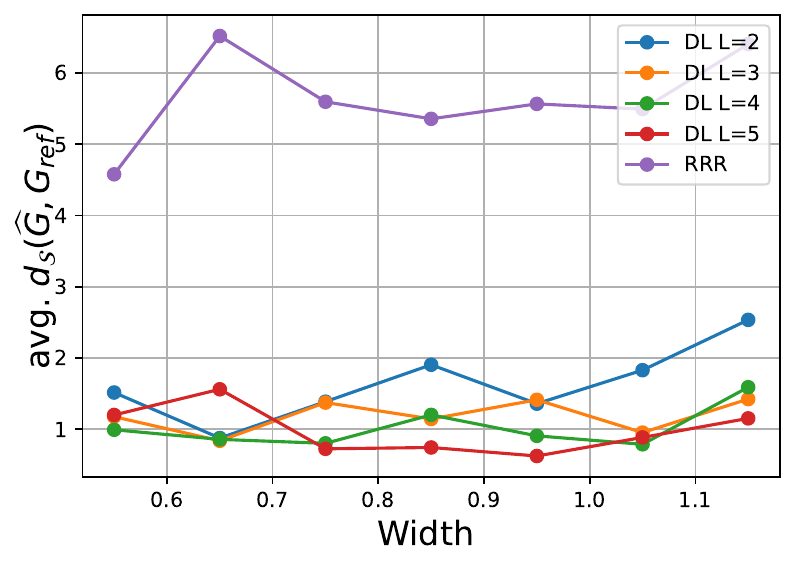}
    \caption{Comparison of operator estimation per potential width with SGOT error on trajectories of length 1.2k. Including individual RRR and DOODL.}
    \label{fig: a2}
\end{figure}

\section{Plasma Experiments}
\label{app:plasma}
This section details the plasma pipeline used in the Tokam2D experiments. Starting from
high-dimensional plasma trajectories, we construct finite-rank generator representations
and then test whether their operator geometry retains the physical parameters of the dynamicals systems.
The pipeline is
\[
    \mathbf x_t
    \longmapsto
    \psi_{\rm P}(\mathbf x_t)
    \longmapsto
    \mathbf z_t
    \longmapsto
    \widehat G
    \longmapsto
    \widehat{\boldsymbol\alpha}.
\]
Here $\mathbf x_t$ is a four-channel plasma snapshot, $\psi_{\rm P}(\mathbf x_t)$ is the
frozen POSEIDON descriptor, $\mathbf z_t$ is the learned temporal coordinate,
$\widehat G\in\mathcal S_r(\mathcal H)$ is the finite-rank generator representation, and
$\widehat{\boldsymbol\alpha}\in\Delta_K$ is the \DOODL{} barycentric coordinate.

The parameters $(g,\kappa)$ are never used during POSEIDON feature extraction, temporal
representation learning, generator-resolvent estimation, or dictionary learning (All the pipeline is Unsupervised). They are
used only afterward as diagnostics. Thus, recovering $(g,\kappa)$ from the learned
operator coordinates indicates that the unsupervised operator geometry captures physical
structure in the Tokam2D regime family.

Note that the whole experiment is ran on a single GPU (NVIDIA RTX A6000).

\subsection{Plasma details.}
\label{sec:plasma details}
The plasma benchmark is based on a reduced two-field model for edge turbulence with
density $n(x,y,t)$ and electrostatic potential $\phi(x,y,t)$
\citep{ghendrih2018generation,ghendrih2022role}. The vorticity is
$W=\Delta_\perp\phi$, with $\Delta_\perp=\partial_{xx}+\partial_{yy}$. The potential
also defines the advecting velocity through its spatial derivatives
$(-\partial_y\phi,\partial_x\phi)$, which are used as input channels in the observation
map. The nonlinear advection is written with the Poisson bracket
\[
[\phi,f]=\partial_x\phi\,\partial_y f-\partial_y\phi\,\partial_x f .
\]
We use the gradient-driven configuration, where a prescribed background density gradient
replaces a localized source. This gives the scalar control parameter $\kappa=1/L_n$.
The dynamics are
\[
\left\{
\begin{aligned}
\partial_t n+\kappa\partial_y\phi+[\phi,n]-D_n\Delta_\perp n
&=-\sigma_n n+\sigma_{n\phi}\phi,\\
\partial_t W+g\partial_y n+[\phi,W]-D_\phi\Delta_\perp W
&=-\sigma_{\phi n}n+\sigma_\phi\phi,
\qquad W=\Delta_\perp\phi .
\end{aligned}
\right.
\]
The parameter $\kappa$ controls the density-gradient drive in the density equation, while
$g$ controls the interchange coupling in the vorticity equation through $g\partial_y n$.
Thus, $\kappa$ acts directly on the transported density field, whereas $g$ acts through the
density--vorticity coupling and modifies the potential-driven flow. In the linear
interchange analysis, the drive depends multiplicatively on $g$ and the background
gradient, which explains why the strongest changes are observed when both parameters are
large~\citep{ghendrih2018generation}. All diffusion and loss coefficients are fixed.

We generate the data with Tokam2D~\citep{tokam2d}, a reduced 2D fluid solver in the plane
transverse to the magnetic field. We uniformly sample $M=400$ dynamics with
$(g,\kappa)\in[0,0.5]\times[1,3.5]$. Each system is simulated on a $128\times128$ grid
for $T=4093$ time steps with $\Delta t=0.05$, and we use an $80/20$ train/test split over
simulations. Each regime $m$ is represented by
$X^{(m)}=(x_1^{(m)},\ldots,x_T^{(m)})$, with
\[
x_t^{(m)}
=
\bigl(
n_t^{(m)},
-\partial_y\phi_t^{(m)},
\partial_x\phi_t^{(m)},
\phi_t^{(m)}
\bigr)
\in\mathbb{R}^{4\times128\times128}.
\]

\begin{figure}[H]
    \centering
    \includegraphics[width=0.88\linewidth]{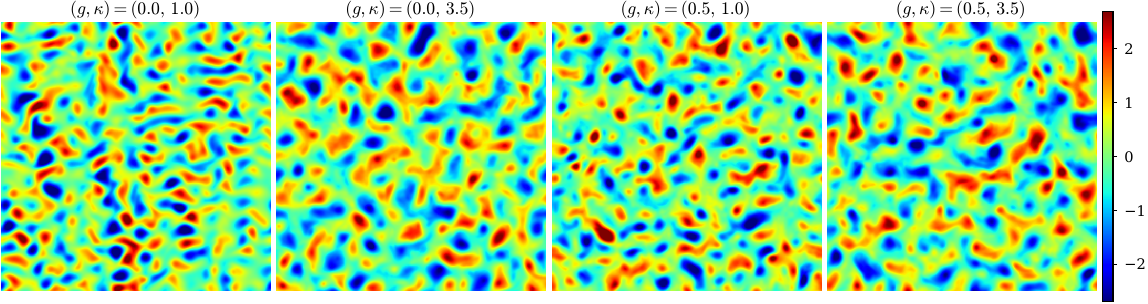}
    \caption{
    electrostatic potential $\phi(x,y,t)$ fluctuations at the four corners of the $(g,\kappa)$ grid.
    In contrast to density snapshots, which mainly show the transported scalar field,
    the potential highlights the organization of the advecting flow. Increasing $g$
    makes the potential structures more deformed and vortex-like, consistently with its
    role in the vorticity equation.
    }
    \label{fig:app_plasma_potential}
\end{figure}
Figure~\ref{fig:app_plasma_potential} complements the density snapshots by showing the
potential-driven flow organization. The panels are standardized independently, so they
compare morphology rather than amplitude. Increasing $\kappa$ changes the spatial scale,
whereas increasing $g$ makes the structures more distorted and vortex-like, especially in
the high-$\kappa$ regime. This suggests that part of the information on $g$ is carried by
the potential/vorticity dynamics rather than by density fluctuations alone.
\subsection{POSEIDON as a PDE foundation model encoder}
\label{app:poseidon_observation}

The first stage maps each plasma state to a finite-dimensional PDE-aware descriptor. We use
POSEIDON-B~\citep{herde2024poseidon} as a frozen feature extractor. POSEIDON is a PDE
foundation model pretrained to approximate solution operators: given a state
$\mathbf u(t_i)$ and a lead time $\tau$, its native task is
\[
    \mathcal S_\theta(\tau,\mathbf u(t_i)) \simeq \mathbf u(t_i+\tau).
\]
Its scOT backbone uses patch embeddings, multiscale shifted-window attention, and
lead-time conditioning. The POSEIDON paper also reports strong transfer to
out-of-distribution PDEs and physical processes not seen during pretraining. This is the
setting we rely on here: POSEIDON is not used to forecast Tokam2D trajectories. The
recovery head is discarded, all POSEIDON parameters are frozen, and only the internal
encoder representation is retained.

The choice of channels is guided by the variables and differential operators appearing in
the equations. Define
\[
    \nabla^\perp\phi := (-\partial_y\phi,\partial_x\phi),
    \qquad
    \mathbf V_E := \nabla^\perp\phi,
    \qquad
    \nabla\cdot\mathbf V_E=0 .
\]
Then
\[
    [\phi,f]
    =
    \partial_x\phi\,\partial_y f-\partial_y\phi\,\partial_x f
    =
    \mathbf V_E\cdot\nabla f ,
\]
and the Tokam2D system can be written as
\[
\left\{
\begin{aligned}
\partial_t n+\mathbf V_E\cdot\nabla n+\kappa\partial_y\phi
-D_n\Delta_\perp n
&=-\sigma_n n+\sigma_{n\phi}\phi,\\
\partial_t W+\mathbf V_E\cdot\nabla W+g\partial_y n
-D_\phi\Delta_\perp W
&=-\sigma_{\phi n}n+\sigma_\phi\phi,
\qquad W=\Delta_\perp\phi .
\end{aligned}
\right.
\]
Thus Tokam2D has the same basic transport form as the incompressible equations used in
POSEIDON pretraining: a divergence-free vector field advects scalar or vorticity-like
quantities, with diffusion and coupling terms. In two-dimensional incompressible
Navier--Stokes, writing $\mathbf u=(u,v)$, one has
\[
    \partial_t\mathbf u+(\mathbf u\cdot\nabla)\mathbf u+\nabla p-\nu\Delta\mathbf u=0,
    \qquad
    \nabla\cdot\mathbf u=0,
\]
or equivalently in vorticity form
\[
    \partial_t\omega+\mathbf u\cdot\nabla\omega-\nu\Delta\omega=0,
    \qquad
    \mathbf u=\nabla^\perp\psi,
    \qquad
    \omega=\Delta\psi .
\]
The corresponding structures are
\[
    \mathbf V_E=\nabla^\perp\phi
    \quad\leftrightarrow\quad
    \mathbf u=\nabla^\perp\psi,
    \qquad
    W=\Delta_\perp\phi
    \quad\leftrightarrow\quad
    \omega=\Delta\psi .
\]
Moreover, POSEIDON is also pretrained on compressible Euler variables, which include a
density $\rho$, a velocity field $(u,v)$, and a pressure-like scalar $p$. We therefore feed
POSEIDON the four-channel Tokam2D state
\[
    \mathbf x_t^{(m)}
    =
    \bigl(
    n_t^{(m)},
    V_{E,x,t}^{(m)},
    V_{E,y,t}^{(m)},
    \phi_t^{(m)}
    \bigr)
    =
    \bigl(
    n_t^{(m)},
    -\partial_y\phi_t^{(m)},
    \partial_x\phi_t^{(m)},
    \phi_t^{(m)}
    \bigr)
    \in\mathbb R^{4\times128\times128},
\]
with the channel correspondence
\[
    n_t \leftrightarrow \rho,
    \qquad
    (V_{E,x,t},V_{E,y,t}) \leftrightarrow (u,v),
    \qquad
    \phi_t \leftrightarrow p .
\]
This gives POSEIDON fields with the same mathematical roles as in its pretraining
interface: a density-like scalar, a vector transport field, and a scalar potential or
pressure-like channel.

Before applying the frozen encoder, each channel is standardized independently for each
snapshot. The normalized four-channel field is passed through the frozen POSEIDON patch
embedding and scOT encoder. If
\[
    \mathbf H_t^{(m)}
    =
    \mathrm{Enc}_{\rm P}(\widetilde{\mathbf x}_t^{(m)})
    \in\mathbb R^{N_{\rm tok}\times 768}
\]
denotes the last hidden token matrix, we define the frozen POSEIDON descriptor by token
averaging,
\[
    \psi_{\rm P}(\mathbf x_t^{(m)})
    =
    \frac{1}{N_{\rm tok}}
    \sum_{\ell=1}^{N_{\rm tok}}
    \mathbf H_{t,\ell}^{(m)}
    \in\mathbb R^{768}.
\]
For each regime, this gives the descriptor trajectory
\[
    \mathbf\Psi_{\rm P}^{(m)}
    =
    \bigl(
    \psi_{\rm P}(\mathbf x_1^{(m)}),\ldots,
    \psi_{\rm P}(\mathbf x_T^{(m)})
    \bigr)
    \in\mathbb R^{T\times768}.
\]
These descriptors are precomputed once for every simulation. The operators used by
\DOODL{} are not estimated directly from $\mathbf\Psi_{\rm P}^{(m)}$, but from the learned
temporal coordinates described next.
\subsection{Observable space learning for single dynamical system}
\label{app:temporal_learning}

The frozen POSEIDON descriptors encode spatial morphology, but they are not optimized for
spectral operator estimation. We therefore learn temporal coordinates using the contrastive
loss from Neural Conditional Probability (NCP)~\citep{kostic2024neural}. We use
the same loss as NCP, but not for conditional probability estimation: here, it is used only
to extract a latent space  adapted to the  plasma dynamics.

For each regime $m$, we sample temporal windows $\{(t_i,t_i+1)\}_{i=1}^{n}$ and apply two
shared MLP heads
\[
    u_\theta:\mathbb R^{768}\to\mathbb R^{d_z},
    \qquad
    v_\theta:\mathbb R^{768}\to\mathbb R^{d_z},
    \qquad
    d_z=64 .
\]
The current and future coordinates are
\[
    \mathbf Z_-^{(m)}
    =
    \big[
    u_\theta(\psi_{\rm P}(\mathbf x_{t_1}^{(m)})),\ldots,
    u_\theta(\psi_{\rm P}(\mathbf x_{t_n}^{(m)}))
    \big]
    \in\mathbb R^{d_z\times n},
\]
\[
    \mathbf Z_+^{(m)}
    =
    \big[
    v_\theta(\psi_{\rm P}(\mathbf x_{t_1+1}^{(m)})),\ldots,
    v_\theta(\psi_{\rm P}(\mathbf x_{t_n+1}^{(m)}))
    \big]
    \in\mathbb R^{d_z\times n}.
\]
We define the score matrix
\[
    \mathbf S^{(m)}
    =
    \big(\mathbf Z_-^{(m)}\big)^\top \mathbf Z_+^{(m)},
    \qquad
    S_{ij}^{(m)}
    =
    \left\langle
    u_\theta(\psi_{\rm P}(\mathbf x_{t_i}^{(m)})),
    v_\theta(\psi_{\rm P}(\mathbf x_{t_j+1}^{(m)}))
    \right\rangle .
\]
The NCP contrastive loss for regime $m$ is
\begin{equation}
\label{eq:app_temp_loss}
    \ell_m(\theta)
    =
    \frac{1}{n(n-1)}
    \sum_{i\neq j}
    \big(S_{ij}^{(m)}\big)^2
    -
    \frac{2}{n}
    \sum_{i=1}^{n}
    S_{ii}^{(m)}
    +
    \gamma
    \left[
    r(\mathbf Z_-^{(m)})+r(\mathbf Z_+^{(m)})
    \right].
\end{equation}
The diagonal term aligns true one-step pairs, while the off-diagonal term penalizes
mismatched temporal pairs.

The regularizer prevents collapse and keeps the coordinates well conditioned. For
$\mathbf Z\in\mathbb R^{d_z\times n}$, define
\[
    \bar{\mathbf z}
    =
    \frac{1}{n}\mathbf Z\mathbf 1,
    \qquad
    \widehat{\mathbf \Sigma}_{\mathbf Z}
    =
    \frac{1}{n}
    (\mathbf Z-\bar{\mathbf z}\mathbf 1^\top)
    (\mathbf Z-\bar{\mathbf z}\mathbf 1^\top)^\top .
\]
We use
\begin{equation}
\label{eq:app_temp_reg}
    r(\mathbf Z)
    =
    \frac{1}{d_z}
    \|\widehat{\mathbf \Sigma}_{\mathbf Z}-\mathbf I\|_F^2
    +
    2\|\bar{\mathbf z}\|_2^2 .
\end{equation}
This centers and approximately whitens the learned coordinates over each temporal window.

\subsection{Observable space learning on multiple dynamical systems with GradNorm multitask strategy.}
\label{app:plasma_gradnorm}

\paragraph{Multitask balancing across plasma regimes.}
Each Tokam2D regime $m$ corresponds to a parameter pair $(g^{(m)},\kappa^{(m)})$ and is
treated as one task. Let $\mathcal M_{\rm tr}$ be the set of training regimes and
$\ell_m(\theta)$ the temporal representation loss of regime $m$. A uniform objective,
$|\mathcal M_{\rm tr}|^{-1}\sum_{m\in\mathcal M_{\rm tr}}\ell_m(\theta)$, can be poorly
balanced because regimes have different temporal complexity. The parameter $\kappa$
controls the density-gradient drive, while $g$ acts through the interchange coupling
between density and vorticity; regimes where both effects are large can exhibit stronger
fluctuations and richer temporal spectra.

We therefore use a GradNorm-weighted multitask objective~\citep{chen2018gradnorm}. For a
mini-batch of regimes $\mathcal B\subset\mathcal M_{\rm tr}$,
\begin{equation}
\label{eq:app_mtl_loss}
    \mathcal L_{\rm MTL}(\theta,\omega)
    =
    \sum_{m\in\mathcal B}\omega_m\ell_m(\theta),
    \qquad
    \omega_m
    =
    |\mathcal M_{\rm tr}|
    \frac{\exp(\beta_m)}
    {\sum_{j\in\mathcal M_{\rm tr}}\exp(\beta_j)} ,
\end{equation}
so that the average task weight over training regimes is one. The logits $\beta_m$ are
learned with GradNorm, while $\theta$ denotes the shared temporal-head parameters.

Let $\ell_m(0)$ be the initial loss of regime $m$. At iteration $t$, define the relative
training rate
\[
    r_m(t)
    =
    \frac{\ell_m(t)/\ell_m(0)}
    {|\mathcal B|^{-1}\sum_{j\in\mathcal B}\ell_j(t)/\ell_j(0)} .
\]
Thus, $r_m(t)>1$ means that regime $m$ is learning more slowly than the mini-batch average.
Let $W_{\rm ref}$ be the last shared layer of the temporal heads. GradNorm measures
\[
    G_m(t)
    =
    \left\|
    \nabla_{W_{\rm ref}}
    \bigl(\omega_m\ell_m(t)\bigr)
    \right\|_2,
    \qquad
    \widetilde G_m(t)
    =
    \overline G(t) r_m(t)^{\alpha_{\rm GN}},
\]
where
\[
    \overline G(t)
    =
    |\mathcal B|^{-1}
    \sum_{j\in\mathcal B}G_j(t),
    \qquad
    \alpha_{\rm GN}=0.12 .
\]
The task weights are updated by minimizing
\begin{equation}
\label{eq:app_gradnorm_loss}
    \mathcal L_{\rm GN}
    =
    \sum_{m\in\mathcal B}
    \bigl|G_m(t)-\widetilde G_m(t)\bigr| ,
\end{equation}
and the shared representation parameters are updated with $\mathcal L_{\rm MTL}$.

This balancing does not force all regimes to be equally easy. It prevents regimes with
larger loss scale or stronger turbulent activity from dominating the shared representation.
In practice, higher weights on high-drive or strongly coupled regimes are consistent with
their richer temporal spectra and help learn coordinates that remain usefu.
\begin{figure}[H]
    \centering
    \begin{minipage}{0.48\linewidth}
        \centering
        \includegraphics[width=\linewidth]{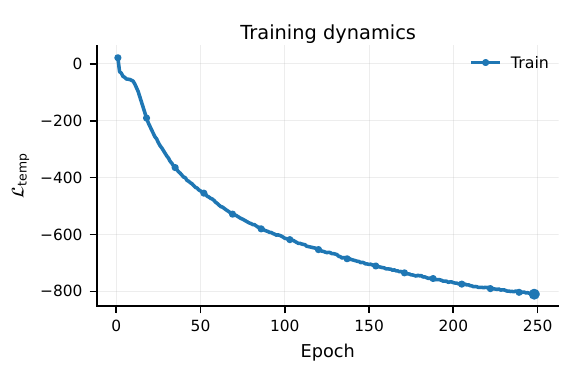}
    \end{minipage}
    \hfill
    \begin{minipage}{0.45\linewidth}
        \centering
        \includegraphics[width=\linewidth]{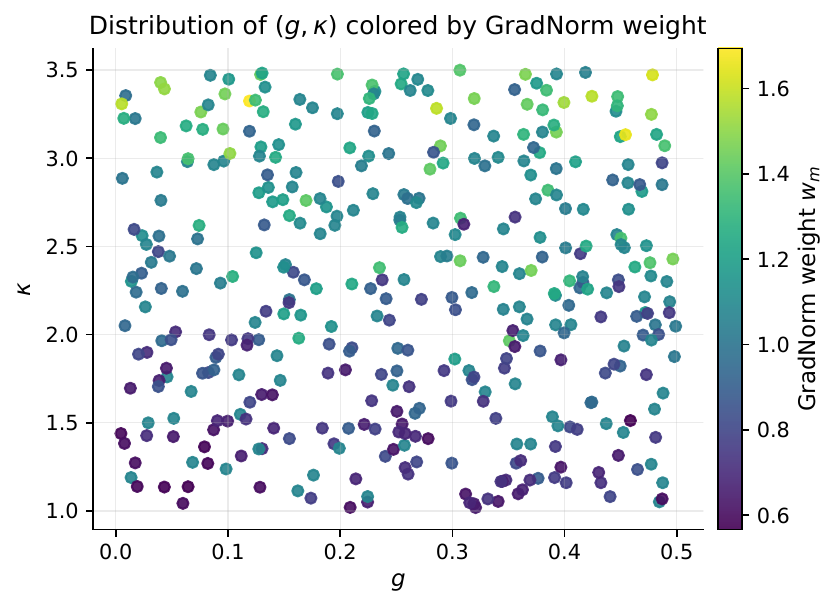}
        \label{fig:app_gradnorm_weights}
    
    \end{minipage}
    
    \caption{
Temporal representation learning diagnostics. Left: training objective in
\eqref{eq:app_temp_loss}. Right: final GradNorm weights over the $(g,\kappa)$ map.
Larger weights indicate regimes that require more optimization effort in the shared
temporal representation.
}
    \label{fig:app_temporal_gradnorm}
\end{figure}

\begin{figure}[H]
    \centering
    \includegraphics[width=0.98\linewidth]{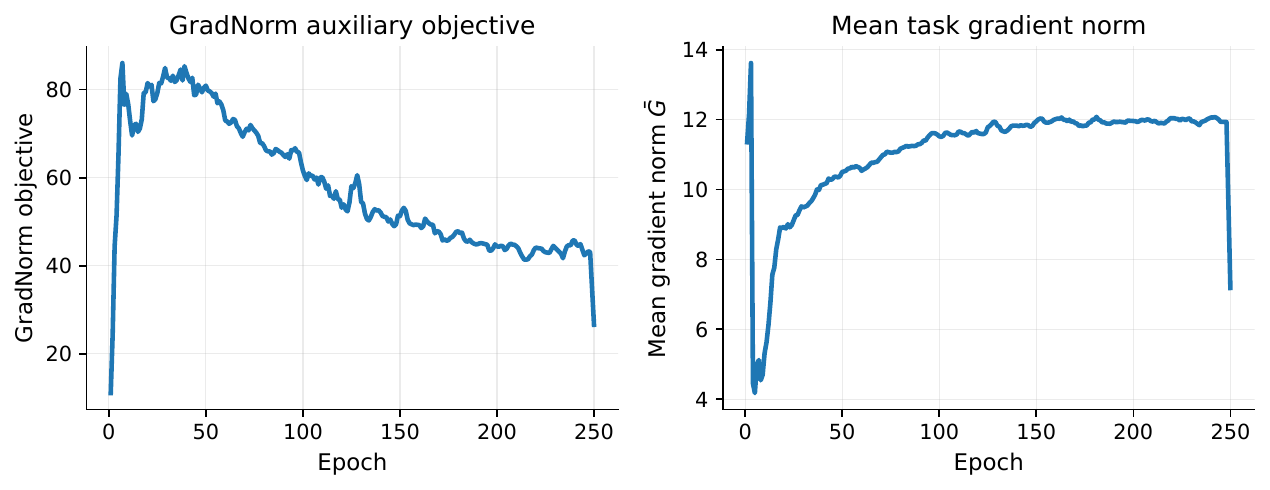}
\caption{
GradNorm optimization diagnostics. Left: auxiliary GradNorm objective. Right: gradient
norm of the main multitask objective. The curves show that the task weights stabilize
during training.
}
\label{fig:app_gradnorm_diagnostics}
\end{figure}

The final GradNorm weights in Figure~\ref{fig:app_gradnorm_weights} indicate where the
shared temporal representation has more difficulty fitting all regimes simultaneously. The
largest weights are concentrated at high $\kappa$, matching the density snapshots where
larger $\kappa$ produces stronger, more organized fluctuations. Increasing $g$ alone has a
weaker effect on the weights, consistent with its more indirect role through the
density--vorticity coupling. Its influence becomes more visible at high $\kappa$, where the
potential snapshots show more distorted flow structures. Overall, the weighting diagnostic
suggests that the hardest regimes are those with stronger turbulent activity and richer
temporal content.

\paragraph{Reduced-rank validation of the temporal representation.}
We use reduced-rank regression (RRR) only as a validation diagnostic for the temporal
representation; it is not the spectral estimator used by \DOODL{}. The goal is to check
whether the learned coordinates admit a predictable finite-rank one-step evolution.

For each regime $m$, we form one-step pairs in the learned coordinate space,
\[
    \mathbf z_i^{-,(m)}
    =
    u_\theta\!\left(\psi_{\rm P}(\mathbf x_{t_i}^{(m)})\right),
    \qquad
    \mathbf z_i^{+,(m)}
    =
    u_\theta\!\left(\psi_{\rm P}(\mathbf x_{t_i+1}^{(m)})\right).
\]
For a candidate rank $r$, we fit a ridge-regularized reduced-rank predictor
\begin{equation}
\label{eq:app_plasma_rrr}
    \widehat A_{m,r}
    \in
    \arg\min_{\mathrm{rank}(A)\le r}
    \frac{1}{|\mathcal I_{\rm tr}^{(m)}|}
    \sum_{i\in\mathcal I_{\rm tr}^{(m)}}
    \left\|
        A\mathbf z_i^{-,(m)}-\mathbf z_i^{+,(m)}
    \right\|_2^2
    +
    \lambda\|A\|_F^2 .
\end{equation}
We first select the diagnostic rank by evaluating the held-out one-step error
\[
    \mathrm{Err}_{\rm RRR}(r)
    =
    \frac{1}{|\mathcal M_{\rm val}|}
    \sum_{m\in\mathcal M_{\rm val}}
    \frac{1}{|\mathcal I_{\rm val}^{(m)}|}
    \sum_{i\in\mathcal I_{\rm val}^{(m)}}
    \left\|
        \widehat A_{m,r}\mathbf z_i^{-,(m)}
        -
        \mathbf z_i^{+,(m)}
    \right\|_2^2 .
\]
Figure~\ref{fig:app_rrr_rank_selection} shows that the error decreases at small ranks and
then saturates. We therefore use rank $r_{\rm eval}=40$ for checkpoint selection.

\begin{figure}[H]
    \centering
    \includegraphics[width=0.62\linewidth]{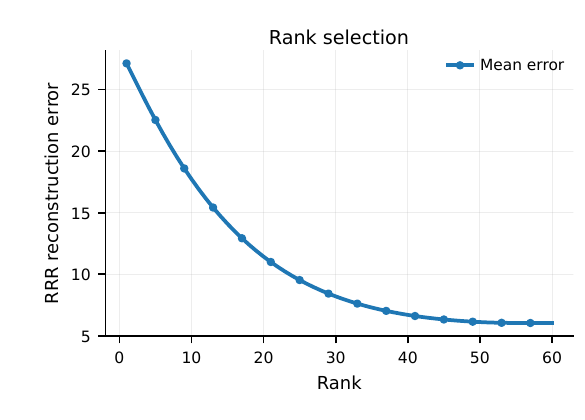}
    \caption{
    Diagnostic rank selection for reduced-rank validation. The held-out one-step error
    decreases at small ranks and then saturates, indicating that a moderate-rank predictor
    captures most of the predictable temporal structure in the learned coordinates.
    }
    \label{fig:app_rrr_rank_selection}
\end{figure}

Given $r_{\rm eval}=40$, we monitor the one-step coefficient of determination
\[
    R^2_{\rm RRR}
    =
    1
    -
    \frac{
    \sum_{m}\sum_{i\in\mathcal I^{(m)}}
    \left\|
        \widehat A_{m,r_{\rm eval}}\mathbf z_i^{-,(m)}
        -
        \mathbf z_i^{+,(m)}
    \right\|_2^2
    }{
    \sum_{m}\sum_{i\in\mathcal I^{(m)}}
    \left\|
        \mathbf z_i^{+,(m)}-\bar{\mathbf z}^{+,(m)}
    \right\|_2^2
    },
    \qquad
    \bar{\mathbf z}^{+,(m)}
    =
    \frac{1}{|\mathcal I^{(m)}|}
    \sum_{i\in\mathcal I^{(m)}}\mathbf z_i^{+,(m)} .
\]
This score measures how much one-step variation is explained by a finite-rank linear
predictor, relative to a constant predictor. We compute it on both training and validation
regimes and select the checkpoint maximizing validation $R^2_{\rm RRR}$.
\begin{figure}[H]
    \centering
    \includegraphics[width=0.62\linewidth]{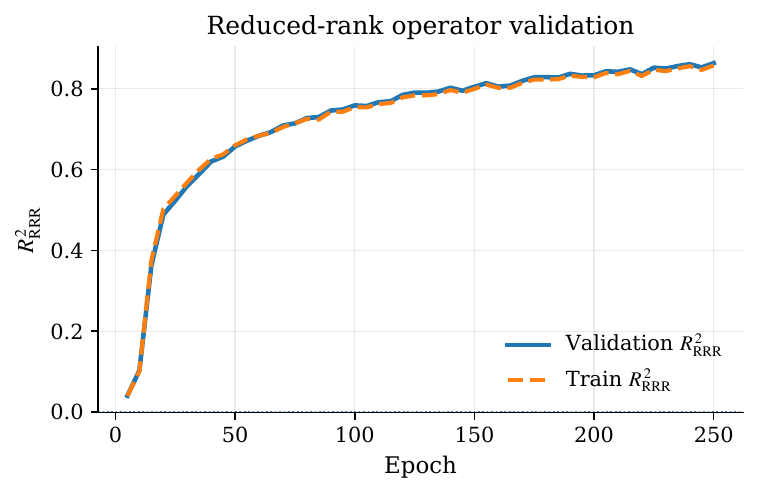}
    \caption{
    Checkpoint selection with reduced-rank validation. Train and validation
    $R^2_{\rm RRR}$ measure finite-rank one-step predictability in the learned coordinates.
    The selected checkpoint maximizes validation $R^2_{\rm RRR}$, while the train/validation
    comparison checks that the representation improves without relying on physical labels.
    }
    \label{fig:app_rrr_train_val}
\end{figure}
\subsection{Generator-resolvent estimation}
\label{app:plasma_gr}

For each regime $m$, the temporal representation step produces the latent trajectory
\[
    \mathbf z_{1:T}^{(m)}
    =
    \bigl(\mathbf z_1^{(m)},\ldots,\mathbf z_T^{(m)}\bigr),
    \qquad
    \mathbf z_t^{(m)}
    =
    u_\theta\!\left(\psi_{\rm P}(\mathbf x_t^{(m)})\right)
    \in\mathbb R^{d_z}.
\]
From this trajectory, we estimate a finite-rank spectral generator representation
\[
    \widehat G_m
    =
    (\widehat\Lambda_m,\widehat L_m,\widehat R_m)
\]
using a generator-resolvent (GR) estimator~\citep{kostic2024laplace}. The GR estimator is
well suited to plasma trajectories because it aggregates time-lagged information over
multiple lags. This helps capture dynamics occurring at different time scales, including
oscillatory drift/interchange activity and slower transport structures, while producing a
stable spectral representation of the generator.

The estimated eigenvalues
$\widehat\lambda_i=\widehat\sigma_i+\mathrm{i}\widehat\omega_i$ encode growth or decay
through $\widehat\sigma_i$ and oscillatory content through $\widehat\omega_i$, while the
spectral factors $(\widehat L_m,\widehat R_m)$ encode the associated modes. The resulting
$\widehat G_m$ is the operator representation passed to \DOODL{}. Details of the
generator-resolvent construction are given in \Cref{app:operators-repr}.
\paragraph{Choice of shift, lag, and rank.}
For the plasma experiments, each learned latent trajectory
$\mathbf z_{1:T}^{(m)}=(\mathbf z_1^{(m)},\ldots,\mathbf z_T^{(m)})$ is mapped to a
finite-rank spectral generator representation
\[
    \widehat G_m=(\widehat\Lambda_m,\widehat L_m,\widehat R_m),
\]
where $\widehat\Lambda_m$ contains the recovered generator eigenvalues and
$\widehat L_m,\widehat R_m$ are the associated left and right spectral factors. We use a
generator-resolvent (GR) filter with
\[
    a=8,\qquad K=600,\qquad r_{\rm GR}=40,
\]
Here $a$ is the GR shift, $K$ is the maximum lag used by the time-lagged filter, and
$r_{\rm GR}$ is the number of retained spectral components. The eigenvalues
$\widehat\nu_i$ of the GR-filtered operator are mapped back to generator eigenvalues by
\[
    \widehat\lambda_i
    =
    a\left(1-\widehat\nu_i^{-1}\right).
\]
All spectra shown below are therefore spectra of the recovered generator, not spectra of
the resolvent.

The shift $a$ controls the resolution--stability tradeoff of the recovered generator
spectrum. Small shifts reveal finer spectral variation but are more sensitive to numerical
noise, whereas large shifts smooth the estimate but compress the recovered eigenvalues. The
value $a=8$ gives a stable compromise for the downstream SGOT geometry. The maximum lag
controls the temporal support of the GR estimate: with $\Delta t=0.05$, $K=600$
corresponds to a time window of length $30$. This provides enough temporal context to
capture slow transport and oscillatory content while remaining stable across regimes.
Finally, $r_{\rm GR}=40$ retains a rich spectral description, including oscillatory and
subspace information useful for \DOODL{}, without making the representation dominated by
weak numerical modes.

\begin{figure}[H]
    \centering
    \includegraphics[width=\linewidth]{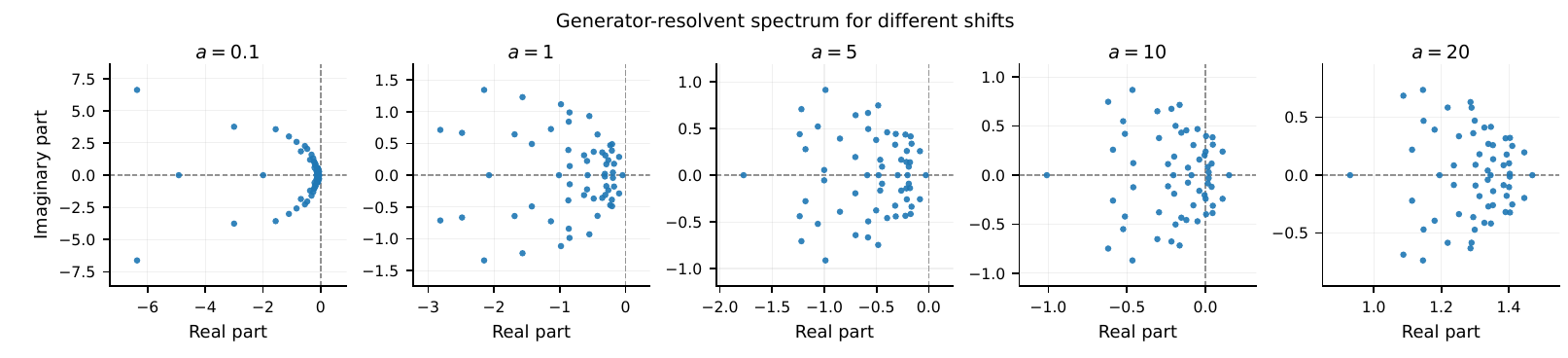}
    \caption{
    Recovered generator spectra for different GR shifts $a$, with fixed maximum lag.
    The plotted points are generator eigenvalues obtained after inversion of the GR
    filter. Small shifts reveal finer spectral variation, while larger shifts produce
    smoother but more compressed spectra. We use $a=8$ in the plasma pipeline.
    }
    \label{fig:app_gr_shift}
\end{figure}

Figure~\ref{fig:app_gr_shift} shows that the recovered generator spectra remain mostly in
the stable half-plane and retain nonzero imaginary components. This is consistent with the
operator structure expected for statistically stable plasma regimes: negative real parts
encode decay or relaxation, while imaginary parts encode oscillatory drift/interchange
content. Since the latent dynamics are real-valued, complex eigenvalues appear in conjugate
pairs up to numerical error.

\begin{figure}[H]
    \centering
    \includegraphics[width=\linewidth]{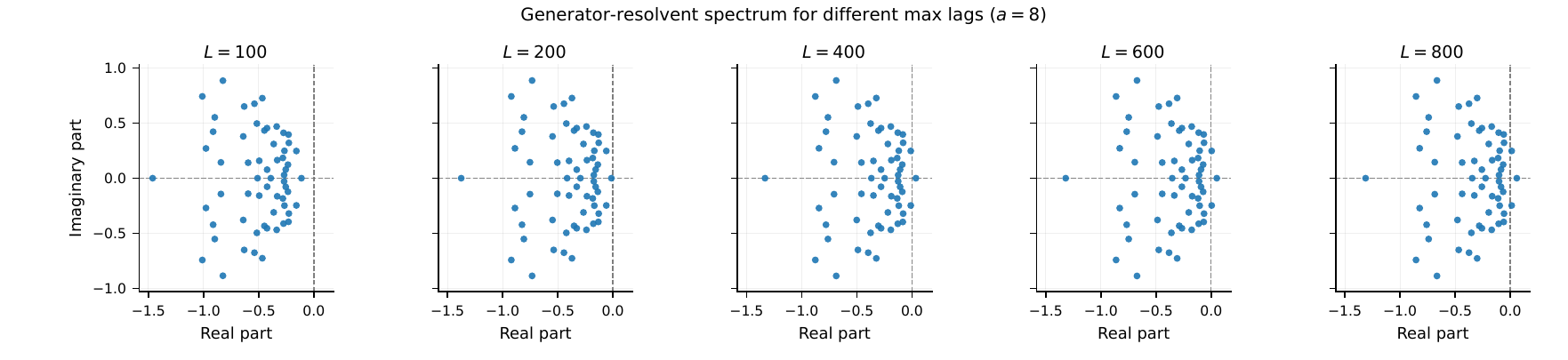}
    \caption{
    Recovered generator spectra for different maximum lags $K$, with fixed GR shift
    $a=8$. The plotted points are generator eigenvalues obtained after inversion of the GR
    filter. The dominant spectral organization is stable across the tested lags,
    supporting the use of $K=600$.
    }
    \label{fig:app_gr_lag}
\end{figure}

Figure~\ref{fig:app_gr_lag} shows that the main spectral structure is stable across the
tested lag values. We therefore use $K=600$, which provides enough temporal context for
slow modes without visibly destabilizing the recovered spectrum.

\paragraph{Generator-resolvent eigenfunction prediction score.}
We use an eigenfunction prediction score as a dynamic consistency diagnostic. Let
$\widehat\lambda_i$ be the recovered generator eigenvalues and $\widehat\psi_i(t)$ the
corresponding right eigenfunction evaluations along the trajectory. A consistent generator
estimate should satisfy
\[
    \widehat\psi_i(t+1)
    \simeq
    e^{\Delta t\,\widehat\lambda_i}\widehat\psi_i(t).
\]
We report
\[
    R^2_{\rm GR}
    =
    1-
    \frac{
    \sum_{t=1}^{T-1}
    \sum_{i=1}^{r_{\rm GR}}
    \left|
    \widehat\psi_i(t+1)
    -
    e^{\Delta t\,\widehat\lambda_i}
    \widehat\psi_i(t)
    \right|^2
    }{
    \sum_{t=1}^{T-1}
    \sum_{i=1}^{r_{\rm GR}}
    \left|
    \widehat\psi_i(t+1)-\bar\psi_i
    \right|^2
    },
    \qquad
    \bar\psi_i
    =
    \frac{1}{T-1}\sum_{t=1}^{T-1}\widehat\psi_i(t+1).
\]
This score measures whether the recovered spectral coordinates evolve according to the
estimated generator eigenvalues.
. In the selected configuration, we obtain $R^2_{\rm GR}=0.97$,
\subsection{Spectral coordinates and physical parameters}
\label{app:plasma_spectral_coordinates}

We finally inspect whether the recovered generator spectra retain the physical parameters of
the Tokam2D dynamical systems from unsupervised learning.

For each regime $m\in[M]$, the generator-resolvent estimator returns generator eigenvalues
$\widehat{\lambda}_i^{(m)}=\widehat{\sigma}_i^{(m)}+
\mathrm{i}\widehat{\omega}_i^{(m)}$, $i=1,\ldots,r_{\rm GR}$. We encode the recovered
spectrum by
\[
    \mathbf{s}^{(m)}
    =
    \big(
    \widehat{\sigma}_1^{(m)},\ldots,\widehat{\sigma}_{r_{\rm GR}}^{(m)},
    |\widehat{\omega}_1^{(m)}|,\ldots,|\widehat{\omega}_{r_{\rm GR}}^{(m)}|
    \big)^\top
    \in\mathbb{R}^{2r_{\rm GR}} .
\]
The real parts encode decay or growth rates, while the absolute imaginary parts encode
frequency scales. We use absolute values because non-real eigenvalues appear in complex
conjugate pairs. Stacking all regimes gives a matrix
$\mathbf{S}\in\mathbb{R}^{M\times 2r_{\rm GR}}$, whose columns are standardized across
regimes. The parameter vectors $\mathbf{g},\boldsymbol{\kappa}\in\mathbb{R}^M$ are
standardized as well.
\paragraph{Marginal spectral associations.}
For $\mathbf{y}\in\{\mathbf{g},\boldsymbol{\kappa}\}$, we compute the marginal association
between the $j$-th spectral coordinate and $\mathbf{y}$ as
\[
    \rho_{y,j}
    =
    \frac{\langle \mathbf{S}_{:,j},\mathbf{y}\rangle}
    {\|\mathbf{S}_{:,j}\|_2\|\mathbf{y}\|_2},
    \qquad
    j=1,\ldots,2r_{\rm GR}.
\]
Thus, $\rho_{y,j}$ is the empirical correlation between a recovered generator coordinate
and the physical parameter.

\begin{figure}[H]
    \centering
    \includegraphics[width=0.56\linewidth]{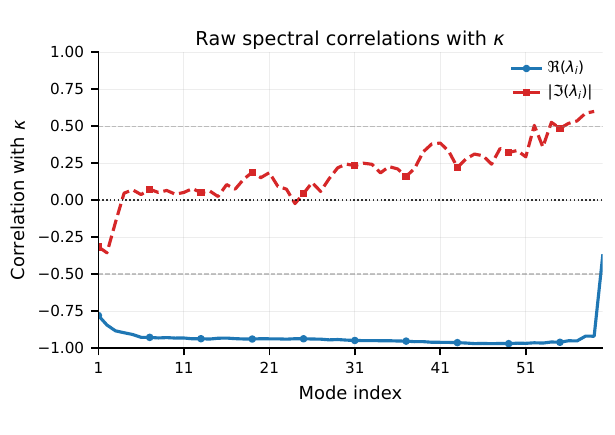}
    \caption{
    Marginal correlations between recovered generator coordinates and the density-gradient
    parameter $\kappa$. Large coefficients identify decay-rate or frequency coordinates
    that vary systematically with the imposed density-gradient drive.
    }
    \label{fig:app_corr_kappa_raw}
\end{figure}

Figure~\ref{fig:app_corr_kappa_raw} shows a strong spectral signature of $\kappa$. This is
consistent with the model: $\kappa$ directly controls the density-gradient drive, and
therefore changes the dominant activity level, spatial scale, and temporal organization of
the plasma dynamics.

\paragraph{Residual association with $g$ after regressing out $\kappa$.}
The interchange parameter $g$ has a weaker marginal spectral signature than $\kappa$.
This is consistent with the Tokam2D equations: $\kappa$ enters as the density-gradient
drive, whereas $g$ acts through the coupling between density perturbations and vorticity.
Its effect can therefore be partially masked by the dominant variation induced by
$\kappa$.

To isolate the part of the recovered spectrum associated with $g$ beyond this dominant
$\kappa$-direction, we remove the linear effect of $\kappa$ from both $g$ and each
spectral coordinate. Let
$\mathbf S\in\mathbb R^{M\times 2r_{\rm GR}}$ be the standardized matrix of spectral
coordinates, where the $j$-th column $\mathbf S_{:,j}$ contains one coordinate across the
$M$ regimes. We also denote by
$\boldsymbol\kappa,\mathbf g\in\mathbb R^M$ the standardized parameter vectors.

For each spectral coordinate $j$, we fit the linear regression
\[
    \mathbf S_{:,j}
    =
    \beta_{0,j}\mathbf 1
    +
    \beta_{1,j}\boldsymbol\kappa
    +
    \mathbf r_j^{(\kappa)} ,
\]
where $\mathbf r_j^{(\kappa)}$ is the residual part of the $j$-th generator coordinate
after removing its linear dependence on $\kappa$. We similarly regress $g$ on
$(\mathbf 1,\boldsymbol\kappa)$,
\[
    \mathbf g
    =
    \gamma_0\mathbf 1
    +
    \gamma_1\boldsymbol\kappa
    +
    \mathbf r_g^{(\kappa)} .
\]
We then define
\[
    \rho_{g\mid\kappa,j}
    =
    \frac{
    \left\langle
    \mathbf r_j^{(\kappa)},\mathbf r_g^{(\kappa)}
    \right\rangle
    }{
    \left\|\mathbf r_j^{(\kappa)}\right\|_2
    \left\|\mathbf r_g^{(\kappa)}\right\|_2
    },
    \qquad
    j=1,\ldots,2r_{\rm GR}.
\]
This is the empirical partial correlation between the $j$-th recovered generator
coordinate and $g$, after removing the linear contribution of $\kappa$.

\begin{figure}[t]
    \centering
    \begin{minipage}{0.51\linewidth}
        \centering
        \includegraphics[width=\linewidth]{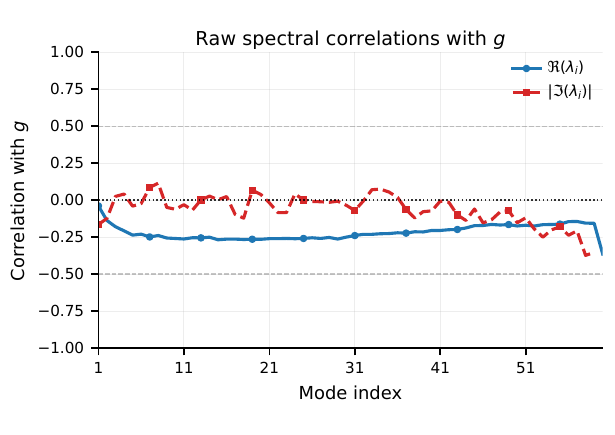}
    \end{minipage}
    \hfill
    \begin{minipage}{0.48\linewidth}
        \centering
        \includegraphics[width=\linewidth]{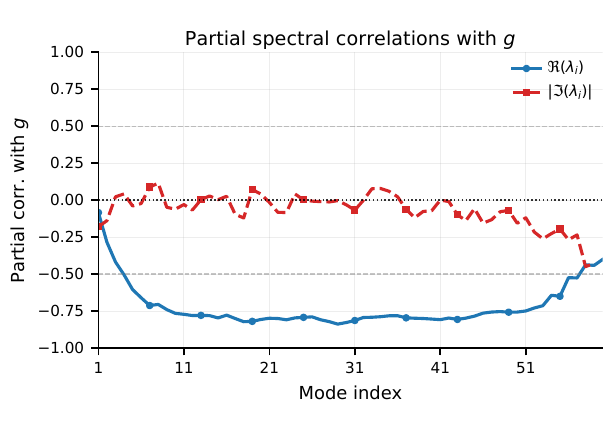}
    \end{minipage}
    \caption{
    Spectral correlations with the interchange parameter $g$. Left: marginal correlations
    between $g$ and the recovered generator coordinates. Right: partial correlations
    obtained after regressing out the linear effect of $\kappa$ from both $g$ and each
    spectral coordinate. The residual signal shows that the recovered generator spectrum
    retains information about $g$ beyond the dominant $\kappa$-driven variation.
    }
    \label{fig:app_corr_g_raw_residual}
\end{figure}

Figure~\ref{fig:app_corr_g_raw_residual} shows that the marginal spectral correlations
with $g$ are weaker than those for $\kappa$, but that a clear residual signal remains after
regressing out $\kappa$. This indicates that the recovered generator spectrum does not only
encode the dominant density-gradient drive. It also retains information about the
interchange parameter through more distributed spectral coordinates, consistent with the
role of $g$ in the density--vorticity coupling.
\paragraph{Representative spectral coordinates.}
We visualize representative coordinates by selecting those with the largest values of
$|\rho_{\kappa,j}|$ and $|\rho_{g\mid\kappa,j}|$. The first group corresponds to generator eigenvalues with the most important correlation with $\kappa$, while the second group captures
the strongest residual correlation with $g$.

\begin{figure}[t]
    \centering
    \includegraphics[width=0.6\linewidth]{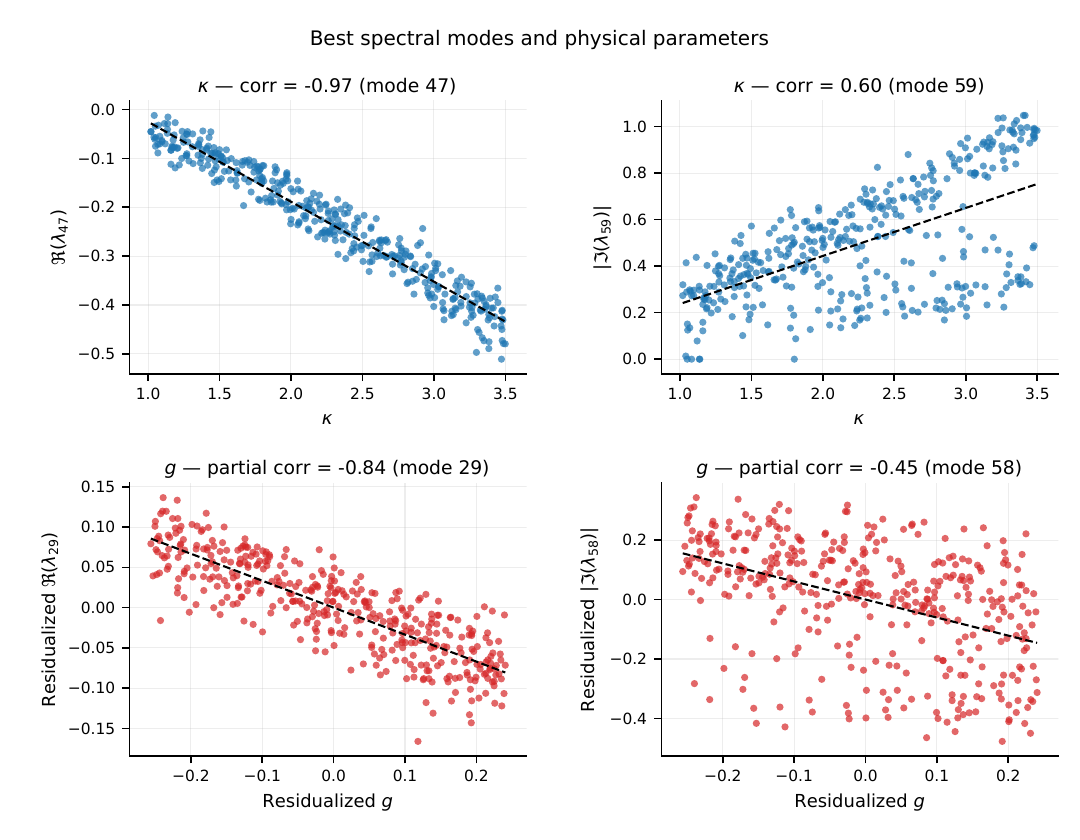}
    \caption{
    Representative recovered generator coordinates. Top: coordinates most associated with
    $\kappa$. Bottom: coordinates most associated with $g$ after removing the effect of
    $\kappa$. The plots show that $\kappa$ is encoded in dominant spectral directions,
    while $g$ appears through weaker but structured residual variation.
    }
    \label{fig:app_scatter_spectral_modes}
\end{figure}

Figure~\ref{fig:app_scatter_spectral_modes} makes the previous correlations explicit.
Coordinates associated with $\kappa$ organize regimes along the main density-gradient
direction. Coordinates associated with $g$ are less dominant, but remain structured once
the leading $\kappa$ effect is removed. Overall, the recovered generator coordinates are
therefore not arbitrary numerical features: they retain interpretable information about the
physical control parameters of the Tokam2D dynamical systems. 

\paragraph{SGOT geometry before dictionary projection.}
We also inspect the operator geometry before fitting the \DOODL{} dictionary. From the
generator-resolvent representations $\{\widehat G_m\}_{m=1}^M$, we compute the pairwise
SGOT distance matrix using the log-Martin projector distance and spectral weight $\eta=0.9$,
then embed this distance matrix with t-SNE.

\begin{figure}[t]
    \centering
    \includegraphics[width=0.92\linewidth]{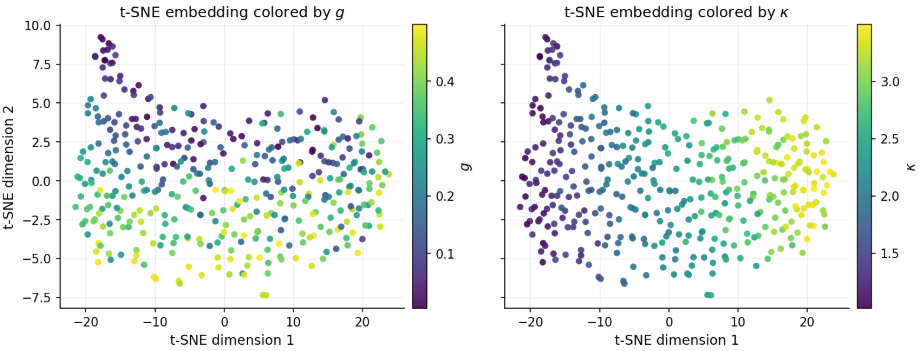}
    \caption{
    t-SNE embedding of the pairwise SGOT distance matrix between recovered generator
    operators, colored by the physical parameters $g$ and $\kappa$. The embedding shows a
    clear organization along $\kappa$ and a weaker but visible organization along $g$,
    indicating that physical information is already present in the operator geometry before
    dictionary learning.
    }
    \label{fig:app_sgot_tsne_raw}
\end{figure}

Figure~\ref{fig:app_sgot_tsne_raw} shows that the SGOT geometry of the recovered generators
is already structured by the scanned parameters. The organization is strongest for
$\kappa$, consistent with its direct role as the density-gradient drive, while $g$ appears
more weakly and in a less monotone way. This supports the role of \DOODL{} as a compression
of an already informative operator geometry, rather than a supervised mechanism for
creating parameter information.

\subsection{Dictionary learning and parameter recovery}
\label{app:plasma_dictionary_regression}

We apply \DOODL{} to the full-trajectory generator representations
$\widehat G_m=(\widehat\Lambda_m,\widehat L_m,\widehat R_m)$, $m=1,\ldots,M$. The
dictionary contains $d=12$ atoms and is trained with the SGOT distance $d_S$, using the
log-Martin projector distance and spectral weight $\eta=0.9$.The training of each dictionary last for 6 epochs with batch sizes of 64. Dictionary is learned
following the procedure described in \cref{app: optimization scheme} with the default settings. Each regime is represented
by barycentric coordinates $\widehat{\boldsymbol\alpha}_m\in\Delta^{d-1}$ obtained by
projecting $\widehat G_m$ onto the learned operator dictionary, i.e. by fitting the SGOT
barycenter $B_{\widehat{\boldsymbol\alpha}_m}(D_1,\ldots,D_d)$ to $\widehat G_m$.

To test whether the unsupervised coordinates retain physical information, we fit two
independent degree-2 polynomial ridge regressors on the training regimes, one from
$\widehat{\boldsymbol\alpha}_m$ to $g^{(m)}$ and one from
$\widehat{\boldsymbol\alpha}_m$ to $\kappa^{(m)}$.. As shown in
\Cref{fig:plasma_sgot_regression}, regression from full-trajectory \DOODL{} coordinates
reaches $R_g^2=0.811$ and $R_\kappa^2=0.967$ on held-out regimes. The stronger recovery of
$\kappa$ is consistent with its dominant role as density-gradient drive, while the recovery
of $g$ shows that the dictionary also preserves information about the more indirect
interchange dynamics.

\subsection{Early operator and parameter recovery}
\label{app:plasma_early_recovery}

We also evaluate short trajectories. For a trajectory length $\tau<T$, let
$\widehat{\boldsymbol\alpha}_{m,\tau}$ denote the \DOODL{} coordinates estimated from the
short trajectory of regime $m$. These coordinates define the dictionary-constrained
operator
$\widehat G_{m,\tau}^{\DOODL}=B_{\widehat{\boldsymbol\alpha}_{m,\tau}}(D_1,\ldots,D_d)$.
We compare this estimate with a direct generator-resolvent estimate from the same short
trajectory, using the SGOT distance to the full-trajectory operator $\widehat G_{m,T}$.
As shown in \Cref{fig:plasma_early_operator}, the dictionary-constrained estimate is closer
to $\widehat G_{m,T}$ in the short-data regime. This supports the interpretation of
\DOODL{} as an operator-level prior: early estimation is reduced to fitting coordinates in
a learned spectral dictionary, rather than estimating an unconstrained generator from few
observations.

For early parameter recovery, we train a separate diagnostic regressor. Full-trajectory
recovery uses only the full-trajectory coordinates $\widehat{\boldsymbol\alpha}_m$. Early
recovery uses data augmentation: the training set includes dictionary coordinates obtained
from both full and short trajectories. For this early regressor, we use the pre-softmax
dictionary logits $\widehat{\boldsymbol\ell}_{m,\tau}\in\mathbb R^d$ rather than the
simplex-normalized coordinates,  We again fit two independent degree-2 polynomial ridge
regressors, one for $g$ and one for $\kappa$. This augmentation is used only for the
diagnostic parameter regressor; it does not change the learned operator dictionary.

\Cref{fig:plasma_early_operator} shows that parameter recovery improves as longer
trajectories become available. The recovery of $\kappa$ rises quickly, consistent with its
dominant spectral signature, while $g$ improves more gradually, consistent with its more
distributed residual signature in the generator spectrum. At the longest trajectory length,
the early-recovery protocol reaches $R_g^2=0.81$ and $R_\kappa^2=0.97$.
\Cref{fig:plasma_early_operator} show that short trajectories
already contain useful operator-level information, and that the learned \DOODL{}
dictionary makes this information recoverable in both SGOT geometry and physical parameter
space.
\subsection{Hyperparameters}
\label{app:plasma_hyperparameters}

\begin{table}[H]
\centering
\caption{Main hyperparameters used in the Tokam2D plasma pipeline.}
\label{tab:app_plasma_hyperparameters}
\begin{tabular}{ll}
\toprule
Parameter & Value \\
\midrule
Number of regimes & \(400\) \\
Train/test split & \(80/20\) over regimes \\
Grid size & \(128\times128\) \\
Time step & \(\Delta t=0.05\) \\
POSEIDON descriptor dimension & \(768\) \\
Temporal latent dimension & \(d=64\) \\
Temporal heads \(u_\theta,v_\theta\) & two MLPs [512,256,64]  \\
Task definition & one regime = one task \\
Task balancing & GradNorm \\
GradNorm exponent & \(\alpha=0.12\) \\
Maximum epochs & \(250\) \\
Validation metric & \(R^2_{\mathrm{RRR}}\) \\
RRR validation rank & \(r_{\mathrm{eval}}=40\) \\
Selected validation score & \(R^2_{\mathrm{RRR}}=0.86\) \\
Resolvent shift & \(a=8\) \\
Maximum lag & \(K=600\) \\
Generator resolvant rank & \(r=40\) \\
GR validation score & \(R^2_{\mathrm{GR}}=0.97\) \\
Regression diagnostic & degree-2 polynomial ridge \\
Full-trajectory regression scores & \(R_g^2=0.81,\;R_\kappa^2=0.97\) \\
\bottomrule
\end{tabular}
\end{table}

\section{Statistical Guarantees}
\label{app:stat}

\subsection{Setup and notation}

Throughout we assume $(X_t)_{t \geq 1}$ is stationary with marginal
$\pi$, so that $\mathrm{Var}_\pi(\ell_t(G))$ is time-independent
for every fixed $G$.
Let $\feat_t := \phi(X_t) \in \mcH$ denote the encoded state under
the feature map $\phi : \mcX \to \mcH$.
Let $\bar{\bs\mcG} = (\bar{\bs\mcG}_j)_{j \in [d]} \in \mcN^d$ be a
dictionary of $d$ atoms, $\tB(\bs\alpha; \bar{\bs\mcG}) \in \mcN$
the decoder (either $\tB^{\text{\tiny OT}}$ or $\tB^{\text{\tiny p}}$),
and
\begin{equation}
    G_{\bs\alpha} := \tO\tp[\tB(\bs\alpha; \bar{\bs\mcG})]
    \in \mcS_r(\mcH),
    \qquad
    \Gset(\bar{\bs\mcG})
    := \bigl\{ G_{\bs\alpha} : \bs\alpha \in \Delta^{d-1} \bigr\}
    \subset \mcS_r(\mcH).
\end{equation}
The sample loss at operator $G \in \mcS_r(\mcH)$ is
$\ell_t(G) := \|\feat_{t+1} - G^*\feat_t\|_{\mcH}^2$.
For a fixed dictionary $\bar{\bs\mcG}$, the population and empirical
risks are
\begin{equation}
    \risk(\bs\alpha \mid \bar{\bs\mcG})
    := \mathbb{E}_\pi[\ell_t(G_{\bs\alpha})],
    \qquad
    \erisk(\bs\alpha \mid \bar{\bs\mcG})
    := \frac{1}{\nsizeds}
       \sum_{t=1}^{\nsizeds} \ell_t(G_{\bs\alpha}),
\end{equation}
with minimizers $\bs\alpha^*(\bar{\bs\mcG}) \in \argmin_{\bs\alpha}
\risk(\bs\alpha\mid\bar{\bs\mcG})$ and
$\hat{\bs\alpha}(\bar{\bs\mcG}) \in \argmin_{\bs\alpha}
\erisk(\bs\alpha\mid\bar{\bs\mcG})$.
We identify the following quantities:
\begin{itemize}
    \item $\kappa := \sup_{x \in \mcX} \|\phi(x)\|_{\mcH}$,
          the uniform bound on the feature map;
    \item $M := \sup_{j \in [d]} \|\tO\tp[\bar{\bs\mcG}_j]\|_{\mathrm{op}}$,
          the operator norm of the dictionary atoms;
          $\sup_{\bs\alpha} \|G_{\bs\alpha}\|_{\mathrm{op}} \leq M$
          since SGOT barycenters do not increase operator norm;
    \item $m \leq d-1$, the intrinsic dimension of $\Gset(\bar{\bs\mcG})$
under $\|\cdot\|_{\mathrm{op}}$, with diameter $\leq 2M$;
    \item $\sigmavariance :=
          \sup_{G \in \Gset(\bar{\bs\mcG})} \mathrm{Var}_\pi(\ell_t(G))$,
          the uniform variance of the sample loss;
    \item $\sigmavariancebis :=
          \sup_{G \in \Gset(\bar{\bs\mcG})} \mathbb{E}_\pi[\ell_t(G)^2]$,
          the uniform second moment of the sample loss.
\end{itemize}

\subsection{Technical results}

\paragraph{Regularity and entropy of the decoder}

\begin{lmm}[Proposition~5 of \cite{germain2026spectral}]
\label{lem:sgot-op}
For all $G, G' \in \mcS_r(\mcH)$,
$\|G - G'\|_{\mathrm{op}} \leq C_{\mcS}\,\dS(G, G')$,
where $C_{\mcS} = 2\sqrt{2}\,C_{\xi\psi}$.
\end{lmm}

\begin{lmm}[Lipschitz continuity of the OT barycenter decoder]
\label{lem:ot-decoder-lipschitz}
Let $\bar{\bs\mcG} = (\bar{\bs\mcG}_j)_{j \in [d]} \in \mcN^d$ be a
fixed dictionary with atoms of bounded operator norm $M < \infty$.
Then $\bs\alpha \mapsto G_{\bs\alpha} =
\tO\tp[\tB^{\text{\tiny OT}}(\bs\alpha; \bar{\bs\mcG})]$
is $L_{\mathrm{dec}}$-Lipschitz from $(\Delta^{d-1}, \|\cdot\|_1)$
to $(\mcS_r(\mcH), \|\cdot\|_{\mathrm{op}})$, with
$L_{\mathrm{dec}} := C_{\mcS}\max_{j,k \in [d]}
\dS(\tO\tp[\bar{\bs\mcG}_j], \tO\tp[\bar{\bs\mcG}_k])$.
\end{lmm}

\begin{proof}
By stability of SGOT barycenters with respect to
weights~\citep{germain2026spectral},
\begin{equation}
    \dS(G_{\bs\alpha}, G_{\bs\alpha'})
    \leq
    \max_{j,k \in [d]}
    \dS\!\left(\tO\tp[\bar{\bs\mcG}_j],\,\tO\tp[\bar{\bs\mcG}_k]\right)
    \cdot \|\bs\alpha - \bs\alpha'\|_1.
\end{equation}
Then $\bs\alpha \mapsto G_{\bs\alpha} =
\tO\tp[\tB^{\text{\tiny OT}}(\bs\alpha; \bar{\bs\mcG})]$
is Lipschitz from $(\Delta^{d-1}, \|\cdot\|_1)$ to
$(\mcS_r(\mcH), \dS)$ with Lipschitz constant bounded from above by
$\max_{j,k \in [d]}
\dS(\tO\tp[\bar{\bs\mcG}_j], \tO\tp[\bar{\bs\mcG}_k])$.

Applying Lemma~\ref{lem:sgot-op},
$\|G_{\bs\alpha} - G_{\bs\alpha'}\|_{\mathrm{op}}
\leq C_{\mcS}\,\dS(G_{\bs\alpha}, G_{\bs\alpha'})
\leq C_{\mcS} L_{\mathrm{dec}}\|\bs\alpha - \bs\alpha'\|_1$.
\end{proof}

\begin{lmm}[Metric entropy of the OT barycenter decoder image]
\label{lem:metric-entropy}
Let $\Gset(\bar{\bs\mcG}) = \{G_{\bs\alpha} : \bs\alpha \in \Delta^{d-1}\}$
for a fixed dictionary $\bar{\bs\mcG}$, and let $m \leq d-1$ be the
intrinsic dimension of $\Gset(\bar{\bs\mcG})$ under $\|\cdot\|_{\mathrm{op}}$.
Then for every $\varepsilon > 0$,
\begin{equation}
    \log \mathcal{N}\!\left(\Gset(\bar{\bs\mcG}),\,\|\cdot\|_{\mathrm{op}},\,\varepsilon\right)
    \;\lesssim\;
    m\log\!\left(\frac{L_{\mathrm{dec}}}{\varepsilon}\right).
\end{equation}
\end{lmm}

\begin{proof}
Since $\Gset(\bar{\bs\mcG})$ has intrinsic dimension $m$ under
$\|\cdot\|_{\mathrm{op}}$, a standard volumetric argument gives
$\log\mathcal{N}(\Gset(\bar{\bs\mcG}),\|\cdot\|_{\mathrm{op}},\varepsilon)
\lesssim m\log(D_{\Gset}/\varepsilon)$,
where $D_{\Gset} \leq L_{\mathrm{dec}}$ is the diameter of
$\Gset(\bar{\bs\mcG})$ under $\|\cdot\|_{\mathrm{op}}$,
since $\Gset(\bar{\bs\mcG})$ is the image of $\Delta^{d-1}$
under an $L_{\mathrm{dec}}$-Lipschitz map from
$(\Delta^{d-1}, \|\cdot\|_1)$ of diameter $1$.
\end{proof}

\begin{lmm}[Lipschitz continuity of the projection-based decoder]
\label{lem:proj-decoder-lipschitz}
Let $\bar{\bs\mcG} = (\bar{\bs\mcG}_j)_{j \in [d]} \in \mcN^d$ be a
fixed dictionary with atoms $\bar{\bs\mcG}_j = (\bar{\bs\Lambda}_j,
\bar{\mbL}_j, \bar{\mbR}_j)$.
Assume that the dictionary atoms have bounded operator norm $M < \infty$
and that the right eigenvector aggregation is uniformly non-degenerate:
\begin{equation}
\label{eq:sigma-min}
    \sigma_{\min}\!\left(\sum_{j \in [d]} \alpha_j \bar{\mbR}_j\right)
    \geq c > 0,
    \qquad \forall\, \bs\alpha \in \Delta^{d-1}.
\end{equation}
Then $\bs\alpha \mapsto G_{\bs\alpha} =
\tO\tp[\tB^{\text{\tiny p}}(\bs\alpha; \bar{\bs\mcG})]$
is Lipschitz from $(\Delta^{d-1}, \|\cdot\|_1)$ to
$(\mcS_r(\mcH), \|\cdot\|_{\mathrm{op}})$ with constant
\begin{equation}
    L_{\mathrm{dec}}
    \;\lesssim\;
    M^2(1+M^2)^2/c^4,
\end{equation}
and in particular is $L_{\mathrm{dec}} C_{\mcS}$-Lipschitz from
$(\Delta^{d-1}, \|\cdot\|_1)$ to $(\mcS_r(\mcH), \dS)$.
\end{lmm}

\begin{proof}
Fix $\bs\alpha, \bs\alpha' \in \Delta^{d-1}$ and write
$\mbR := \sum_j \alpha_j \bar{\mbR}_j$,
$\mbR' := \sum_j \alpha'_j \bar{\mbR}_j$,
$\mbL := \sum_j \alpha_j \bar{\mbL}_j$,
$\mbL' := \sum_j \alpha'_j \bar{\mbL}_j$,
$\bs\Lambda := \sum_j \alpha_j \bar{\bs\Lambda}_j$,
$\bs\Lambda' := \sum_j \alpha'_j \bar{\bs\Lambda}_j$.
By linearity and $\|\cdot\|_1$-contractivity of convex combinations,
\begin{equation}
\label{eq:linear-lip}
    \|\mbR - \mbR'\|_{\mathrm{op}},\;
    \|\mbL - \mbL'\|_{\mathrm{op}},\;
    |\bs\Lambda - \bs\Lambda'|
    \;\leq\;
    M\|\bs\alpha - \bs\alpha'\|_1.
\end{equation}

\paragraph{Projection step.}
The projection $\tP_{\mcN}$ maps $(\bs\Lambda, \mbL, \mbR)$ to
$(\bs\Lambda, \tilde{\mbL}, \mbR)$ where
\begin{equation}
\label{eq:proj-solution}
    \tilde{\mbL}
    =
    \mbL
    +
    \mbR(\mbR^*\mbR)^{-1}(\mbI_r - \mbL^*\mbR)^*.
\end{equation}
To see this, note that any $\tilde{\mbL}$ satisfying
$\tilde{\mbL}^*\mbR = \mbI_r$ can be written as
$\tilde{\mbL} = \mbL + \mbR(\mbR^*\mbR)^{-1}\mbV^*$ for some
matrix $\mbV$, and minimizing $\|\tilde{\mbL} - \mbL\|_{\mathrm{op}}$
over $\mbV$ gives $\mbV = \mbI_r - \mbL^*\mbR$,
yielding~\eqref{eq:proj-solution}.
We bound $\|\tilde{\mbL} - \tilde{\mbL}'\|_{\mathrm{op}}$ by writing
\begin{align}
    \tilde{\mbL} - \tilde{\mbL}'
    &=
    (\mbL - \mbL')
    +
    \mbR(\mbR^*\mbR)^{-1}(\mbI_r - \mbL^*\mbR)^*
    -
    \mbR'(\mbR'^*\mbR')^{-1}(\mbI_r - \mbL'^*\mbR')^*.
\end{align}
The first term satisfies
$\|\mbL - \mbL'\|_{\mathrm{op}} \leq M\|\bs\alpha - \bs\alpha'\|_1$
by~\eqref{eq:linear-lip}.
For the second term, decompose as
\begin{align}
    &\mbR(\mbR^*\mbR)^{-1}(\mbI_r - \mbL^*\mbR)^*
    -
    \mbR'(\mbR'^*\mbR')^{-1}(\mbI_r - \mbL'^*\mbR')^*
    \notag\\
    &=
    \bigl[\mbR(\mbR^*\mbR)^{-1} - \mbR'(\mbR'^*\mbR')^{-1}\bigr]
    (\mbI_r - \mbL^*\mbR)^*
    +
    \mbR'(\mbR'^*\mbR')^{-1}
    \bigl[(\mbL'^*\mbR')^* - (\mbL^*\mbR)^*\bigr].
\end{align}

\emph{First bracket.}
Using
$\mbR(\mbR^*\mbR)^{-1} - \mbR'(\mbR'^*\mbR')^{-1}
= (\mbR-\mbR')(\mbR^*\mbR)^{-1}
+ \mbR'[(\mbR^*\mbR)^{-1} - (\mbR'^*\mbR')^{-1}]$,
the resolvent identity, $\|(\mbR^*\mbR)^{-1}\|_{\mathrm{op}} \leq c^{-2}$,
and $\|\mbR'^*\mbR' - \mbR^*\mbR\|_{\mathrm{op}} \leq 2M\|\mbR-\mbR'\|_{\mathrm{op}}$,
\begin{equation}
    \|\mbR(\mbR^*\mbR)^{-1} - \mbR'(\mbR'^*\mbR')^{-1}\|_{\mathrm{op}}
    \;\leq\;
    \frac{M(1 + 2M^2/c^2)}{c^2}
    \|\bs\alpha - \bs\alpha'\|_1.
\end{equation}
The factor $\|(\mbI_r - \mbL^*\mbR)^*\|_{\mathrm{op}}
\leq 1 + \|\mbL\|_{\mathrm{op}}\|\mbR\|_{\mathrm{op}} \leq 1 + M^2$.

\emph{Second bracket.}
$\|\mbL'^*\mbR' - \mbL^*\mbR\|_{\mathrm{op}}
\leq \|\mbL'-\mbL\|_{\mathrm{op}}\|\mbR'\|_{\mathrm{op}}
+ \|\mbL\|_{\mathrm{op}}\|\mbR'-\mbR\|_{\mathrm{op}}
\leq 2M^2\|\bs\alpha-\bs\alpha'\|_1$,
and $\|\mbR'(\mbR'^*\mbR')^{-1}\|_{\mathrm{op}} \leq M/c^2$.

Combining both brackets,
\begin{equation}
\label{eq:ltilde-bound}
    \|\tilde{\mbL} - \tilde{\mbL}'\|_{\mathrm{op}}
    \;\lesssim\;
    \frac{M(1+M^2)^2}{c^4}
    \|\bs\alpha - \bs\alpha'\|_1.
\end{equation}

\paragraph{Operator norm step.}
Writing
\begin{align}
    G_{\bs\alpha} - G_{\bs\alpha'}
    &=
    (\mbR - \mbR')\mathrm{Diag}(\bs\Lambda)\tilde{\mbL}^*
    +
    \mbR'\mathrm{Diag}(\bs\Lambda - \bs\Lambda')\tilde{\mbL}^*
    +
    \mbR'\mathrm{Diag}(\bs\Lambda')(\tilde{\mbL} - \tilde{\mbL}')^*,
\end{align}
and using $\|\mbR\|_{\mathrm{op}},\|\mbR'\|_{\mathrm{op}} \leq M$,
$|\bs\Lambda|,|\bs\Lambda'| \leq M$,
$\|\tilde{\mbL}\|_{\mathrm{op}} \leq M + M(1+M^2)/c^2$,
together with~\eqref{eq:linear-lip} and~\eqref{eq:ltilde-bound},
\begin{equation}
    \|G_{\bs\alpha} - G_{\bs\alpha'}\|_{\mathrm{op}}
    \;\lesssim\;
    \frac{M^2(1+M^2)^2}{c^4}
    \|\bs\alpha - \bs\alpha'\|_1.
\end{equation}
This establishes Lipschitz continuity with constant
$L_{\mathrm{dec}} \lesssim M^2(1+M^2)^2/c^4$.
\end{proof}

\begin{lmm}[Metric entropy of the projection-based decoder image, $\tB^{\text{\tiny p}}$]
\label{lem:metric-entropy-proj}
Let Assumption~\ref{ass:sigma-min} be satisfied. Then for every $\varepsilon > 0$,
\begin{equation}
    \log \mathcal{N}\!\left(\Gset(\bar{\bs\mcG}),\,\|\cdot\|_{\mathrm{op}},\,\varepsilon\right)
    \;\lesssim\;
    m\log\!\left(\frac{M}{c^2\,\varepsilon}\right),
\end{equation}
where $m \leq d-1$ is the intrinsic dimension of $\Gset(\bar{\bs\mcG})$
under $\|\cdot\|_{\mathrm{op}}$.
\end{lmm}

\begin{proof}
Under Assumption~\ref{ass:sigma-min}, the map
$\bs\alpha \mapsto G_{\bs\alpha} = \tO\tp[\tB^{\text{\tiny p}}(\bs\alpha;\bar{\bs\mcG})]$
is $L_{\mathrm{dec}}$-Lipschitz from $(\Delta^{d-1}, \|\cdot\|_1)$
to $(\mcS_r(\mcH), \|\cdot\|_{\mathrm{op}})$ with
$L_{\mathrm{dec}} \lesssim M(1 + M/c^2)$
by Lemma~\ref{lem:proj-decoder-lipschitz}. Note also that the diameter satisfies
$D_{\Gset} \leq 2M$. The result now follows from an identical argument to that of Lemma~\ref{lem:metric-entropy}.
\end{proof}

%========================
% ORACLE DICTIONARY
%========================

\paragraph{Oracle dictionary and dictionary bias.}

We consider $d$ training dynamical systems, each observed through
a trajectory of length $\ntrain$, from which the transfer operators
$\{\hat{\bs\mcG}_j\}_{j \in [d]}$ are estimated.
The empirically learned dictionary $\hat{\bs\mcG}$ minimizes
objective~\eqref{eq:spectral_dictionary_learning} over these
estimated operators.
The oracle dictionary $\bar{\bs\mcG}^*$ is the minimizer of the
same objective evaluated at the true operators
$\{\bs\mcG_j\}_{j \in [d]}$, i.e.\ what dictionary learning
would produce from infinitely long training trajectories:
\begin{equation}
    \bar{\bs\mcG}^*
    \in
    \argmin_{\bar{\bs\mcG} \in \mcN^d}\;
    \frac{1}{d}\sum_{j=1}^d
    \min_{\bs\alpha \in \Delta^{d-1}}
    \dS\!\left(\tB(\bs\alpha; \bar{\bs\mcG}),\, \bs\mcG_j\right).
\end{equation}
The dictionary bias
\begin{equation}
    \mathrm{bias}(\hat{\bs\mcG})
    :=
    \inf_{\bs\alpha \in \Delta^{d-1}}
    \risk(\bs\alpha \mid \hat{\bs\mcG})
    -
    \inf_{\bs\alpha \in \Delta^{d-1}}
    \risk(\bs\alpha \mid \bar{\bs\mcG}^*)
    \;\geq\; 0
\end{equation}
measures the effect of operator estimation error on the learned
dictionary. It vanishes when $\ntrain \to \infty$.
For two dictionaries $\bar{\bs\mcG}, \bar{\bs\mcG}' \in \mcN^d$,
we extend the operator norm componentwise by
\begin{equation}
    \|\hat{\bs\mcG} - \bar{\bs\mcG}^*\|_{\mathrm{op}}
    :=
    \max_{j \in [d]}
    \|\tO\tp[\hat{\bs\mcG}_j] - \tO\tp[\bar{\bs\mcG}^*_j]\|_{\mathrm{op}}.
\end{equation}

\begin{prp}[Control of the dictionary bias]
\label{prop:bias-control}
Under the $L_{\mathrm{dec}}$-Lipschitz assumption on the decoder
in $\bar{\bs\mcG}$ under $\|\cdot\|_{\mathrm{op}}$
(Lemmas~\ref{lem:ot-decoder-lipschitz}
and~\ref{lem:proj-decoder-lipschitz}),
\begin{equation}
\label{eq:bias-lipschitz}
    \mathrm{bias}(\hat{\bs\mcG})
    \;\leq\;
    2(1+M)\kappa^2 L_{\mathrm{dec}}\,
    \|\hat{\bs\mcG} - \bar{\bs\mcG}^*\|_{\mathrm{op}}.
\end{equation}
\end{prp}

\begin{proof}
For any $\bs\alpha \in \Delta^{d-1}$ and two dictionaries
$\bar{\bs\mcG}, \bar{\bs\mcG}' \in \mcN^d$, let
$G_{\bs\alpha} = \tO\tp[\tB(\bs\alpha;\bar{\bs\mcG})]$ and
$G'_{\bs\alpha} = \tO\tp[\tB(\bs\alpha;\bar{\bs\mcG}')]$. Writing $e_t(G) := \feat_{t+1} - G^*\feat_t$,
\begin{equation}
    |\ell_t(G) - \ell_t(\bar G)|
    = \bigl|\langle e_t(G)+e_t(\bar G),(G-\bar G)^*\feat_t\rangle_{\mcH}\bigr|
    \leq \bigl(\|e_t(G)\|_{\mcH}+\|e_t(\bar G)\|_{\mcH}\bigr)
         \|G-\bar G\|_{\mathrm{op}}\|\feat_t\|_{\mcH}.
\end{equation}
Since $\|\feat_t\|_{\mcH} \leq \kappa$ and
$\|G\|_{\mathrm{op}},\|\bar G\|_{\mathrm{op}} \leq M$,
we have $\|e_t(G)\|_{\mcH} \leq (1+M)\kappa$. The
$L_{\mathrm{dec}}$-Lipschitz assumption on the decoder
in $\|\cdot\|_{\mathrm{op}}$,
\begin{equation}
    |\risk(\bs\alpha \mid \bar{\bs\mcG})
     - \risk(\bs\alpha \mid \bar{\bs\mcG}')|
    \leq
    \mathbb{E}_\pi|\ell_t(G_{\bs\alpha}) - \ell_t(G'_{\bs\alpha})|
    \leq
    2(1+M)\kappa^2 L_{\mathrm{dec}}\,
    \|\bar{\bs\mcG} - \bar{\bs\mcG}'\|_{\mathrm{op}}.
\end{equation}
The result follows from the value function inequality
$|\inf_{\bs\alpha} f(\bs\alpha) - \inf_{\bs\alpha} g(\bs\alpha)|
\leq \sup_{\bs\alpha}|f(\bs\alpha) - g(\bs\alpha)|$.
\end{proof}

\begin{rmk}[Rate of the dictionary bias]
\label{rmk:bias-rate}
By Lipschitz continuity of the dictionary learning
objective~\eqref{eq:spectral_dictionary_learning} in the
training operators,
$\|\hat{\bs\mcG} - \bar{\bs\mcG}^*\|_{\mathrm{op}}
\lesssim \max_{j \in [d]} \|\tO\tp[\hat{\bs\mcG}_j]
- \tO\tp[\bs\mcG_j]\|_{\mathrm{op}}$.
Each operator estimation error
$\|\tO\tp[\hat{\bs\mcG}_j] - \tO\tp[\bs\mcG_j]\|_{\mathrm{op}}$
is controlled by the spectral estimation rates of
Appendix~\ref{app:operators-repr},
see~\eqref{eq:rrr_nonparametric_rate}, with trajectory length $\ntrain$.
A union bound over $j \in [d]$ gives with probability
at least $1 - \delta$,
\begin{equation}
    \mathrm{bias}(\hat{\bs\mcG})
    \;\lesssim\;
    (1+M)\kappa^2 L_{\mathrm{dec}}\,
    (\ntrain)^{-\frac{\alpha}{2(\alpha+\beta)}}
    \log(d/\delta),
\end{equation}
where $\alpha \in (0,2]$ and $\beta \in (0,1]$ are the regularity
and spectral decay parameters of Appendix~\ref{app:operators-repr}.
The bias decays with the training trajectory length $\ntrain$,
confirming that longer training trajectories yield a learned
dictionary closer to the oracle.
\end{rmk}

\subsection{Main results}
%========================
% COORDINATE FITTING, IID
%========================

\paragraph{Coordinate fitting guarantee.}

\begin{thm}[Coordinate fitting, i.i.d.]
\label{thm:coord-iid}
Let $\bar{\bs\mcG} \in \mcN^d$ be a fixed dictionary.
Suppose $(X_t)_{t \in [\nsizeds]}$ has marginal $\pi$ and is i.i.d.
Then for any $\delta \in (0,1)$, with probability at least $1-\delta$,
\begin{equation}
\label{eq:coord-iid}
    \risk(\hat{\bs\alpha} \mid \bar{\bs\mcG})
    - \inf_{\bs\alpha \in \Delta^{d-1}}\risk(\bs\alpha \mid \bar{\bs\mcG})
    \;\leq\; \mathrm{bias}(\hat{\bs\mcG})+ 
    c'\sqrt{
        \frac{
            \sigmavariance\cdot
            m\log\!\left(M L_{\mathrm{dec}}(1+M)\kappa^2\sqrt{\nsizeds}\right)
            + \log(1/\delta)
        }{\nsizeds}
    },
\end{equation}
where $c'>0$ is a numerical constant. 
\end{thm}

\begin{proof}[Proof of Theorem~\ref{thm:coord-iid}]
We proceed in five steps.

\paragraph{Step~1: Reduction to uniform deviation.}
Decompose the excess risk as
\begin{equation}
    \risk(\hat{\bs\alpha} \mid \hat{\bs\mcG})
    - \inf_{\bs\alpha \in \Delta^{d-1}}\risk(\bs\alpha \mid \bar{\bs\mcG}^*)
    =
    \underbrace{
        \risk(\hat{\bs\alpha} \mid \hat{\bs\mcG})
        - \inf_{\bs\alpha}\risk(\bs\alpha \mid \hat{\bs\mcG})
    }_{\text{coordinate fitting excess risk}}
    +
    \underbrace{
        \inf_{\bs\alpha}\risk(\bs\alpha \mid \hat{\bs\mcG})
        - \inf_{\bs\alpha}\risk(\bs\alpha \mid \bar{\bs\mcG}^*)
    }_{\mathrm{bias}(\hat{\bs\mcG})}.
\end{equation}
We now bound the coordinate fitting excess risk.
By optimality of $\hat{\bs\alpha}$ for $\erisk$ and of $\bs\alpha^*(\hat{\bs\mcG})$
for $\risk(\cdot \mid \hat{\bs\mcG})$,
\begin{equation}
    \risk(\hat{\bs\alpha} \mid \hat{\bs\mcG})
    - \inf_{\bs\alpha}\risk(\bs\alpha \mid \hat{\bs\mcG})
    \leq
    2\sup_{G \in \Gset(\hat{\bs\mcG})}
    \bigl|\bar\ell(G) - \mathbb{E}_\pi[\ell_t(G)]\bigr|,
\end{equation}
where $\bar\ell(G) :=
\frac{1}{\nsizeds}\sum_{t=1}^{\nsizeds}\ell_t(G)$.

\paragraph{Step~2: Covering net.}
For $\varepsilon > 0$ to be chosen, let $\Gset_\varepsilon$ be an
$\varepsilon$-net of $\Gset(\bar{\bs\mcG})$ under $\|\cdot\|_{\mathrm{op}}$.
By Lemmas~\ref{lem:ot-decoder-lipschitz}
and~\ref{lem:proj-decoder-lipschitz} and the volumetric argument
of Lemma~\ref{lem:metric-entropy},
$\log|\Gset_\varepsilon| \lesssim m\log(M L_{\mathrm{dec}}/\varepsilon)$.
For any $G \in \Gset(\bar{\bs\mcG})$, pick $\bar G \in \Gset_\varepsilon$
with $\|G - \bar G\|_{\mathrm{op}} \leq \varepsilon$.

\paragraph{Step~3: Lipschitz continuity of the loss.}
Writing $e_t(G) := \feat_{t+1} - G^*\feat_t$,
\begin{equation}
    |\ell_t(G) - \ell_t(\bar G)|
    = \bigl|\langle e_t(G)+e_t(\bar G),(G-\bar G)^*\feat_t\rangle_{\mcH}\bigr|
    \leq \bigl(\|e_t(G)\|_{\mcH}+\|e_t(\bar G)\|_{\mcH}\bigr)
         \|G-\bar G\|_{\mathrm{op}}\|\feat_t\|_{\mcH}.
\end{equation}
Since $\|\feat_t\|_{\mcH} \leq \kappa$ and
$\|G\|_{\mathrm{op}},\|\bar G\|_{\mathrm{op}} \leq M$,
we have $\|e_t(G)\|_{\mcH} \leq (1+M)\kappa$, and thus
\begin{equation}
    |\ell_t(G) - \ell_t(\bar G)|
    \leq 2(1+M)\kappa^2\,\|G-\bar G\|_{\mathrm{op}}
    \leq 2(1+M)\kappa^2\,\varepsilon.
\end{equation}

\paragraph{Step~4: Bernstein's inequality.}
For fixed $\bar G \in \Gset_\varepsilon$,
$0 \leq \ell_t(\bar G) \leq (1+M)^2\kappa^2$ and
$\mathrm{Var}_\pi(\ell_t(\bar G)) \leq \sigmavariance$.
By Bernstein's inequality applied to the i.i.d.\ sequence
$(\ell_t(\bar G))_{t=1}^{\nsizeds}$, for any $s > 0$,
\begin{equation}
    \mathbb{P}\!\left(
        \bigl|\bar\ell(\bar G) - \mathbb{E}_\pi[\ell_t(\bar G)]\bigr|
        \geq \frac{s}{2}
    \right)
    \leq
    2\exp\!\left(
        -\frac{\nsizeds\,s^2/4}
              {2\sigmavariance
               + \frac{2}{3}(1+M)^2\kappa^2\,s}
    \right).
\end{equation}

\paragraph{Step~5: Union bound and optimization.}
Set $\varepsilon := s/(4(1+M)\kappa^2)$ so that
$2(1+M)\kappa^2\varepsilon = s/2$.
By the triangle inequality,
$|\bar\ell(G) - \mathbb{E}_\pi[\ell_t(G)]|
\leq |\bar\ell(\bar G) - \mathbb{E}_\pi[\ell_t(\bar G)]| + s/2$.
The net size at this $\varepsilon$ satisfies
$\log|\Gset_\varepsilon|
\lesssim m\log(M L_{\mathrm{dec}}(1+M)\kappa^2/s)$.
Taking a union bound over $\Gset_\varepsilon$ and solving for $s$
such that the right-hand side equals $\delta$ yields
\begin{equation}
    \sup_{G \in \Gset(\bar{\bs\mcG})}
    \bigl|\bar\ell(G) - \mathbb{E}_\pi[\ell_t(G)]\bigr|
    \;\lesssim\;
    \sqrt{
        \frac{
            \sigmavariance\cdot
            m\log\!\left(M L_{\mathrm{dec}}(1+M)\kappa^2\sqrt{\nsizeds}\right)
            + \log(1/\delta)
        }{\nsizeds}
    }.
\end{equation}
The result follows from Step~1.
\end{proof}

\paragraph{Extension to $\beta$-mixing trajectories.}

For a stationary process $(X_t)_{t\geq 1}$ with marginal $\pi$,
the $\beta$-mixing coefficients are
\begin{equation}
    \beta(\tau)
    :=
    \sup_{t \geq 1}
    \sup_{B \in \sigma(X_s,\,s\leq t)
              \otimes\sigma(X_s,\,s\geq t+\tau)}
    \bigl|\mathbb{P}(B) - (\mathbb{P}\otimes\mathbb{P})(B)\bigr|,
    \quad\tau\in\mathbb{N}.
\end{equation}
Fix block length $\tau \in \mathbb{N}$, let
$\nsizeds_{\mathrm{eff}} := \lfloor\nsizeds/(2\tau)\rfloor$,
and define the effective variance
\begin{equation}
\label{eq:sigmaeff}
    \sigmavarianceeff
    :=
    \sigmavariance
    +
    4\,\sigmavariancebis
    \sum_{s=1}^{\tau-1} \beta(s),
\end{equation}
which accounts for within-block temporal dependence via the
Davydov--Rio covariance inequality~\citep{rio1993covariance}.

\statguarantee*

\begin{proof}
\textbf{Steps~1--3: Reduction, covering, and Lipschitz control.}
By optimality of $\hat{\bs\alpha}$,
\begin{equation}
    \risk(\hat{\bs\alpha} \mid \hat{\bs\mcG})
    - \inf_{\bs\alpha}\risk(\bs\alpha \mid \hat{\bs\mcG})
    \leq
    2\sup_{G \in \Gset(\hat{\bs\mcG})}
    \bigl|\bar\ell(G) - \mathbb{E}_\pi[\ell_t(G)]\bigr|.
\end{equation}
For $\varepsilon > 0$ to be chosen, let $\Gset_\varepsilon$ be an
$\varepsilon$-net of $\Gset(\hat{\bs\mcG})$ under $\|\cdot\|_{\mathrm{op}}$,
with $\log|\Gset_\varepsilon|
\lesssim m\log(M L_{\mathrm{dec}}/\varepsilon)$.
For any $G \in \Gset(\hat{\bs\mcG})$, pick
$\bar G \in \Gset_\varepsilon$ with
$\|G - \bar G\|_{\mathrm{op}} \leq \varepsilon$.
By Step~3 of Theorem~\ref{thm:coord-iid},
$|\ell_t(G) - \ell_t(\bar G)|
\leq 2(1+M)\kappa^2\,\varepsilon$.

\textbf{Step~4: Control of $\mathrm{Var}(Y_j)$.}
Following the blocking construction
of~\citep[Lemma~1]{kostic2022learning}, split
$(\ell_t(\bar G))_{t=1}^{\nsizeds}$ into odd block sums
$Y_j := \sum_{i=2(j-1)\tau+1}^{(2j-1)\tau}\ell_{t_i}(\bar G)$,
$j = 1,\ldots,\nsizeds_{\mathrm{eff}}$.
Expanding the variance of $Y_j$,
\begin{equation}
    \mathrm{Var}(Y_j)
    =
    \sum_{i=1}^\tau \mathrm{Var}(\ell_{t_i}(\bar G))
    +
    2\sum_{1 \leq i < k \leq \tau}
    \mathrm{Cov}\!\left(\ell_{t_i}(\bar G),\,\ell_{t_k}(\bar G)\right).
\end{equation}
The diagonal terms satisfy
$\sum_{i=1}^\tau\mathrm{Var}(\ell_{t_i}(\bar G))
\leq \tau\sigmavariance$ by stationarity.
For the off-diagonal terms at lag $s = k-i$,
the Davydov--Rio inequality~\citep{rio1993covariance} gives
$|\mathrm{Cov}(\ell_{t_i}(\bar G), \ell_{t_k}(\bar G))|
\leq 2\beta(s)\sigmavariancebis$.
Since lag $s \in \{1,\ldots,\tau-1\}$ appears at most $\tau$ times,
\begin{equation}
\label{eq:var-block}
    \mathrm{Var}(Y_j)
    \leq
    \tau\sigmavariance
    + 4\tau\sigmavariancebis\sum_{s=1}^{\tau-1}\beta(s)
    =
    \tau\sigmavarianceeff.
\end{equation}

\textbf{Step~5: Bernstein for $\beta$-mixing.}

The odd blocks are $\beta(\tau)$-approximately independent:
replacing them by exactly independent copies incurs total
variation error $2(\nsizeds_{\mathrm{eff}}-1)\beta(\tau)$,
following~\citep[Lemma~1]{kostic2022learning}.
Each $Y_j$ satisfies $|Y_j - \mathbb{E}[Y_j]|
\leq \tau(1+M)^2\kappa^2$ almost surely and
$\mathrm{Var}(Y_j) \leq \tau\sigmavarianceeff$
by~\eqref{eq:var-block}.
Applying Bernstein's inequality to the approximately independent
odd blocks, for any $s > 0$,
\begin{equation}
    \mathbb{P}\!\left(
        \Bigl|
            \frac{1}{\nsizeds_{\mathrm{eff}}}
            \sum_{j=1}^{\nsizeds_{\mathrm{eff}}}Y_j
            - \mathbb{E}_\pi[\ell_t(\bar G)]
        \Bigr| \geq \frac{s}{2}
    \right)
    \leq
    2\exp\!\left(
        -\frac{\nsizeds_{\mathrm{eff}}\,s^2/4}
              {2\sigmavarianceeff
               + \frac{2}{3}(1+M)^2\kappa^2\,s}
    \right)
    + 2(\nsizeds_{\mathrm{eff}}-1)\beta(\tau).
\end{equation}

\textbf{Step~6: Union bound and optimization.}

Set $\varepsilon = s/(4(1+M)\kappa^2)$ so that
$2(1+M)\kappa^2\varepsilon = s/2$.
By the triangle inequality,
$|\bar\ell(G) - \mathbb{E}_\pi[\ell_t(G)]|
\leq |\bar\ell(\bar G) - \mathbb{E}_\pi[\ell_t(\bar G)]| + s/2$.
The net size at this $\varepsilon$ satisfies
$\log|\Gset_\varepsilon|
\lesssim m\log(M L_{\mathrm{dec}}(1+M)\kappa^2/s)$.
Taking a union bound over $\Gset_\varepsilon$ and solving for $s$
such that the right-hand side equals $\delta$ yields the result,
with the mixing tail $2(\nsizeds_{\mathrm{eff}}-1)\beta(\tau)$
carried additively in the probability statement.

\end{proof}

\begin{rmk}[Exponentially fast mixing]
\label{rmk:mixing}
When $\beta(\tau) \leq Ce^{-c\tau}$ for constants $C,c > 0$,
the sum $\sum_{s=1}^{\tau-1}\beta(s) \leq C/c$ is bounded
uniformly in $\tau$, so
$\sigmavarianceeff \lesssim \sigmavariance + \sigmavariancebis$.
Choosing $\tau = \lceil c^{-1}\log(\nsizeds/\delta)\rceil$
ensures $2(\nsizeds_{\mathrm{eff}}-1)\beta(\tau) \leq \delta$,
giving $\nsizeds_{\mathrm{eff}} =
\Omega(\nsizeds/\log(\nsizeds/\delta))$
and recovering nearly the i.i.d.\ rate of
Theorem~\ref{thm:coord-iid}.
\end{rmk}

\begin{rmk}[Finite-dimensional embeddings]
\label{rmk:finite-dim}
When $\phi$ takes values in $\mathbb{R}^p$,
$\kappa = \sup_x\|\phi(x)\|_2 < \infty$ holds automatically
with bounded activations such as $\arctan$, and
$\|\cdot\|_{\mcH} = \|\cdot\|_2$ throughout.
Both theorems apply without modification.
\end{rmk}

\end{document}